\documentclass[journal]{IEEEtran}
\AtBeginDocument{%
  }

\usepackage{times}  
\usepackage{helvet}  
\usepackage{courier}  
\usepackage[hyphens]{url}  
\usepackage{graphicx} 
\usepackage{natbib}  
\usepackage{caption} 
\usepackage{algorithm}
\usepackage{algorithmic}
\usepackage{amsmath,amssymb,amsthm}
\usepackage{comment}
\usepackage{soul}
\usepackage{enumitem}
\usepackage{mathtools}
\usepackage{tikz}
\usepackage{hyperref}
\theoremstyle{definition}

\theoremstyle{plain}
\newtheorem{theorem}{Theorem}  
\newtheorem{lemma}{Lemma}
\newtheorem{proposition}{Proposition}

\theoremstyle{remark}

\newcommand{\x}{\mathbf{x}}
\newcommand{\E}{\mathbb{E}}
\newcommand{\N}{\mathcal{N}}
\newcommand{\Prob}{\mathbb{P}}

\newcommand{\miqing}[1]{\textcolor{red}{ MiqingComment: \textit{#1}}}

\usepackage{newfloat}
\usepackage{listings}

\usepackage{eso-pic}

\AddToShipoutPictureFG*{%
  \ifnum\value{page}=1
    \AtPageUpperLeft{%
      \raisebox{-12mm}{\hspace*{15mm}\makebox[0pt][l]{%
        \footnotesize The manuscript has been accepted for publication at AAAI 2026.%
      }}%
    }%
  \fi
}

\begin{document}

\title{Random is Faster than Systematic in Multi-Objective Local Search}

\author{Zimin Liang and Miqing Li\(^*\)
\thanks{Zimin Liang and Miqing Li (\emph{corresponding author}) are with the School of Computer Science, University of Birmingham, Edgbaston, Birmingham B15 2TT, UK (emails: zxl525@student.bham.ac.uk; m.li.8@bham.ac.uk).}}

\maketitle


\begin{abstract}
Local search is a fundamental method in operations research and combinatorial optimisation. It has been widely applied to a variety of challenging problems, including multi-objective optimisation where multiple, often conflicting, objectives need to be simultaneously considered.
In multi-objective local search algorithms, a common practice is to maintain an archive of all non-dominated solutions found so far, from which the algorithm iteratively samples a solution to explore its neighbourhood. A central issue in this process is how to explore the neighbourhood of a selected solution.
In general, there are two main approaches: 1) systematic exploration and 2) random sampling.
The former systematically explores the solution's neighbours until a stopping condition is met -- for example, when the neighbourhood is exhausted (i.e., the best improvement strategy) or once a better solution is found (i.e., first improvement).
In contrast, the latter randomly selects and evaluates only one neighbour of the solution. 
One may think systematic exploration may be more efficient, as it prevents from revisiting the same neighbours multiple times.
In this paper, however, we show that this may not be the case. We first empirically demonstrate that the random sampling method is consistently faster than the systematic exploration method across a range of multi-objective problems. We then give an intuitive explanation for this phenomenon using toy examples, showing that the superior performance of the random sampling method relies on the distribution of ``good neighbours''. Next, we show that the number of such neighbours follows a certain probability distribution during the search. Lastly, building on this distribution, we provide a theoretical insight for why random sampling is more efficient than systematic exploration, regardless of whether the best improvement or first improvement strategy is used.
\end{abstract}


\section{Introduction}
Local search (LS) is a class of search heuristics that move from solution to solution in the space of candidate solutions by applying local changes. 
LS has been widely used to solve various challenging search and optimisation problems, including SAT \cite{liang2025diversat,ye2025better,zheng2025integrating}, max $c$-cut \cite{garvardt2024parameterized}, max quasi-clique \cite{chen2021nuqclq}, max $k$-plex \cite{sun2024nukplex}, graph edge partition \cite{guo2021enhancing}, Latin square completion \cite{xie2024swap}, matroid optimisation \cite{benabbou2021combining}, integer linear programming \cite{lin2024parailp}, and other applied problems \cite{shatabda2013mixed,zhang2024towards,su2021weighting,gruttemeier2021efficient}.

Multi-objective optimisation refers to an optimisation scenario where there is more than one objective to be considered simultaneously. LS is a popular tool in tackling multi-objective combinatorial optimisation problems (MOCOPs)~\cite{Blot2018}, due to its simple algorithmic structure and neighbourhood-driven exploration which fits well with the discrete search space. 

Early multi-objective LS methods were often desgined through converting an MOCOP into several single-objective problems and using scalar fitness (aggregated from multiple objectives) to rank solutions. Representative algorithms include~\cite{Ishibuchi1998,Jaszkiewicz2002,Paquete2003}.
Later, Pareto-based local search (PLS) heuristics have received increasing attention and becomes a mainstream method~\cite{Paquete2004, giel_expected_2003, Lust2010, Dubois2015, Drugan2012, jaszkiewicz2018many, shi2022improving, ren_pareto-optimal_2023, kang_new_2024}. PLS adopts Pareto dominance to compare solutions \cite{Paquete2004}. It maintains an archive of non-dominated solutions found so far and iteratively samples solutions from the archive to explore their
neighbourhoods.
Such PLS algorithms show competitive performance on MOCOPs such as the multi-objective TSP~\cite{Lust2010}, knapsack~\cite{alsheddy_guided_2009}, NK-landscape~\cite{Aguirre2005} and scheduling~\cite{Liefooghe2012} problems. More recently, several studies \cite{inja2014queued,phan2023pareto,li_empirical_2024} have shown the performance advantage of PLS over well-established multi-objective evolutionary algorithms (e.g., NSGA-II \cite{Deb2002}, SMS-EMOA \cite{Beume2007} and MOEA/D \cite{Zhang2007}).

A key issue in PLS algorithms is how to explore the neighbourhood of a solution sampled from the archive.
In general, there are two approaches: 1) systematic exploration and 2) random sampling.
The former, denoted by \textit{systematic} PLS or \textit{s}-PLS, systematically explores the solution's neighbours until a stopping condition is met -- for example when the neighbourhood is exhausted (i.e., the best improvement strategy), or once a better solution is found (i.e., the first improvement stragtegy).
The latter, denoted by \textit{randomised} local search or \textit{r}-PLS, randomly selects and evaluates only one neighbour of the solution.  
Systematic PLS is popular in the fields of operations research and metaheuristics~\cite{Dubois2015, derbel_multi-objective_2016, shi_pplsd_2020, santos_multi-objective_2022, ceschia_educational_2023, gao_bi-objective_2025}. On the other hand, a few studies consider randomised PLS \cite{chicano_efficient_2016, jaszkiewicz2018many}, particularly in the theory community of evolutionary computation \cite{laumanns2004, doerr2021theoretical, Dinot2023, wietheger_near-tight_2024}.

A natural question that arises is which of these two approaches is more efficient.
Systematic search might appear preferable, as it methodically explores a solution's neighbourhood and ensures each neighbour is visited at most once. In contrast, randomised PLS may revisit the same neighbours multiple times (as, at each step, it randomly selects a solution from the archive and then randomly picks one of its neighbours), resulting in redundant evaluations. 

In this paper, however, we show that randomised PLS is in general faster than systematic PLS. We first empirically demonstrate consistent performance advantage of randomised PLS over systematic PLS across a wide range of problems, no matter whether the best improvement or first improvement strategy of neighbourhood exploration is used. We then give an intuitive explanation for this phenomenon using toy examples. Lastly, we provide theoretical understanding through analysing the runtime of two PLS approaches in finding a good solution (i.e., one that is not dominated by the archive, thus leading to improvement of the archive quality) under the distribution of neighbouring solutions of the considered problems. Key contributions of this work can be summarised as:

\begin{itemize}
    \item Out of the two major classes of local search algorithms in multi-objective optimisation, we show one is faster than the other, which refutes a commonly-held belief that the two LS algorithm classes may have their own strengths and suit different problems.

    \item We empirically investigate the neighbourhood structure (the number of good neighbours) during the search process of the PLS algorithms, and demonstrate that they follow a certain probability distribution.
    
    \item We theoretically investigate the reason behind our empirical observations, and prove that one algorithm class requires less time to find a good solution under that probability distribution.
\end{itemize}

\makeatletter
\newlength{\@commentwidth}
\renewcommand{\algorithmiccomment}[1]{%
  \settowidth{\@commentwidth}{\scriptsize$\triangleright$\ #1}%
  \ifdim\@commentwidth>\linewidth
    {\scriptsize$\;\triangleright$\hspace{0.5em}#1}%
  \else
    {\scriptsize\hfill$\triangleright$\ #1}%
  \fi
}
\makeatother
\begin{algorithm}[tbp] 
	\footnotesize
	\caption{Systematic Pareto Local Search (\emph{s}-PLS)}
	\label{Alg:s_PLS}
	\begin{algorithmic}[1]
        \REQUIRE $max\_eval$ (algorithm terminating condition), $\mathcal{N}$ (neighbourhood function, e.g. 1-bit neighbourhood)
		\STATE $s \gets random\_solution()$ 
		\STATE $explored(s) \gets \text{FALSE}$ \COMMENT{Mark $s$ as unexplored}
		\STATE $A \gets \{s\}$ \COMMENT{Place $s$ into the archive}
        \STATE $eval\gets1$
		\WHILE{$eval\leq max\_eval$ \textbf{ and } $\exists a\in A : explored(a)=\text{FALSE}$ }
		\STATE {$s \gets selection(\{a\in A | explored(a)=\text{FALSE}\})$ }\COMMENT{Select an unexplored solution from $A$ based on certain strategy (e.g., random selection or using an indicator/scalarisation function) }
        \REPEAT
        \STATE $s'\gets next\_neighbour(\N(s))$ \COMMENT{Get the next neighbour from the neighbourhood of $s$, usually in lexicographic or random order}
		\IF {$\nexists a\in A: a\preceq s' $}
		\STATE $A \gets A \cup \{s^\prime\} \setminus \{a' \in A \mid s'\prec a'\}$\COMMENT{Ensure that the archive only contains unique non-dominated solutions}
        \STATE $explored(s') \gets \text{FALSE}$ 
		\ENDIF
        \STATE $eval\gets eval+1$
		\UNTIL{$stop\_condition()$}\COMMENT{Stop exploring this neighbourhood if a condition is met, e.g., it is exhausted, or a satisfactory solution is found (e.g., it is not dominated by the archive)}        
		\STATE $explored(s) \gets \text{TRUE}$ \COMMENT{ Mark $s$ as explored}
		\ENDWHILE
		\RETURN $A$
	\end{algorithmic}
\end{algorithm}

It is worth noting that randomised PLS heuristics, such as SEMO \cite{laumanns2004}, have been frequently studied in the runtime analysis literature within the evolutionary computation theory community \cite{bian2018general, doerr2021theoretical, Dinot2023, zheng_how_2024, wietheger_near-tight_2024}. These works, along with others in the theory field, focus on the runtime required to find Pareto optimal solutions, offering insights into questions such as whether the algorithm struggles to find a Pareto optimum, and how easy it is to find other optimal solutions in the Pareto front once a single one has been found. Our work does not consider the time required to find optimal solutions. The results presented in this paper does not imply any guarantees of optimality for the solutions found by PLS, but rather how fast different PLS algorithms find a better solution during the search process. This has practical significance, as optimal solutions are often infeasible to obtain within a reasonable time for most real-world MOCOPs \cite{figueira2017easy}, and it is valuable to know which types of algorithms can make faster progress.

\section{Preliminaries}

Let $\boldsymbol{f}(\x)=(f_{1}(\x),\dots,f_{m}(\x))$ denote an $m$-objective (minimisation) problem over a decision space $X$, $\x\in X$, $f_{i}:X\to\mathbb{R}$. We say $\x$ weakly Pareto-dominates $\mathbf{y}$ (notation: $\x\preceq\mathbf{y}$) iff $
  \forall\,i\in\{1,\dots,m\}:\;f_{i}(\x)\le f_{i}(\mathbf{y}),
$
and $\x$ Pareto-dominates $\mathbf{y}$ (notation: $\x\prec\mathbf{y}$) if at least one inequality is strict. The Pareto set consists of all non-dominated solutions, denoted by
$
  PS \;=\;\bigl\{\x\in X \mid \nexists\,\mathbf{y}\in X:\,\mathbf{y}\prec\x\bigr\};
$
and the Pareto front consists of their corresponding point in the objective space, denoted by $PF = \{\boldsymbol{f}(\x) \mid \x\in PS\}.$ 

A central concept in local search algorithms is \emph{neighbourhood}, which defines the structure of the search landscape. The neighbourhood of a solution $\x$ is denoted by $\mathcal{N}(\x) \subseteq X \setminus {\x}$, and typically consists of solutions reachable from $\x$ by a single elementary operation (e.g., a one-bit flip in binary encoded problems).

\begin{algorithm}[tbp] 
	\footnotesize
	\caption{Randomised Pareto Local Search (\emph{r}-PLS)}
	\label{Alg:r_PLS}
	\begin{algorithmic}[1]
       \REQUIRE $max\_eval$ (algorithm terminating condition), $\mathcal{N}$ (neighbourhood function, e.g. 1-bit neighbourhood)
		\STATE $s \gets random\_solution()$
		\STATE $A \gets \{s\}$ \COMMENT{Place $s$ into the archive}
        \STATE $eval\gets1$
		\WHILE{$eval\le max\_eval$}
		\STATE $s \gets {selection}(A)$  \COMMENT{Select a solution from $A$ based on certain strategy (e.g., random selection or using an indicator/scalarisation function)} 
		\STATE $s' \overset{U}{\sim} \N(s)$ \COMMENT{Sample a neighbour uniformly at random from the neighbourhood of $s$}
		\IF {$\nexists a\in A: a\preceq s' $}
		\STATE $A \gets A \cup \{s^\prime\} \setminus \{a' \in A \mid s'\prec a'\}$ \COMMENT{Ensure that the archive only contains unique non-dominated solutions}
		\ENDIF 
        \STATE $eval\gets eval+1$
		\ENDWHILE
		\RETURN $A$
	\end{algorithmic}
\end{algorithm}

\section{Two Classes of Pareto-based Local Search}

In this work, we focus on Pareto-based local search (PLS) algorithms \cite{Paquete2004}, which encompass the majority of multi-objective local search heuristics~\cite{Blot2018}.   
PLS is an iterative search process that performs local variation to solutions and compare them with new solutions using Pareto dominance. 
It maintains an archive to store all non-dominated solutions found so far. 
For each iteration, PLS considers a solution in the archive to generate new solution(s) by exploring its neighbourhood.
Based on how the solution's neighbourhood is explored, there are two classes of PLS algorithms: systematic PLS (\emph{s}-PLS) and randomised PLS (\emph{r}-PLS). 

\begin{figure*}[htbp]
\vspace{-5pt}
\renewcommand{\arraystretch}{0.1} 
\fontsize{8pt}{9.5pt}\selectfont
\begin{center}
    \begin{tabular}{@{}c@{}c@{}c@{}c@{}}
	\includegraphics[scale=0.31]{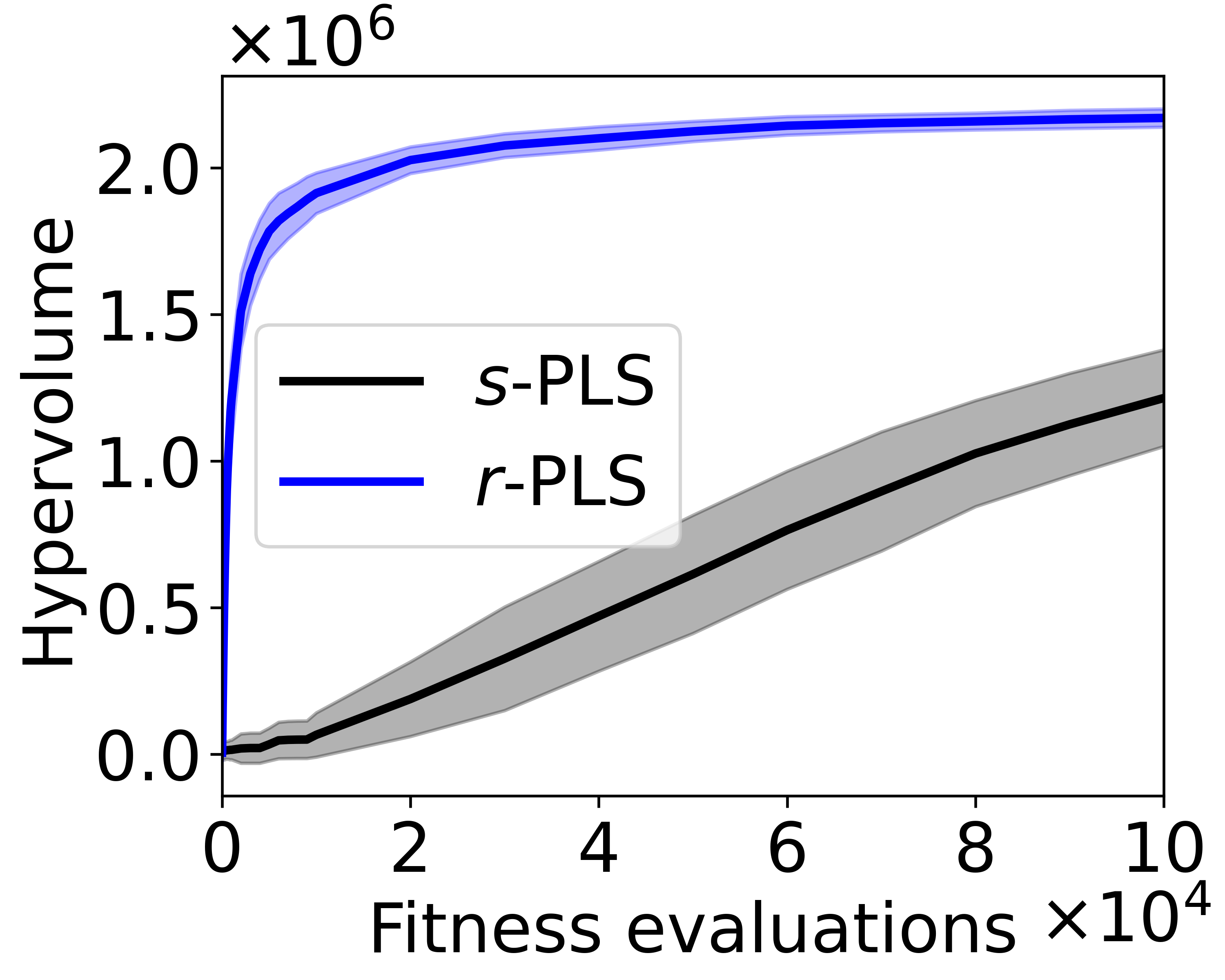} ~~&~~ 
    \includegraphics[scale=0.31] {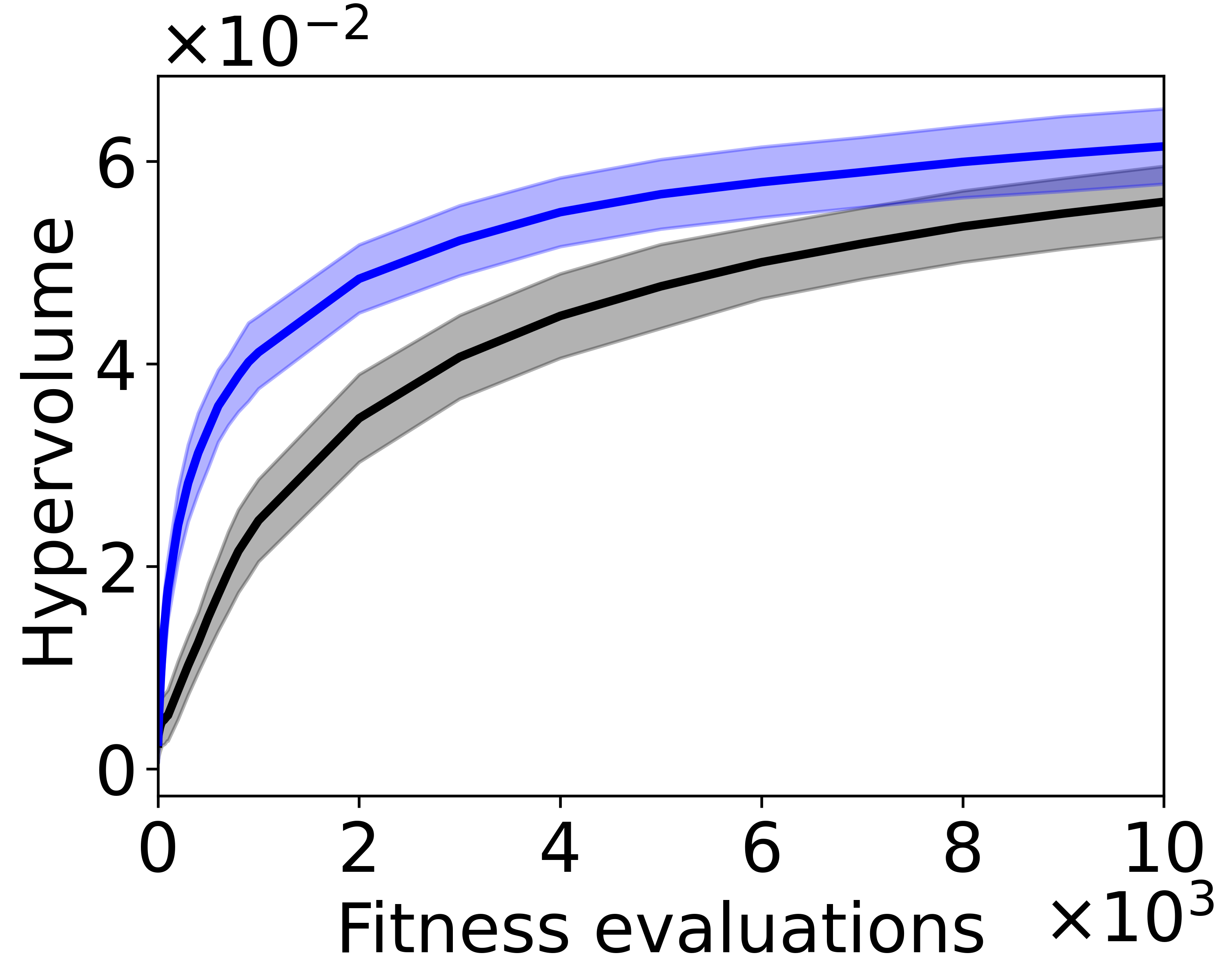} ~~&~~
    \includegraphics[scale=0.31]{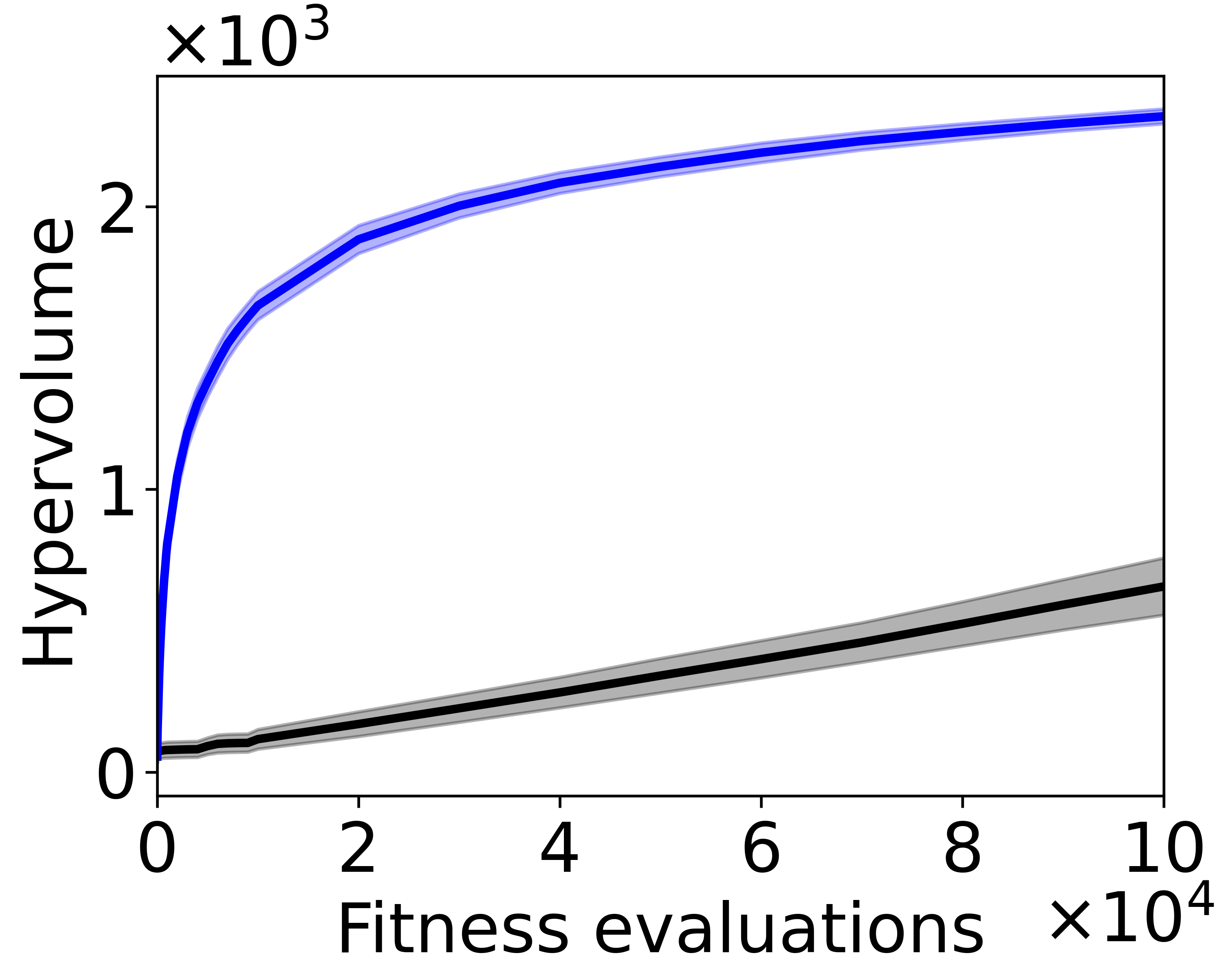} ~~&~~
	\includegraphics[scale=0.31] {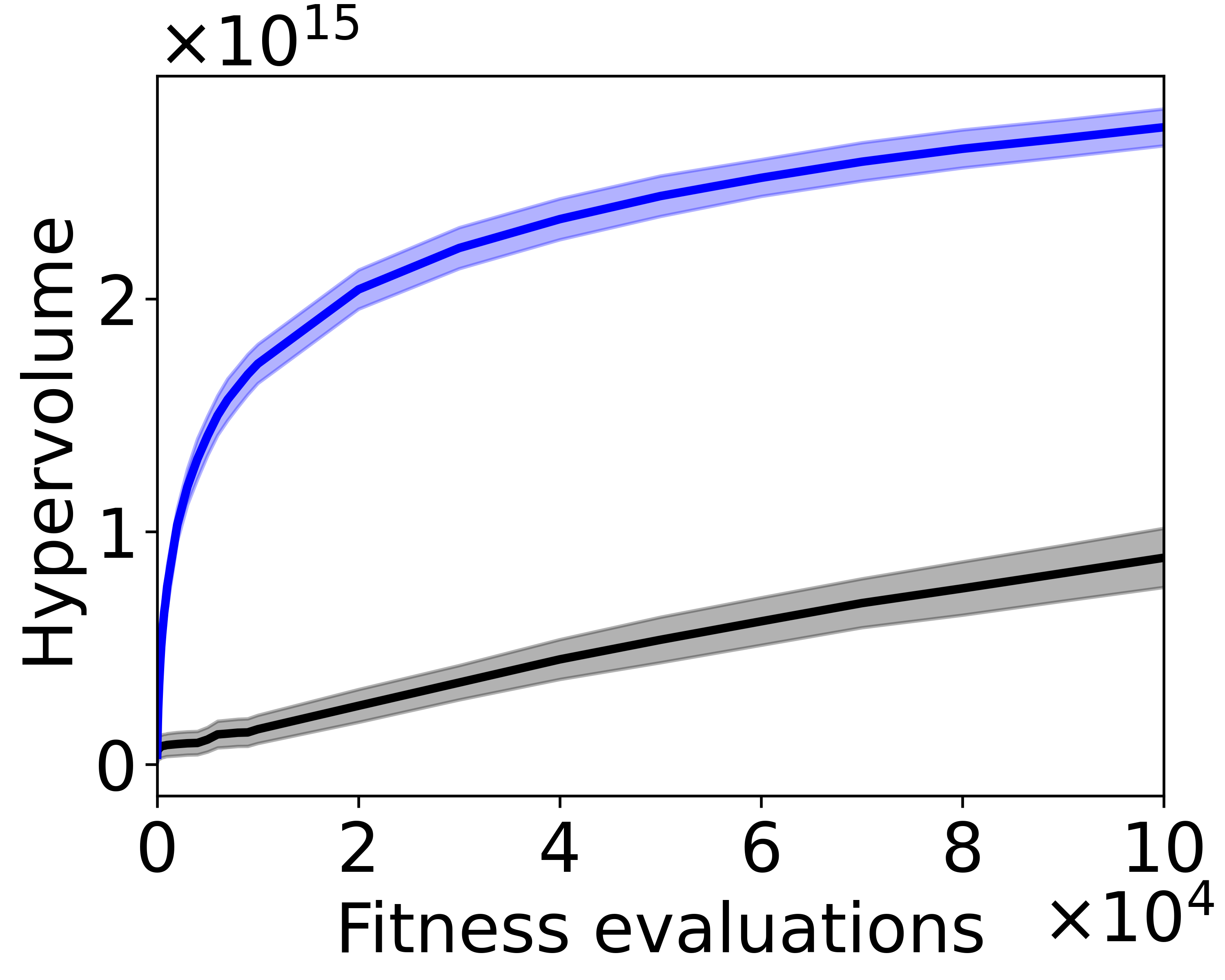} \\
	(a) Knapsack &
    (b) NK-Landscape &
    (c) TSP  &
    (d) QAP   \\	
	\end{tabular}
	\end{center}
        \vspace{-12pt}
	\caption{
    \small The hypervolume (HV) trajectory (higher is better) of the \emph{s}-PLS and \emph{r}-PLS algorithms across 30 runs on the four MOCOPs with 100 decision variables. The bolded line and shaded area represent the mean and standard deviation of the HV, respectively. 
    }
	\label{Fig:hv_2}
\end{figure*} 

\begin{figure*}[htbp]
	\vspace{-5pt}
	\renewcommand{\arraystretch}{0.1} 
	\fontsize{8.5pt}{10pt}\selectfont
	\begin{center}
        \begin{tabular}{@{}c@{}@{}c@{}@{}c@{}@{}c@{}}
			\includegraphics[scale=0.31]{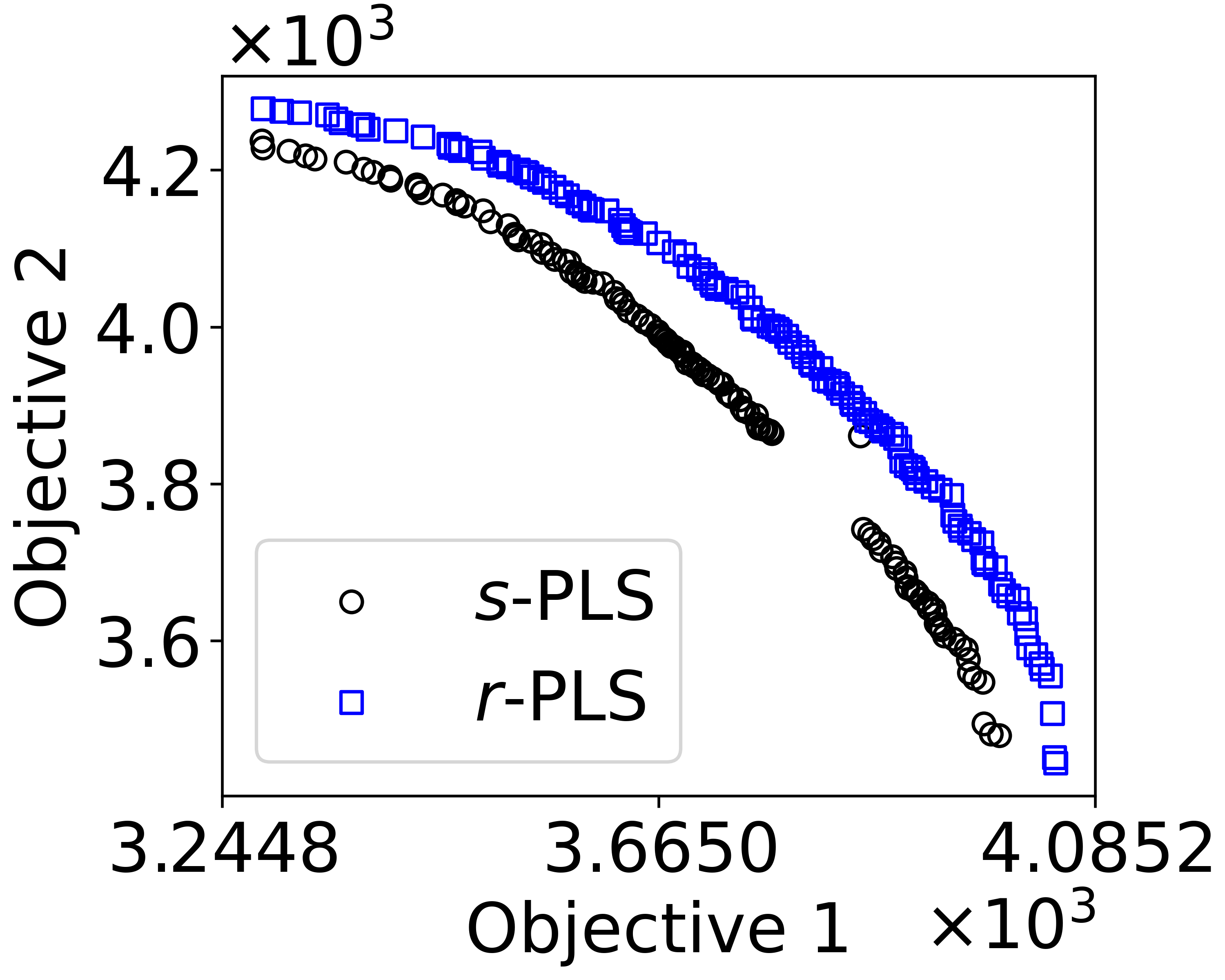} ~~&~~ 
			\includegraphics[scale=0.31] {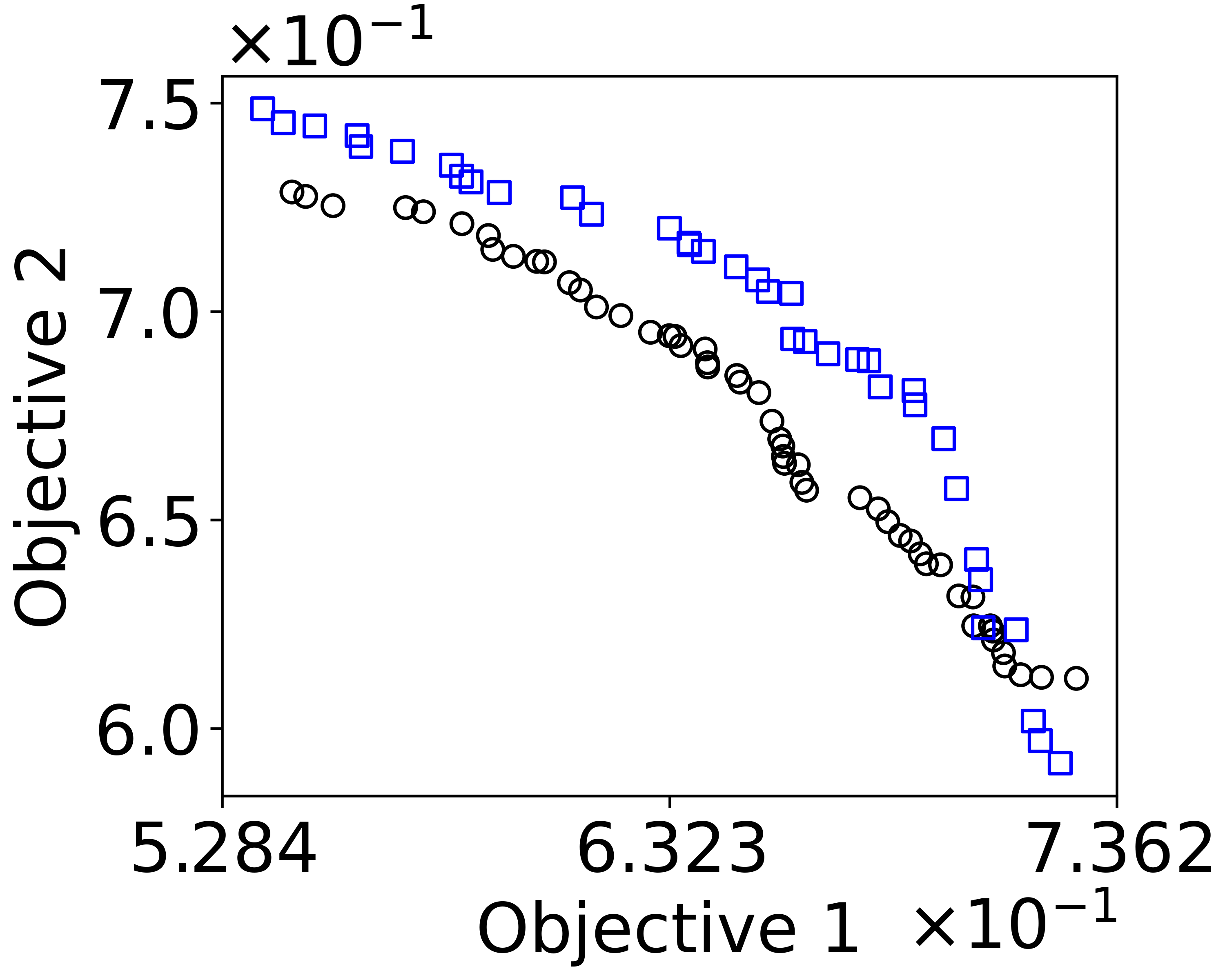} ~~&~~
			\includegraphics[scale=0.31]{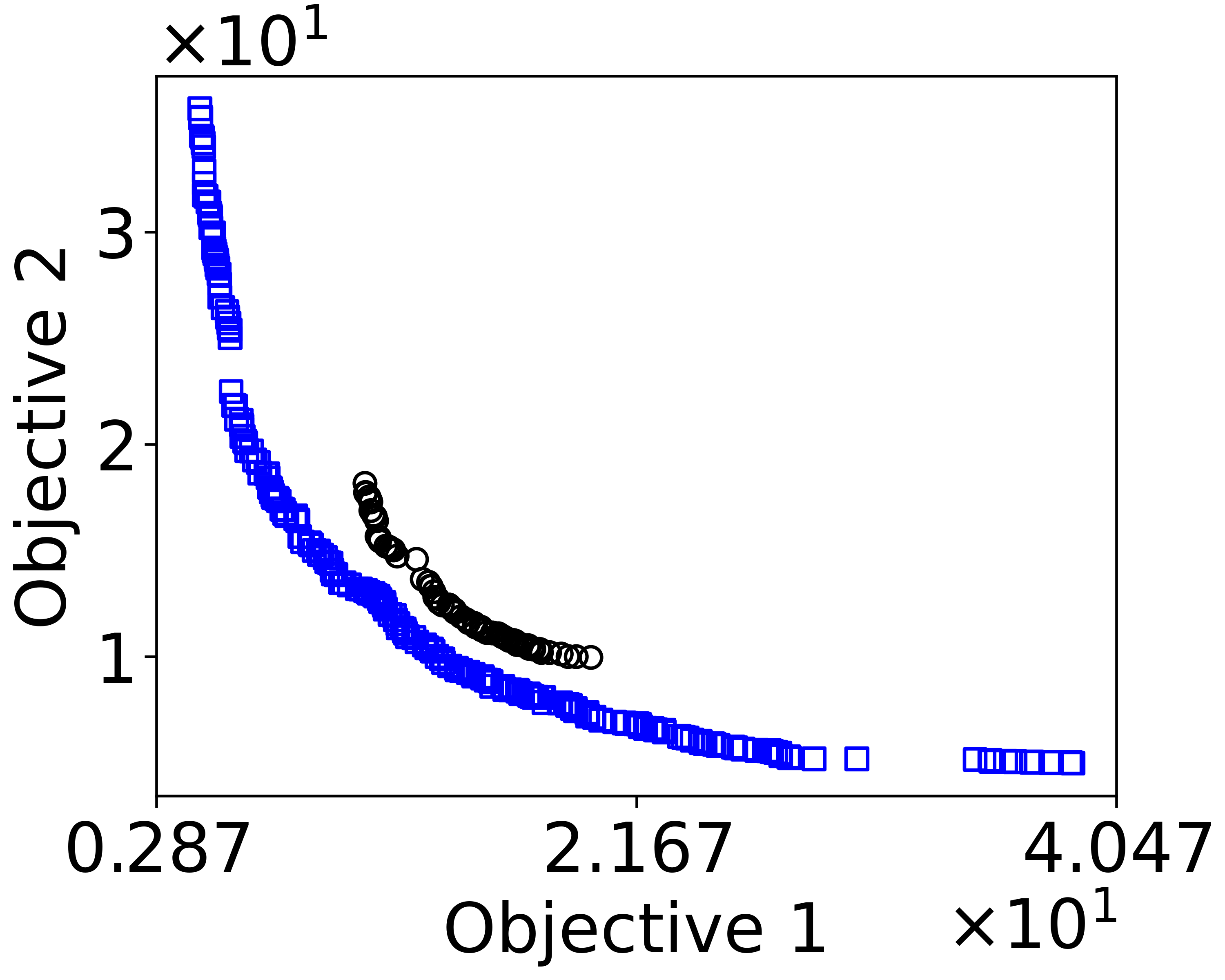} ~~&~~
			\includegraphics[scale=0.31] {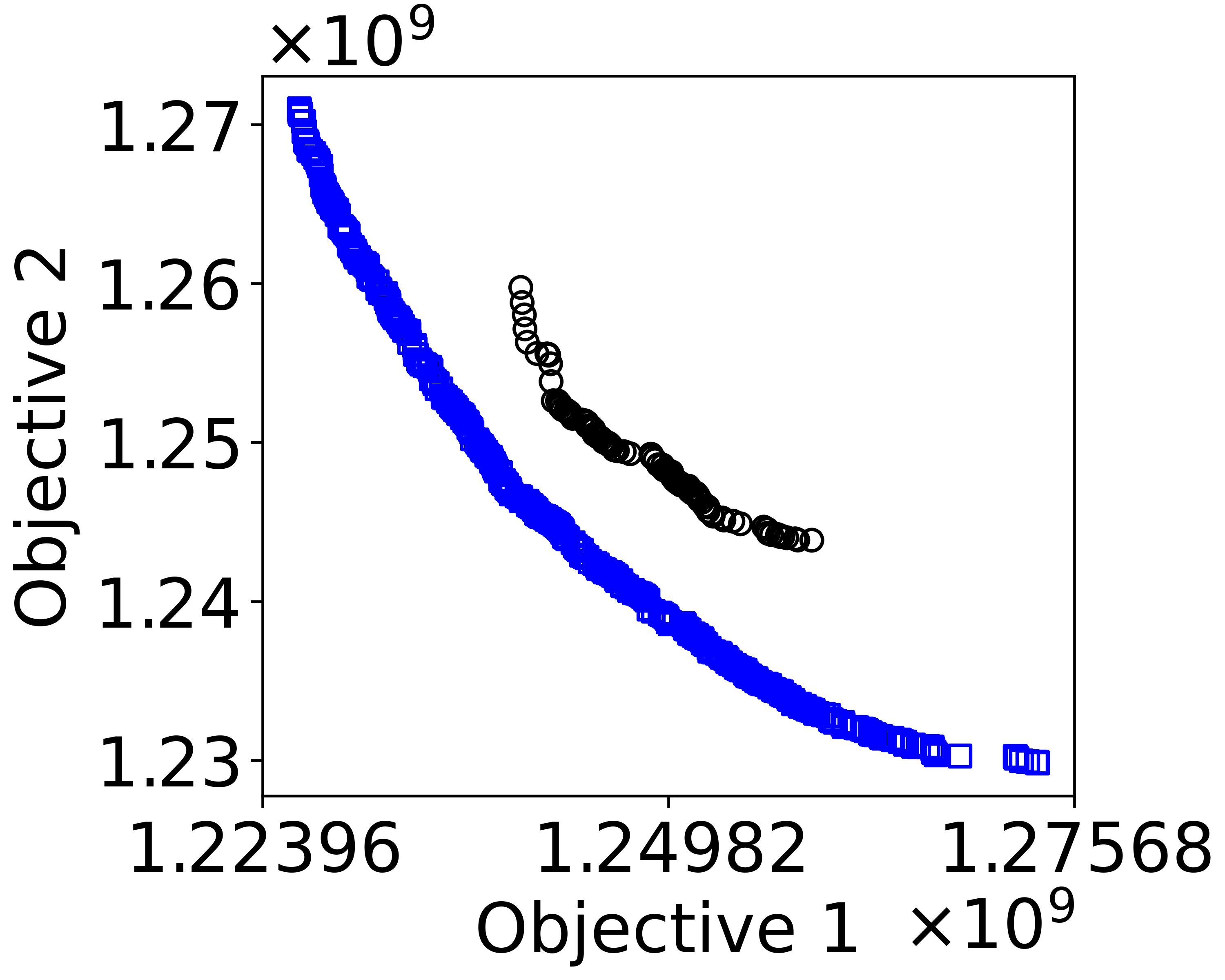}  \\
	(a) Knapsack &
        (b) NK-Landscape &
        (c) TSP  &
        (d) QAP  \\	
		\end{tabular}
	\end{center}
	\vspace{-12pt}
	\caption{\small All non-dominated solutions (i.e., solutions in the archive) obtained by the \emph{s}-PLS and \emph{r}-PLS in a typical run on the four MOCOPs, where the Knapsack and NK-Landscape are maximisation problems, and the TSP and QAP are minimisation problems.}
	\label{Fig:obj}
    \vspace{-5pt}
\end{figure*}

\emph{s}-PLS algorithms visit the solution's neighbours one-by-one (and test if they can be put into the archive) until a condition is met. 
The condition can be neighbourhood exhaustion (i.e., best improvement) or finding a satisfactory solution (i.e., first improvement).
In multi-objective optimisation, a satisfactory solution can be defined in different ways, but the two common ones are i) a solution that is not dominated by any solution in the archive and ii) a solution that dominates at least one solution in the archive \cite{Liefooghe2012,Dubois2015}. We consider both in this study.   
Once a solution’s neighbourhood is explored, it is labelled as ``explored'' and is never revisited.
The framework of \emph{s}-PLS is given in Algorithm~\ref{Alg:s_PLS}.
\emph{s}-PLS is popular in the area. Many existing PLS heuristics fall into this class, such as the original PLS~\cite{Paquete2004}, MORBC~\cite{Aguirre2005} and its variants~\cite{ide_multi-objective_2024}, GPLS~\cite{alsheddy_guided_2009}, 2PPLS~\cite{lust_multiobjective_2012}, anytime PLS~\cite{Dubois2015}, PPLS/D~\cite{shi2022improving}, LOMONAS~\cite{phan2023pareto}, and a new enhanced PLS~\cite{kang_new_2024}.

Unlike \emph{s}-PLS, for each iteration \emph{r}-PLS randomly selects a solution from the archive and then visit one of its neighbours at random, without tracking which solution has been visited or explored. 
This purely randomised sampling approach may visit the same neighbour of a solution multiple times, resulting in redundant cost. 
Algorithm~\ref{Alg:r_PLS} gives the framework of \emph{r}-PLS.
Existing work on developing \emph{r}-PLS algorithms is relatively rare. Notable examples include SEMO~\cite{giel_expected_2003,laumanns2004}, SPLS~\cite{Drugan2012} and MPLS~\cite{jaszkiewicz2018many}. That said, \emph{r}-PLS has been frequently studied in the evolutionary computation theory community~\cite{bian2018general, doerr2021theoretical, Dinot2023, zheng_how_2024, wietheger_near-tight_2024}, due to its appealing nature (simplicity, minimalism and generalisability).

A relevant question is which of the two classes of PLS is faster. One might assume that \emph{s}-PLS to be more efficient, as it conducts a systematic search and avoids repeatedly exploring the same solutions, particularly when using the first improvement strategy. However, in the following sections, we show that this may not be the case. 

\section{Empirical Comparison between \emph{s}-PLS and \emph{r}-PLS}
We consider the most basic versions of the PLS algorithm classes that randomly pick a solution in the archive for neighbourhood exploration, as our focus is on general neighbourhood exploration strategies rather than other tuneable algorithmic components (e.g., selection criteria). That is, for \emph{r}-PLS, we consider SEMO \cite{laumanns2004}. For \emph{s}-PLS, we consider the canonical version \cite{Paquete2004} that randomly selects one unexplored solution from the archive to explore its neighbourhood, and the exploration stops when the neighbourhood is exhausted, namely, the best improvement strategy.
Additionally, we later consider the first improvement strategy, which was regarded to be faster than the original \emph{s}-PLS \cite{Liefooghe2012,Dubois2015}.     

We consider four commonly used MOCOPs, the multi-objective 0/1 knapsack~\cite{teghem_multi-objective_1994}, travelling salesman problem (TSP)~\cite{ribeiro2002study}, quadratic assignment problem (QAP)~\cite{knowles2003instance} and NK-landscapes~\cite{Aguirre2004}. These problems exhibit diverse characteristics, including pseudo-Boolean (knapsack and NK-landscape) and permutation (TSP and QAP) formulations; smooth (knapsack and TSP) to rugged (QAP and NK-landscape with a large $K$) fitness landscape; order-based (TSP) versus position-based (QAP) permutations; and constrained (knapsack) versus unconstrained (TSP, QAP and NK-landscape) settings.
Each problem was instantiated in three sizes (100, 200 and 500 decision variables). Detailed description of these MOCOPs are given in Appendix~\ref{apx:mocop}.
We use the hypervolume \cite{Zitzler1999}, a strongly Pareto-compliant indicator capturing convergence, spread, uniformity, and cardinality \cite{Li2019}, to assess the quality of solutions obtained by search algorithms. The reference point in hypervolume, along with detailed experimental and algorithmic settings, are given Appendix~\ref{apx:setting}.

\begin{figure*}[tbp]
	\vspace{-5pt}
	\renewcommand{\arraystretch}{0.1} 
	\fontsize{8pt}{9.5pt}\selectfont
	\begin{center}
		\hspace*{-5pt}\begin{tabular}{@{}c@{}@{}c@{}@{}c@{}@{}c@{}}
			\includegraphics[scale=0.31]{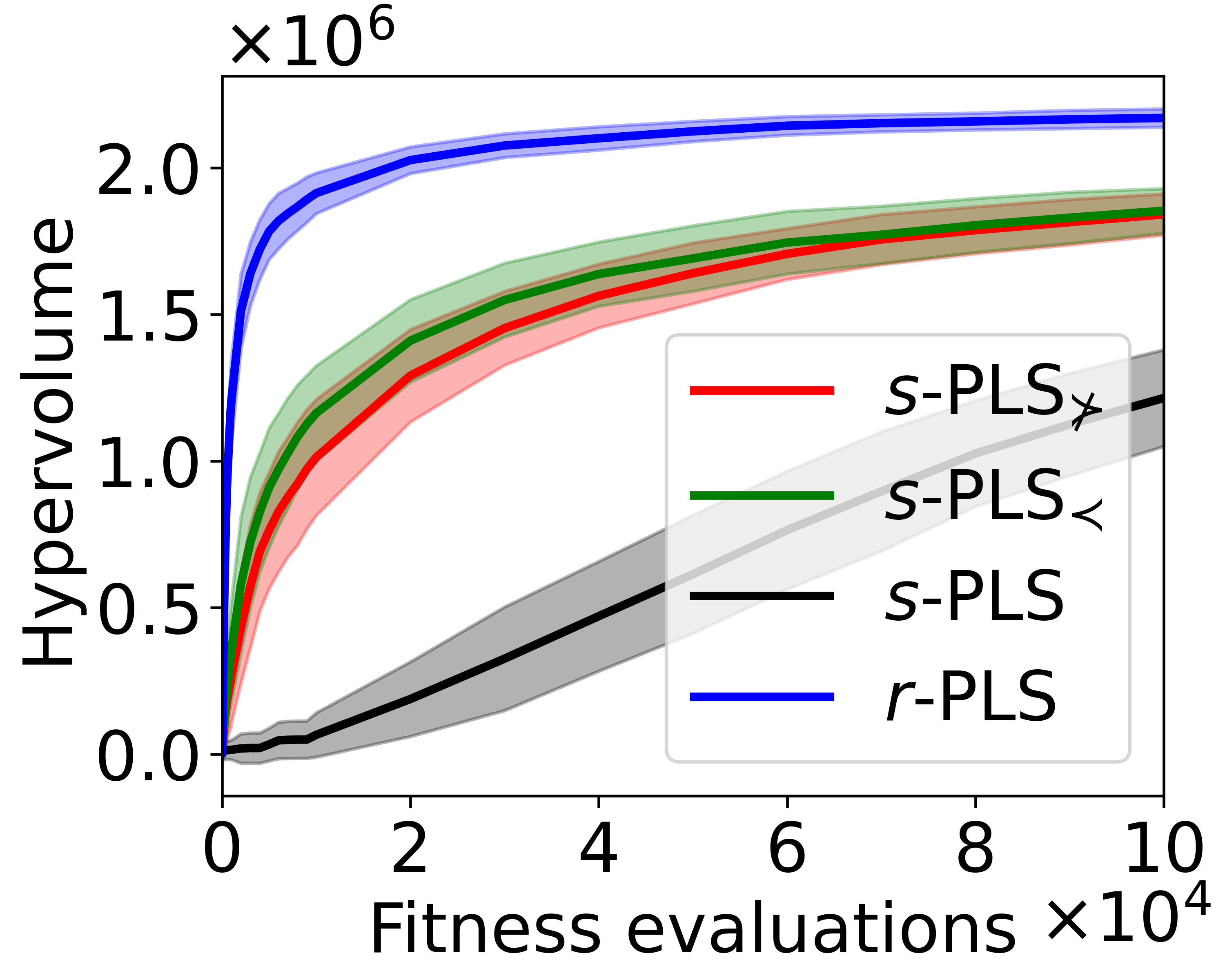} ~~&~~ 
			\includegraphics[scale=0.31] {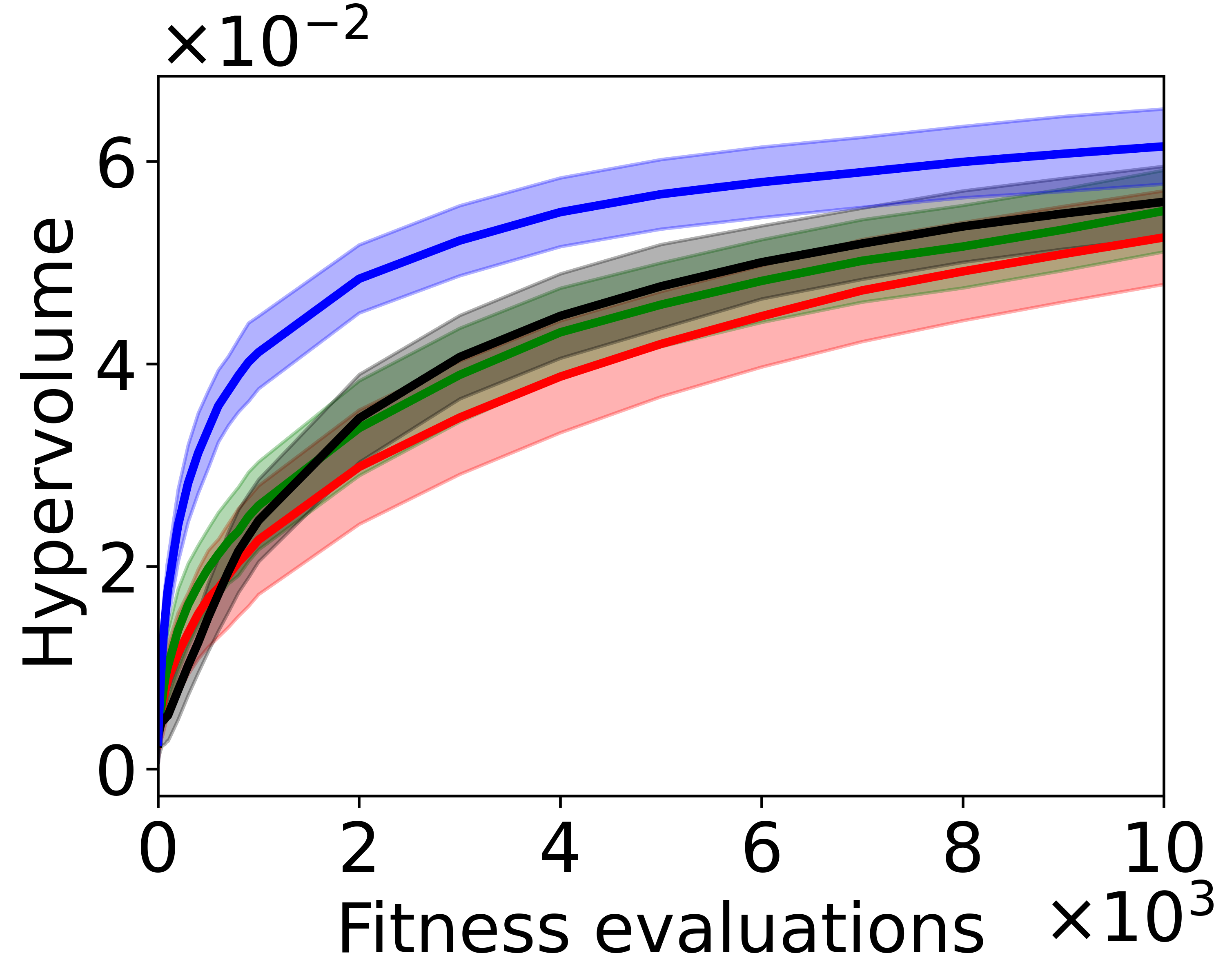} ~~&~~
			\includegraphics[scale=0.31]{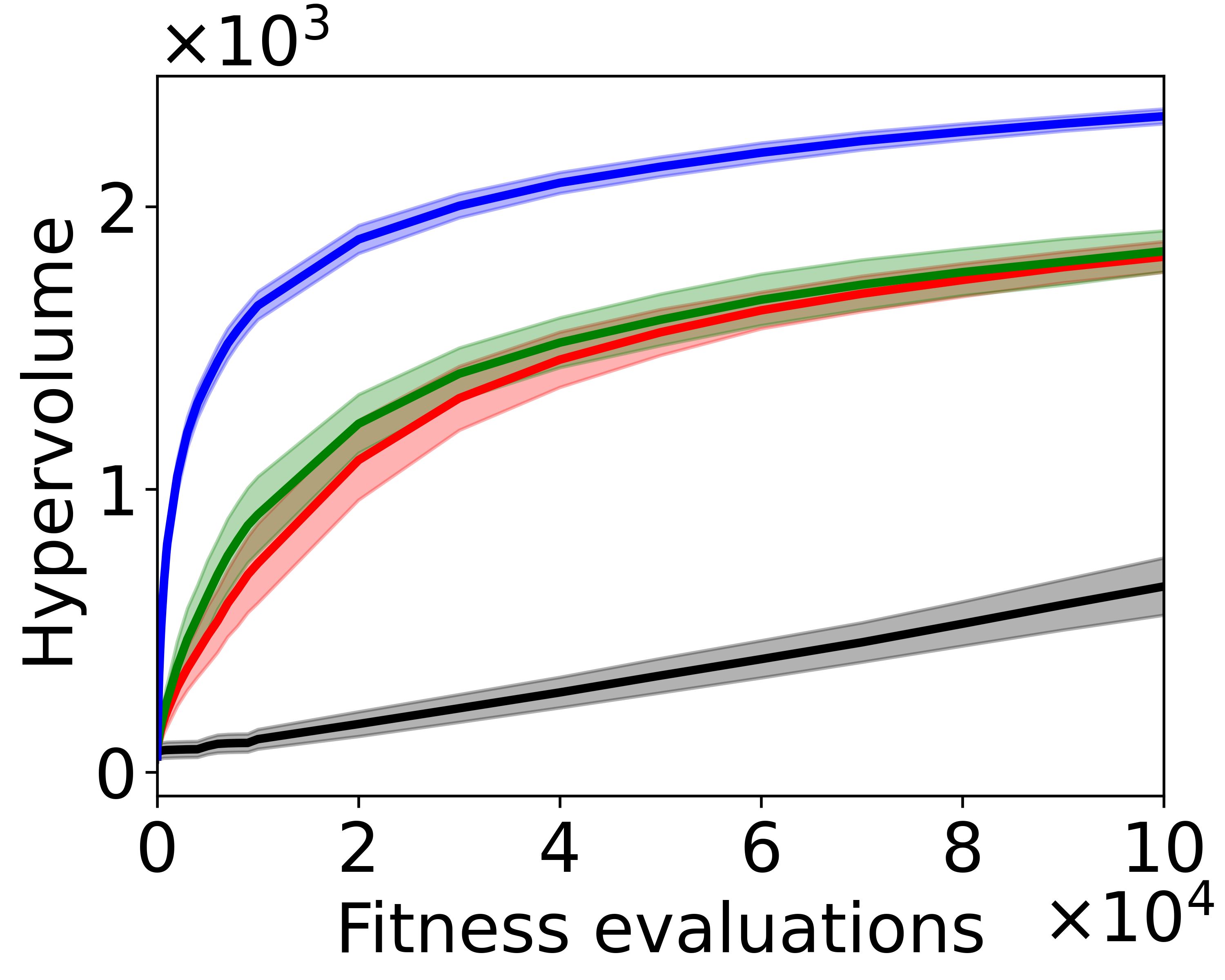} ~~&~~
			\includegraphics[scale=0.31] {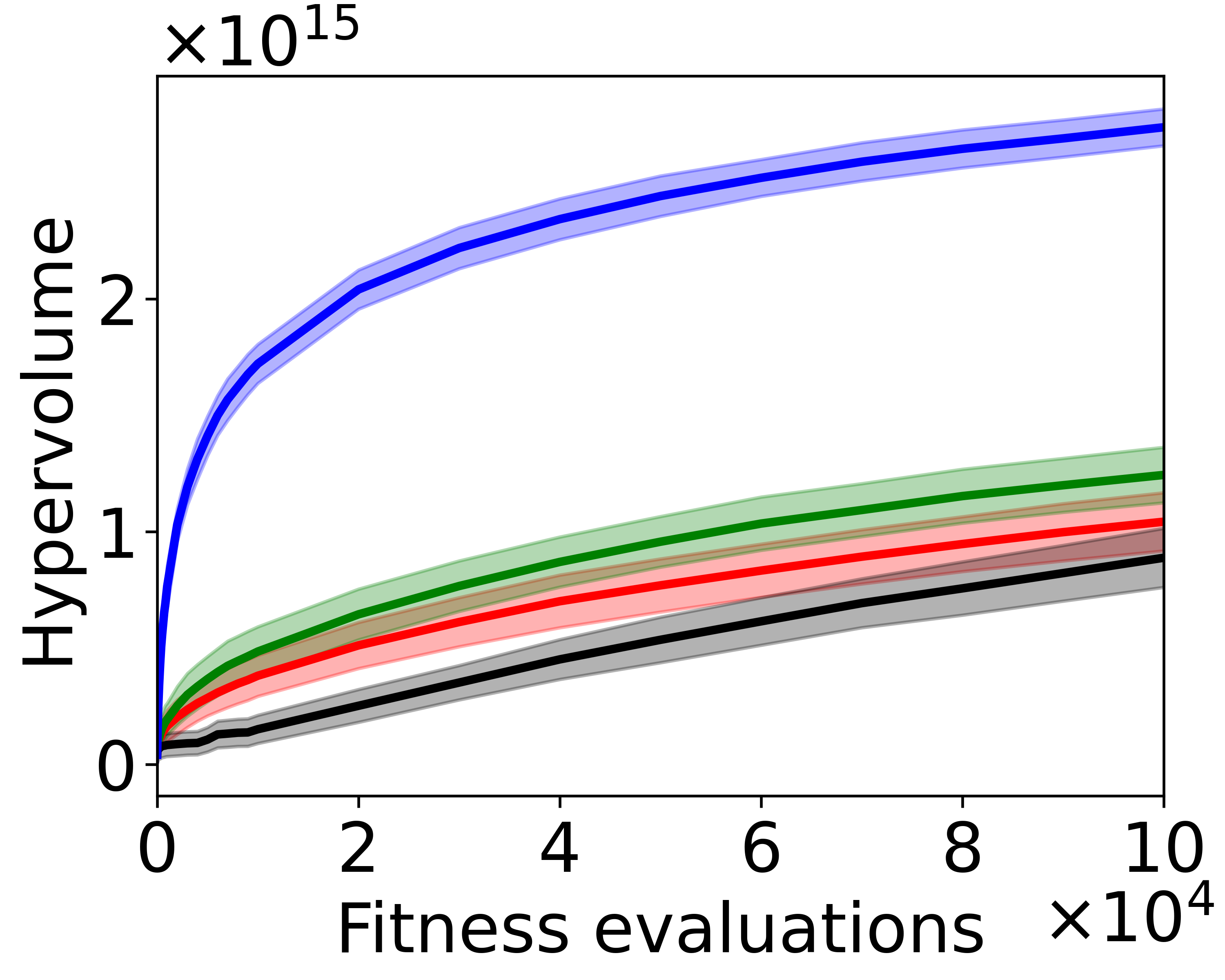} \\
			(a) Knapsack  &
			(b) NK-Landscape &
			(c) TSP  &
			(d) QAP \\	
		\end{tabular}
	\end{center}
	\vspace{-12pt}
	\caption{\small The hypervolume trajectory (higher is better) of the \emph{s}-PLS, \emph{s}-PLS$_\nsucc$ (it stops when finding a solution not dominated by the archive), \emph{s}-PLS$_\prec$ (it stops when finding a solution dominating at least one solution in the archive), and \emph{r}-PLS across 30 runs on the four MOCOPs with 100 decision variables. The bolded line and shaded area represent the mean and standard deviation of the hypervolume, respectively.}
	\label{Fig:hv_all}
    \vspace{-8pt}
\end{figure*}

Figure~\ref{Fig:hv_2} shows the average hypervolume (bolded line) and standard deviation (shaded area) of \emph{s}-PLS (black) and \emph{r}-PLS (blue) across 30 independent runs on the four problems with 100 decision variables.
As can be seen, \emph{r}-PLS obtains a better HV value than \emph{s}-PLS throughout the search on all the problems.
On the NK-landscape problem, the performance advantage of \emph{r}-PLS is relative small, while on the other three problems, \emph{r}-PLS demonstrates a clear edge since the very beginning of the search process. The difference is even clearer when a larger problem size is involved (see results on the four 200D\&500D problems in Appendix~\ref{apx:results}).
Figure~\ref{Fig:obj} gives the non-dominated solutions obtained by the two algorithms in a typical run on the four problems when the search ends in Figure~\ref{Fig:hv_2}. It is clear that the \emph{r}-PLS's solutions have better convergence on the Knapsack and NK-landscape problems, and better convergence and diversity on the TSP and QAP problems. 

One explanation for \emph{s}-PLS's slow progress is its exhaustive neighbourhood exploration (the best improvement strategy) -- it always evaluates the entire neighbourhood of a solution before moving on. In contrast, the first improvement strategy -- the exploration stops once a satisfactory solution is found -- can be faster, as shown in \cite{Liefooghe2012,Dubois2015}. We thus now consider two representative \emph{s}-PLS alternatives that stop the neighbourhood exploration earlier \cite{Dubois2015}. The first one stops the neighbourhood exploration once a non-dominated solution (i.e., it is not dominated by any solution in the archive) is found, denoted by \emph{s}-PLS$_\nsucc$. The other one stops once a dominating solution (i.e., it dominates at least one solution in the archive) is found, denoted by \emph{s}-PLS$_\prec$.

Figure~\ref{Fig:hv_all} shows the average hypervolume (bolded line) and standard deviation (shaded area) of the two \emph{s}-PLS algorithms, as well as the basic \emph{s}-PLS and \emph{r}-PLS across 30 independent runs on the 100D four problems. 
The results on the higher-dimensional problems (200D and 500D) are given in Appendix~\ref{apx:results}.
As shown, the two first-improvement \emph{s}-PLS algorithms achieve noticeably better hypervolume performance than the basic \emph{s}-PLS on most problems. Between them, \emph{s}-PLS$_\prec$ is faster than \emph{s}-PLS$_\nsucc$. However, they are still outperformed by \emph{r}-PLS to a large extent. 

The above result is interesting -- \emph{r}-PLS, as a simple randomised local search heuristic that may revisit solutions multiple times, has been shown to be substantially more effective than the \emph{s}-PLS algorithms which conduct a systematic search and avoid repeated exploration and revisiting of solutions. Next, we will investigate why this happens.

\section{Intuitions from Toy Examples}
In this section, we present two toy examples of solution distributions in the archive to help illustrate why and when \emph{r}-PLS can be faster than \emph{s}-PLS. We here are interested in how quickly (in terms of the number of evaluations) an algorithm can find a ``good'' solution. A good solution is one that is not dominated by the archive, leading to an update of the archive and an improvement of the archive quality.

\begin{figure}[!ht]
  \centering
  \begin{tabular}{c@{\hspace{2.5em}}c}
    \begin{tikzpicture}[scale=1.4]
      \draw[thick] (0,0) rectangle (2,2);
      \foreach \x/\y/\filled in {0.5/0.5, 1.5/0.5, 0.5/1.5, 1.5/1.5} {
      \ifdim \y pt=0.5pt
        \filldraw (\x,\y) circle (0.3);
      \else
        \draw (\x,\y) circle (0.3);
      \fi
      }
    \end{tikzpicture}
    &
    \begin{tikzpicture}[scale=1.4]
      \draw[thick] (0,0) rectangle (2,2);
      \foreach \x/\y in {0.5/0.5, 1.5/0.5, 0.5/1.5, 1.5/1.5} {
        \draw (\x,\y) circle (0.3);
        \begin{scope}
          \clip (\x,\y) circle (0.3);
          \fill (\x,\y) ++(0,-0.3) rectangle ++(0.3,0.6);
        \end{scope}
      }
    \end{tikzpicture}
    \\
    (a) Example 1 & (b) Example 2
  \end{tabular}
  \vspace{-5pt}
  \caption{\small Two toy examples of the distribution of \textit{promising} solutions in the archive. Here, a promising solution (filled circle) means all of its neighbours being non-dominated to the archive; an unpromising solution (empty circle) means all of its neighbours being dominated by at least one solution in the archive; and a half-promising solution (half-filled circle) means half of its neighbours being non-dominated to the archive. 
  In Example 1, the archive contains two promising solutions and two unpromising solutions. In Example 2, all the four solutions in the archive are half-promising.}  
  \vspace{-10pt}
  \label{fig:halfgood}
\end{figure}
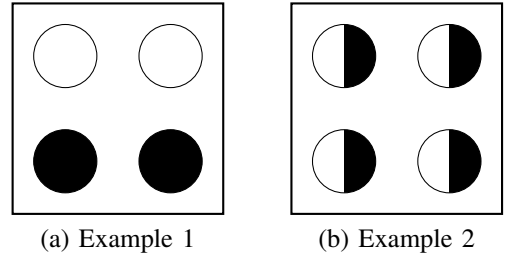

\vspace{-5pt}
\subsection{Example 1}\label{sec:toy2}
Let us consider a scenario where the archive contains four solutions: half of them \textit{promising} and the other half \textit{unpromising}, as illustrated in Figure~\ref{fig:halfgood}(a). 
Here, a promising solution (filled circle in the figure) is one for which every neighbour is a good solution 
(i.e., non-dominated to the archive); an unpromising solution (empty circle) is one for which every neighbour is dominated by at least one solution in the archive. 
To start with, both \emph{s}-PLS and \emph{r}-PLS select one solution uniformly at random from the archive to explore.

For \emph{r}-PLS, 
if a promising solution is selected,
its randomly sampled neighbour will be a non-dominated solution, and the search succeeds in a single evaluation. 
If an unpromising solution is selected, 
its randomly sampled neighbour will be a dominated solution; one evaluation is spent and the algorithm proceeds.
Since the probability of selecting a promising solution is $1/2$, the expected number of evaluations needed to find a good neighbour is 2.

For \emph{s}-PLS, if a promising solution is selected, 
the first neighbour picked must be a non-dominated solution, and the search succeeds in a single evaluation.
However, if an unpromising solution is selected, \emph{s}-PLS will exhaustively evaluate its entire neighbourhood and waste $|\N|$ evaluations before 
marking it as explored and selecting another solution for exploration, 
where $|\N|$ denotes the neighbourhood size.
Therefore, while both algorithms have a 50\% chance of selecting a promising solution, the potential cost after selecting an unpromising solution is much higher for \emph{s}-PLS.
In fact, the expected number of evaluations needed for \emph{s}-PLS to find a good solution in this scenario is $\frac{n|\N|}{n+2}+1$, where $n$ is the number of solutions in the archive. 
The proof can be found in Appendix~\ref{apx:example_proof}. As can be calculated, for a problem whose neighbourhood size is 10, the expected number of evaluations needed for \emph{s}-PLS to find a good solution in our example is $\frac{4\times10}{4+2}+1=7.67$, substantially larger than the expected number of evaluations for \emph{r}-PLS (i.e., $2$).   

\subsection{Example 2}\label{sec:toy2}

The above example illustrates an extremely uneven distribution of their neighbours: half of the solutions have only good neighbours, while the other half have only poor ones. 
In Example 2, we present the opposite case: each solution has an equal mix of good and poor neighbours (Figure~\ref{fig:halfgood}(b)).

For \emph{r}-PLS, the expected number of evaluations to find a new good solution remains $2$, as the probability of selecting a good neighbour for each solution in the archive is $1/2$.
For \emph{s}-PLS, since all the solutions have the same distribution that half of their neighbours are good, the probability of selecting a good neighbour for a solution is $1/2$. For the case that \emph{s}-PLS first chooses a poor neighbour, the probability of selecting a good neighbour of that solution at the second pick is $\frac{|\N|}{2(|\N|-1)}$, as the algorithm does not repeatedly visit the same neighbour. This process can continue until there are only good neighbours left for that solution. In fact, the expected number of evaluations needed for \emph{s}-PLS to find a good neighbour is $\frac{|\N|+1}{(|\N|/2)+1}$, which is always less than that of \emph{r}-PLS (i.e., $2$). The proof is given in Appendix~\ref{apx:example_proof}. 

From the above two examples, we can see the behaviour of the two PLS algorithms depends on the distribution of neighbours of solutions in the archive. If the distribution of solutions' neighbours (in terms of their quality) is even, then \emph{s}-PLS is faster; if the distribution is not even, then \emph{r}-PLS has the edge. But, what is the distribution of the neighbours during the search of the PLS algorithms on common MOCOPs? In the next section, we aim to answer this question.

\section{Distribution of Solutions' Neighbours in \emph{s}-PLS and \emph{r}-PLS}\label{sec:geo}

In this section, we investigate the distribution of solutions' neighbours during the search process of \emph{s}-PLS and \emph{r}-PLS on the MOCOPs. 
Specifically, we count the number of good neighbours among solutions in the archive, and test if this number follows a known discrete probabilistic model (e.g., uniform, binomial and Poisson distributions). Note that this practice is a common approach to test whether an empirical dataset fits a known discrete distribution, for example to fit the number of web requests (which has been found to follow a Zipf distribution) \cite{breslau_web_1999}, and the citation counts of scientific papers~\cite{xie_geometric_2016} (which follows a geometric distribution). 
In this study, we consider commonly seen discrete distributions~\cite{johnson2005}, characterised by their different probabilistic tail type: \emph{uniform} (a balance distribution), \emph{Zipf} (heavy-tailed with polynomial decay), \emph{geometric} (light-tailed with exponential decay), \emph{Poisson} (light-tailed with super-exponential decay), and \emph{binomial} (light-tailed with bounded support and a hard cut-off) distributions. 
Following the common practice, goodness-of-fit here is examined via a $\chi ^2$ test at the $\alpha=5\%$ significance level.
Details on distribution fitting, comparison of different distributions, and parameter settings can be seen in Appendix~\ref{apx:dist}. 

\begin{figure}[t]
  \begin{flushright}
  \includegraphics[scale=0.20]{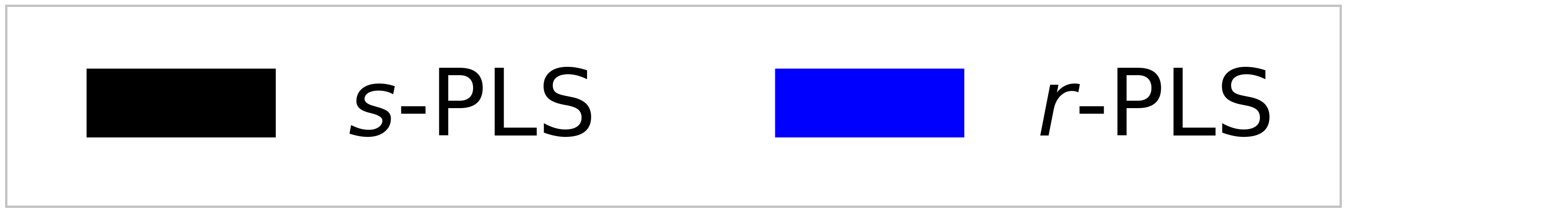}
  \end{flushright}
  \vspace{-8pt}
  \centering
  \includegraphics[scale=0.25]{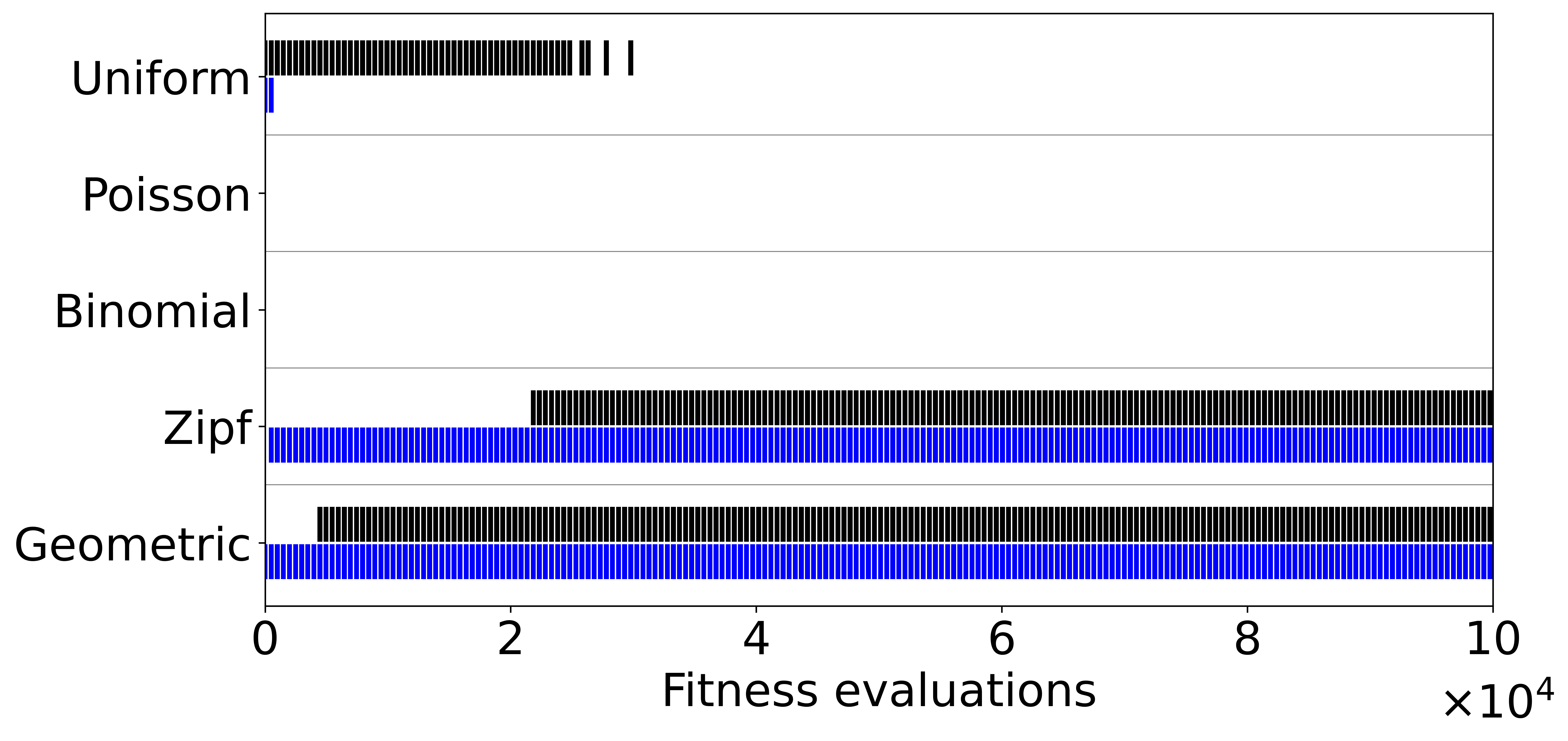}
  \vspace{-10pt}
  \caption{\small Goodness-of-fit of the distributions with respect to the number of good neighbours among solutions in the archive during the search process of \emph{s}-PLS (black) and \emph{r}-PLS (blue) on the TSP (100 cities). A coloured tick in a row indicates that the corresponding algorithm’s data at that point was not rejected under the model. For example, for the geometric distribution model, the blue ticks virtually cover the whole band. This indicates that the number of good neighbours among solutions in the archive for \emph{r}-PLS follows a geometric distribution throughout the search process. Note that an algorithm can fit multiple distribution models.}
  \label{fig:gof_tsp}
  \vspace{-5pt}
\end{figure}

Figure~\ref{fig:gof_tsp} plots the goodness-of-fit of the five discrete distributions with respect to the number of good neighbours among solutions in the archive during the search process of \emph{s}-PLS and \emph{r}-PLS on the TSP. 
Results on the other three problems can be seen in Appendix~\ref{apx:dist}.
In the figures, each horizontal band corresponds to a candidate distribution, and at each sampled evaluation, two coloured ticks, black and blue, respectively indicate a good fit for \emph{s}-PLS and \emph{r}-PLS at $\alpha=0.05$, namely, the $\chi^2$ test does not reject the distribution.
As shown in Figure~\ref{fig:gof_tsp}, \emph{s}-PLS initially follows a uniform distribution, but for most of its search process, it follows a geometric distribution. In contrast, \emph{r}-PLS exhibits behaviour consistent with both Zipf and geometric distributions. Note that multiple distributions can provide statistically acceptable fits. To further compare Zipf and geometric, we use Akaike’s Information Criterion (AIC)~\cite{akaike1974aic}, which considers both the likelihood of the fit and model complexity (i.e., it penalises distributions with more free parameters). Among the two, the geometric distribution consistently yields a lower AIC, indicating a better fit (see Appendix~\ref{apx:dist} for details).
In summary, the geometric distribution best describes the behaviour of both \emph{s}-PLS and \emph{r}-PLS on all the considered problems.

The geometric distribution is defined by the probability mass function $\Prob(G=k)=p(1-p)^k,$ where $k\in\mathbb{N}$ and $p\in(0,1]$.
It is light-tailed, meaning that its probability mass decays exponentially as $k$ increases. 
Unlike uniform distribution (e.g., Example 2 in the previous section) where distribution of the number of good neighbours among solutions is even, the light-tailed yields a highly uneven distribution -- most solutions have only a few good neighbours, while those with many good neighbours exist but are rare.

\section{Theoretical Analysis}

The above result demonstrates that geometric distribution model is the fittest one for the two PLS algorithms during their search process. In this section, building on the previous finding, we aim to provide theoretical explanation for why \emph{r}-PLS is more efficient than \emph{s}-PLS. That is, we analytically compare the expected time of the two algorithms to find a new good neighbour, under the same geometric distribution (i.e., the same parameter $p$ value) that the number of good neighbours among the solutions in the archives follows. 
Note that despite both algorithms fitting the same distribution model, they are likely to follow different geometric distributions (i.e., different $p$ values) at a specific timestep. 
That said, considering the same distribution allows us to analyse their efficiency at an identical search state, hence the performance difference can be attributed solely to the algorithms' neighbourhood exploration approaches, i.e., random sampling versus systematic exploration.

We now begin with a lemma, which will be used in the proof for \emph{s}-PLS in the main theorem.

\begin{lemma}\label{lem:position} 
Suppose a set of $N$ ($N \in \mathbb{Z}^+$) solutions contains $k$ good ones ($1 \leq k\leq N$). Assume the solutions in the set are visited one by one in a random order until a good solution is found. Let $J$ denote the number of solutions visited. Then, the expected value of $J$ is $\E[J] = \frac{N+1}{k+1}$.

\end{lemma}

\begin{proof}
This process is equivalent to sampling uniformly at random without replacement from $N$ items, among which $k$ items are labelled. The number of items examined until encountering the first labelled one is known to have expected value $\E[J] = \frac{N+1}{k+1}$~\cite[Sec.~3.4.1]{knuth1997}.
\end{proof}
This lemma indicates that to first find a good neighbour of a solution, the expected number of its neighbours that \emph{s}-PLS needs to visit is $\frac{|\N|+1}{k+1}$, where $\N$ denotes the neighbourhood size of the problem and $k$ denotes the number of good neighbours the solution has.

\begin{theorem}[\emph{r}-PLS is faster than \emph{s}-PLS in finding the next good solution] \label{thn:geometric}
Let the number of good neighbours among solutions in the archive of \emph{s}-PLS and \emph{r}-PLS follows a geometric distribution with parameter $p\in(0,1]$, i.e., $\Prob(G=k)=p(1-p)^k,\;k\in\mathbb{N}.$
Then, for any neighbourhood size $|\N|\in\mathbb{Z}^+$, \emph{r}-PLS has less expected time (fewer evaluations) than \emph{s}-PLS in finding the next good solution.
\end{theorem}

\begin{proof}

For \emph{r}-PLS, at each iteration, a random neighbour of a randomly selected solution from the archive is sampled. Here, the probability of finding a good neighbour is given by $p_{\text{success}}=\frac{\E[G]}{|\N|}$, where $\E[G]$ represents the expected number of good neighbours.
Since $G$ follows a geometric distribution, $\E[G]=\frac{1-p}{p}$~\cite{johnson2005}.
Thus, the expected number of evaluations for \emph{r}-PLS to find the next good neighbour is:
\noindent
\footnotesize{
\vspace{-2pt}
\begin{align}
  \E[T_{r\text{-PLS}}]=\frac{1}{p_{\text{success}}}=\frac{|\N|}{\E[G]}=\frac{p|\N|}{1-p}.
\end{align}
\vspace{-2pt}
}

\normalsize
Now let us consider \emph{s}-PLS. 
We first compute the expected number of evaluations required when exploring a single solution from the archive.
There are two possible cases.
The first case is that the solution has no good neighbour; this occurs with probability $\Prob(G=0)=p(1-p)^0=p$ and
requires the algorithm to scan the entire neighbourhood, so the expected number of evaluations is $p|\N|$. 
The second case is that the solution contains $k\geq1$ good neighbours; this occurs with probability $\Prob(G=k)=p(1-p)^k$ and it requires the algorithm to visit $\frac{|\N|+1}{k+1}$ neighbours on average (Lemma~\ref{lem:position}), so the expected evaluation number is $p(1-p)^k \cdot\frac{|\N|+1}{k+1}$
Thus, the overall expected number of evaluations to explore a solution is as follows. 
\noindent
\footnotesize{
\vspace{-3pt}
\[
\begin{split}
\E[T^*] &= p|\N|+\sum_{k=1}^{|\N|}p(1-p)^k\cdot\frac{|\N|+1}{k+1} \\
&= p|\N|+p(|\N|+1)\sum_{j=2}^{|\N|+1}\frac{(1-p)^{j-1}}{j} \quad (\text{let }j=k+1)  \\
&= p|\N|+p(|\N|+1)\cdot\frac{1}{1-p}\bigl(\sum_{j=1}^{|\N|+1}\frac{(1-p)^j}{j} - (1-p)\bigr) \\
&= p|\N|+p(|\N|+1)\cdot\frac{1}{r}\sum_{j=1}^{|\N|+1}\frac{r^j}{j} - \frac{r}{r}p(|\N|+1) \\
&= \frac{p|\N|+1}{1-p}\sum_{j=1}^{|\N|+1}\frac{(1-p)^j}{j} - p
\end{split}
\]
\vspace{-2pt}
}

\normalsize
Now that we have the expected number of evaluations required to explore a single solution, the total expected number of evaluations is obtained by multiplying this value by the expected number of solutions explored from the archive. 
Since the probability that a randomly selected solution has at least one good neighbour is $1-P(G=0)=1-p$, the expected number of solutions explored is $\frac{1}{1-p}$.
Thus, the expected number of evaluations for \emph{s}-PLS to find the next good neighbour is:
\noindent
\footnotesize{
\vspace{-2pt}
\begin{align}
\begin{split}
\E[T_{s\text{-PLS}}]&=\E[T^*]\cdot\frac{1}{1-p}  \\
&= \frac{p(|\N|+1)}{(1-p)^2}\sum_{j=1}^{|\N|+1}\frac{(1-p)^j}{j} - \frac{p}{1-p}.
\end{split}
\end{align}
\vspace{-2pt}
}

\normalsize
We now compare the number of evaluations required by \emph{s}-PLS and \emph{r}-PLS via examining their ratio.

\footnotesize{
\vspace{-2pt}
\begin{align}
\begin{split}
\frac{\mathbb{E}[T_{\text{s-PLS}}]}{\mathbb{E}[T_{\text{r-PLS}}]}
&= \frac{\dfrac{p(|\mathcal{N}| + 1)}{(1 - p)^2} \sum_{j=1}^{|\mathcal{N}| + 1} \frac{(1 - p)^j}{j} - \dfrac{p}{1 - p}}{\dfrac{p|\mathcal{N}|}{1 - p}} \\
&= \frac{(|\mathcal{N}| + 1) \sum_{j=1}^{|\mathcal{N}| + 1} \frac{(1 - p)^j}{j} - (1 - p)}{|\mathcal{N}| (1 - p)} \\
&= \frac{|\mathcal{N}|(1 - p) + (|\mathcal{N}| + 1)\sum_{j=2}^{|\mathcal{N}| + 1} \frac{(1 - p)^j}{j}}{|\mathcal{N}|(1 - p)} \\
&= 1 + \frac{|\mathcal{N}| + 1}{|\mathcal{N}|}\cdot\frac{\sum_{j=2}^{|\mathcal{N}| + 1} \frac{(1 - p)^j}{j}}{1 - p} > 1
\label{Eq:main_ratio}
\end{split}
\end{align}
}
\vspace{-2pt}
\end{proof}

The above proof shows that \emph{r}-PLS requires fewer expected evaluations than \emph{s}-PLS to find a good solution. 
This result does not distinguish between the best-improvement and first-improvement strategies within \emph{s}-PLS, as it focuses solely on identifying a good solution the algorithm first encounters.

Further looking at the ratio of the expected evaluations of the two algorithms $\frac{\E[T_{s\text{-PLS}}]}{\E[T_{r\text{-PLS}}]}$, the last part of the second term $\frac{\sum_{j=2}^{|\N|+1}\frac{(1-p)^j}{j}}{1-p}$ can be rewritten as $\sum_{j=2}^{|\N|+1}\frac{(1-p)^{(j-1)}}{j}$. This value increases with increasing $(1-p)$ (decreasing $p$), indicating that when good neighbours are abundant, the difference between \emph{s}-PLS and \emph{r}-PLS becomes even greater. 
This result confirms the earlier observations that the performance gap between \emph{r}-PLS and \emph{s}-PLS is more pronounced in the early stages of the search (see Figures~\ref{Fig:hv_2}), when it is easier to find good neighbours around the current solutions.

Noted that in single-objective optimisation, there is a trade-off between randomised and systematic LS, called Best-from-Multiple-Selections (BMS)~\cite{Cai2015balance}. BMS extends the idea of random neighbour sampling by drawing $k$ neighbours of a solution and retaining the best.
Under the probabilistic model adopted here, BMS acts as a natural generalisation of \emph{r}-PLS (which corresponds to $k=1$). The expected number of evaluations of BMS becomes $\E[T_{r-\text{PLS}}]\cdot(1+O(k/|\N|))$ (see Appendix~\ref{apx:bms}), slightly worse than that of \emph{r}-PLS. 

It is worth pointing out that the theorem presented above is a \emph{time-to-next} comparison. \emph{r}-PLS, under the same distribution of the number of good neighbours, is always faster than \emph{s}-PLS in finding a new good solution. It is not an \emph{end-of-run} analysis and thus does not imply that the final solutions obtained by \emph{r}-PLS -- once the search terminates (under a sufficient time) -- are necessarily better than those produced by \emph{s}-PLS.
In fact, as the search progresses, the neighbourhood structure of solutions continues to change, and good neighbours become increasingly scarce. Since \emph{r}-PLS finds good solutions more quickly, the underlying distribution may shift more rapidly (i.e., an increasing $p$ in the geometric distribution), hence potentially making it more difficult to discover a new good solution. That said, despite that here we cannot tell which algorithm is better in the end, \emph{r}-PLS is much faster than \emph{s}-PLS in the early stages of the search when the distribution is similar, particularly on large-scale problems, where a larger neighbourhood size $|\N|$ leads to a greater ratio (Eq.~\ref{Eq:main_ratio}). This is echoed by the results on the 200D and 500D problems (see Appendix~\ref{apx:results}).

\section{Conclusion}
This work has demonstrated, both empirically and theoretically, that among two standard multi-objective local search approaches, the randomised variant is more efficient than the systematic one. We thus recommend employing randomised local search in multi-objective optimisation when search efficiency is a priority, especially in large-scale problems.

The paper experimentally found that the number of good neighbours encountered during the PLS search follows a tail distribution (specifically, a geometric distribution) rather than being uniform or Gaussian-like. This observation aligned with the general pattern in many optimisation problems, where higher-quality solutions are rarer -- leading to a decreasing number of good neighbours as the search progresses. However, this may not hold in certain pseudo-Boolean benchmarks~\cite{Liang2025}, such as OneMinMax~\cite{giel2006effect}, LOTZ~\cite{Laumanns2002}, OJZJ~\cite{doerr2021theoretical}. In such cases, systematic PLS may be faster.

\bibliography{aaai2026}
\bibliographystyle{plain}

\appendix
\section{}
\subsection{Multi-Objective Combinatorial Problems}\label{apx:mocop}

We consider four MOCOPs, the multi-objective 0/1 knapsack ~\cite{teghem_multi-objective_1994}, travelling salesman problem (TSP)~\cite{ribeiro2002study}, quadratic assignment problem (QAP)~\cite{knowles2003instance} and NK-landscapes~\cite{Aguirre2004}. 
Each problem was instantiated in three sizes (100, 200 and 500 variables). 

\vspace{4pt}
\noindent\textbf{Multi-Objective 0-1 Knapsack Problem (Knapsack).}
The multi-objective 0-1 knapsack problem~\cite{teghem_multi-objective_1994} is a widely studied MOCOP.  
Given a set of $D$ items $\x = (x_1,x_2,\dots,x_D)\in \{0,1\}^D$, the $m$-objective problem is defined as the following.
\begin{align}
\begin{split}
\max~f_{j}{(x)} = \sum_{i=1}^{D}v_{ji}x_{i}, \; j=1,\dots,m  \quad
\text{s.t.}~ \sum_{i=1}^{D}w_{i}x_{i} \leq c 
\label{eq:Knapsack}
\end{split}
\end{align}
Here, $v_{ji} \geq 0$ is the value of the item $i$ in objective $j$, $w_{i}$ is the item's weight, and $c = \frac{1}{2}\sum w_i$ is the capacity.  
Following~\cite{li_empirical_2024}, both $v_{ji}$ and $w_i$ are sampled uniformly from $\{10,11,\dots,100\}$.

\vspace{4pt}
\noindent
\textbf{Multi-Objective Travelling Salesman Problem (TSP)}.
The multi-objective TSP extends the classical TSP, 
with multiple costs between each pair of cities~\cite{ribeiro2002study}, and aims to find the route minimising multiple travelling costs for visiting all the cities exactly once, returning to the start.
Formally, 
given a network $L=(V,C)$, where $V=\{v_{1},\dots,v_{D}\}$
is a set of $D$ nodes and $C=\{C_{j}: j\in \{1,\dots,m\}\}$ is a set of $m$ cost matrices
between nodes $(C_{j}: V \times V)$, the problem is to find the Pareto optimal set of Hamiltonian
cycles that minimise each of the $m$ cost objectives.

\vspace{4pt}
\noindent\textbf{Multi-Objective Quadratic Assignment Problem (QAP).}
The multi-objective QAP~\cite{knowles2003instance} models facility-location assignments with multiple flow types.  
Given $m$ cost matrices $[C_{1,i,j}],\dots,[C_{m,i,j}]$ and a distance matrix $[L_{u,v}]$, a solution is a permutation $x = (x_1, ..., x_D)$ where $x_i$ denotes the location of facility $i$.  
The problem is defined as the following.
\begin{align}
\min~f_{k}(x) = \sum_{i=1}^{D}\sum_{j=1}^{D}C_{k,i,j}L_{x_i,x_j},\quad k=1,\dots,m 
\label{eq:QAP}
\end{align}

\vspace{4pt}
\noindent\textbf{Multi-Objective NK-Landscape.}
The multi-objective NK-landscapes~\cite{Aguirre2004} are widely used due to their tunable ruggedness~\cite{Verel2013}.
Here, $N$ represents the length of the bit-string and $K$ represents the epistasis degree (i.e., each variable is influenced by $K$ other variables, collectively referred to as its locus).
For consistency, we denote the length of the bit-string as $D$.
Then, in a $m$-objective NK-landscape problem, each objective $f_j$ is defined as:
\begin{align}
\begin{split}
\max f_j(x)=\frac{1}{D}\sum_{i=1}^{D}c_{ij}(x_i, x_{k_{ij1}}, ..., x_{k_{ijK}}),\;j=1,\dots,m.
\label{eq:NK-landscape}
\end{split}
\end{align}
\noindent
Where $c_{ij}$ represent the fitness contribution of the $i$-th variable, influenced by $K$ other variables in its locus that collectively decide its contribution to the $j$-th objective. Each $c_{ij}$ depends on the values of the $i$-variable and the variables in its locus, resulting in $2^{K+1}$ possible combinations of input and the corresponding output values. Each output is randomly sampled from $(0,1]$. Following~\cite{Aguirre2007,Daolio2015}, the $K$ other variables of a variable's locus are drawn independently and uniformly at random for each $i$ (variable) and $j$ (objective), resembling a random epistasis pattern.

\begin{figure*}[!ht]
\vspace{-5pt}
\renewcommand{\arraystretch}{0.1} 
\fontsize{8pt}{9.5pt}\selectfont
\begin{center}
    \begin{tabular}{@{}c@{}@{}c@{}@{}c@{}@{}c@{}}
	\includegraphics[scale=0.35]{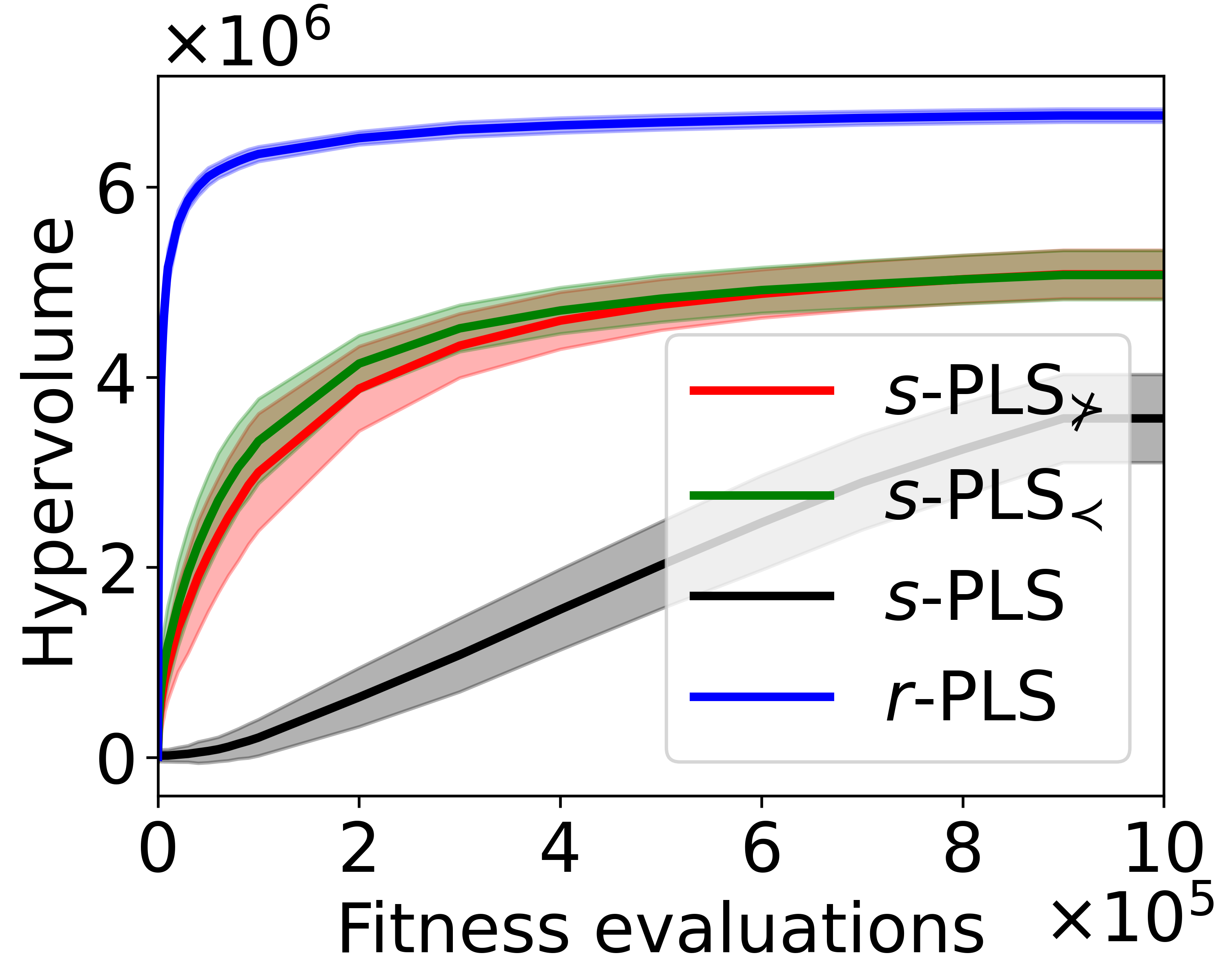} & 
    \includegraphics[scale=0.35] {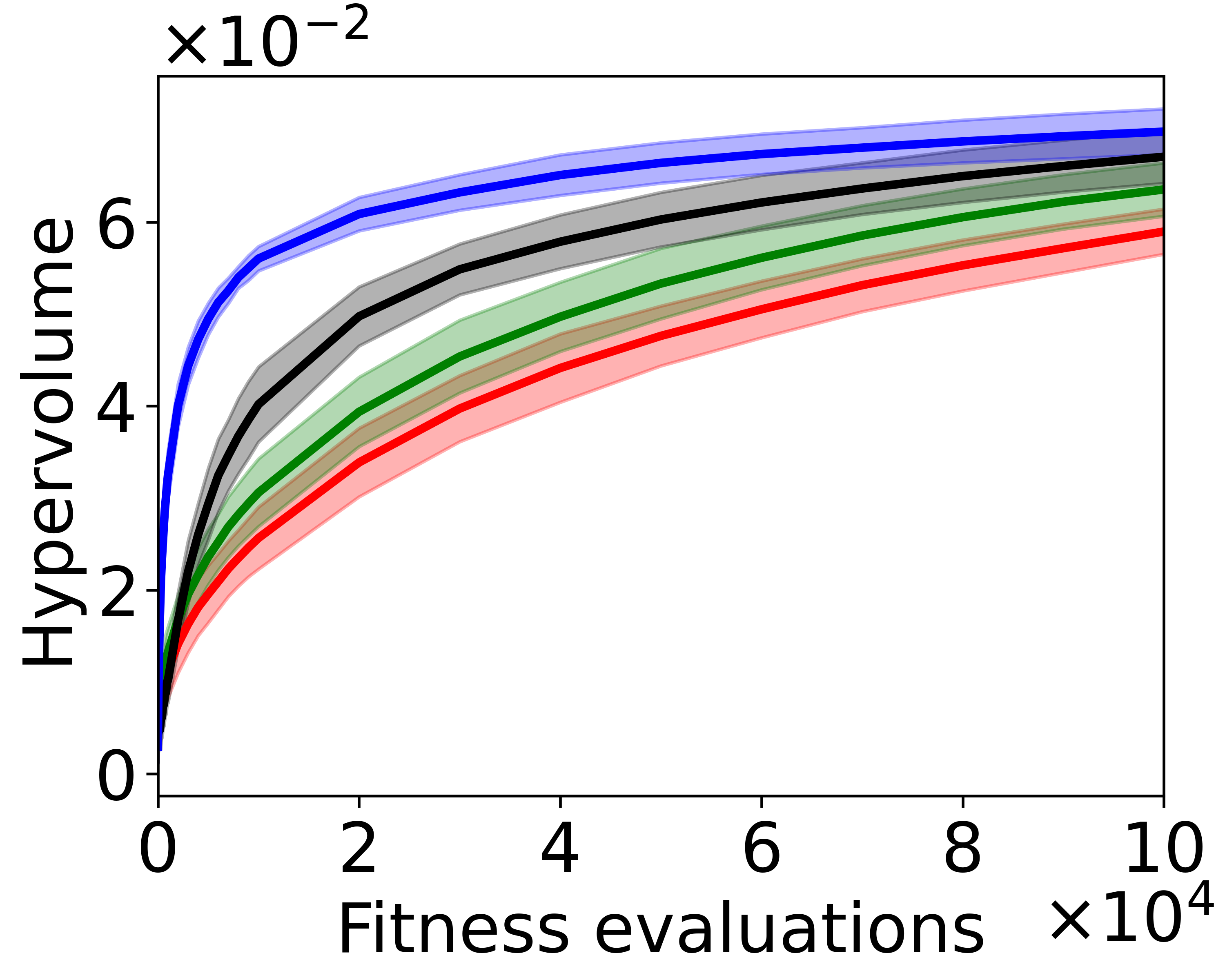} &
    \includegraphics[scale=0.35]{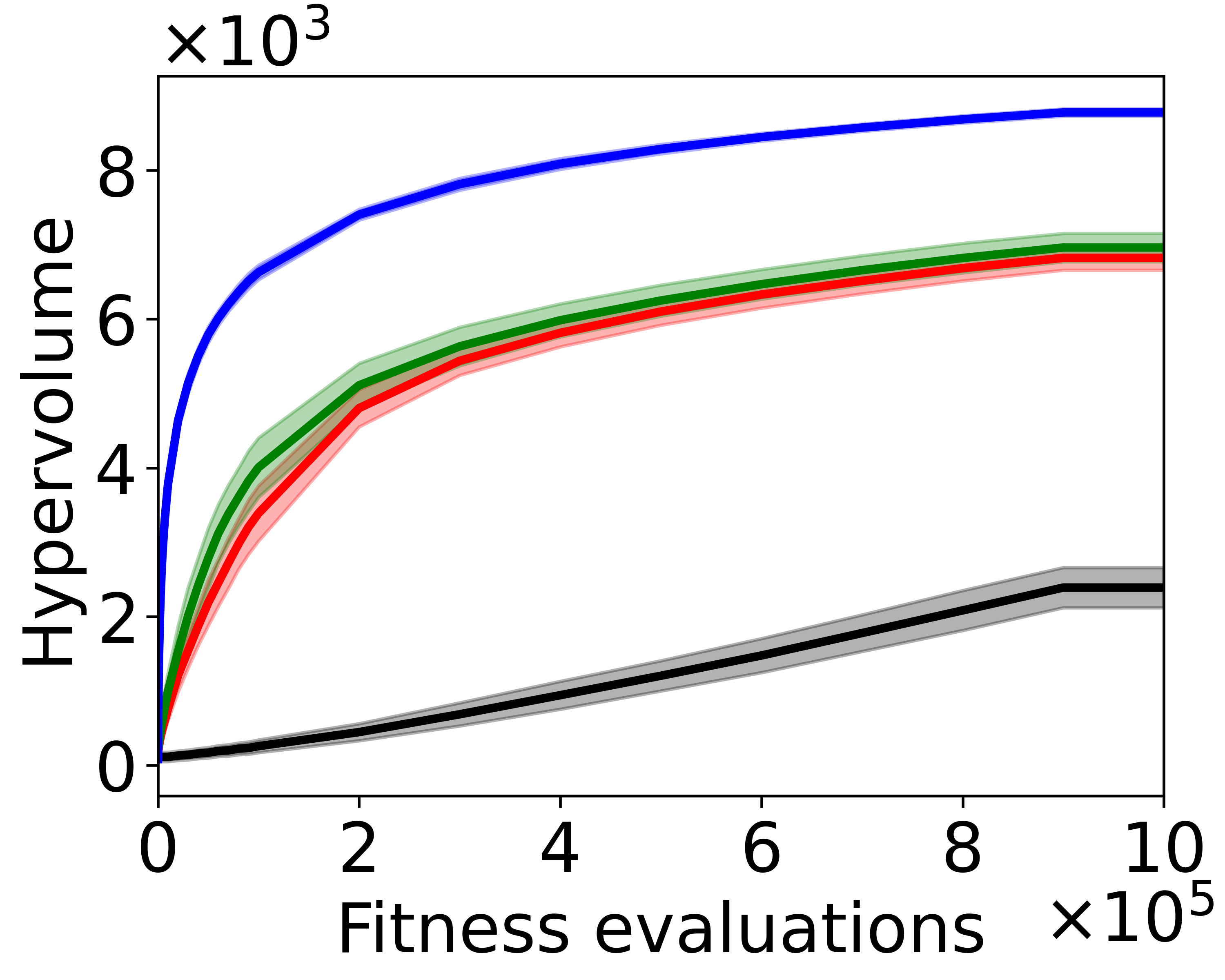}&
	\includegraphics[scale=0.35] {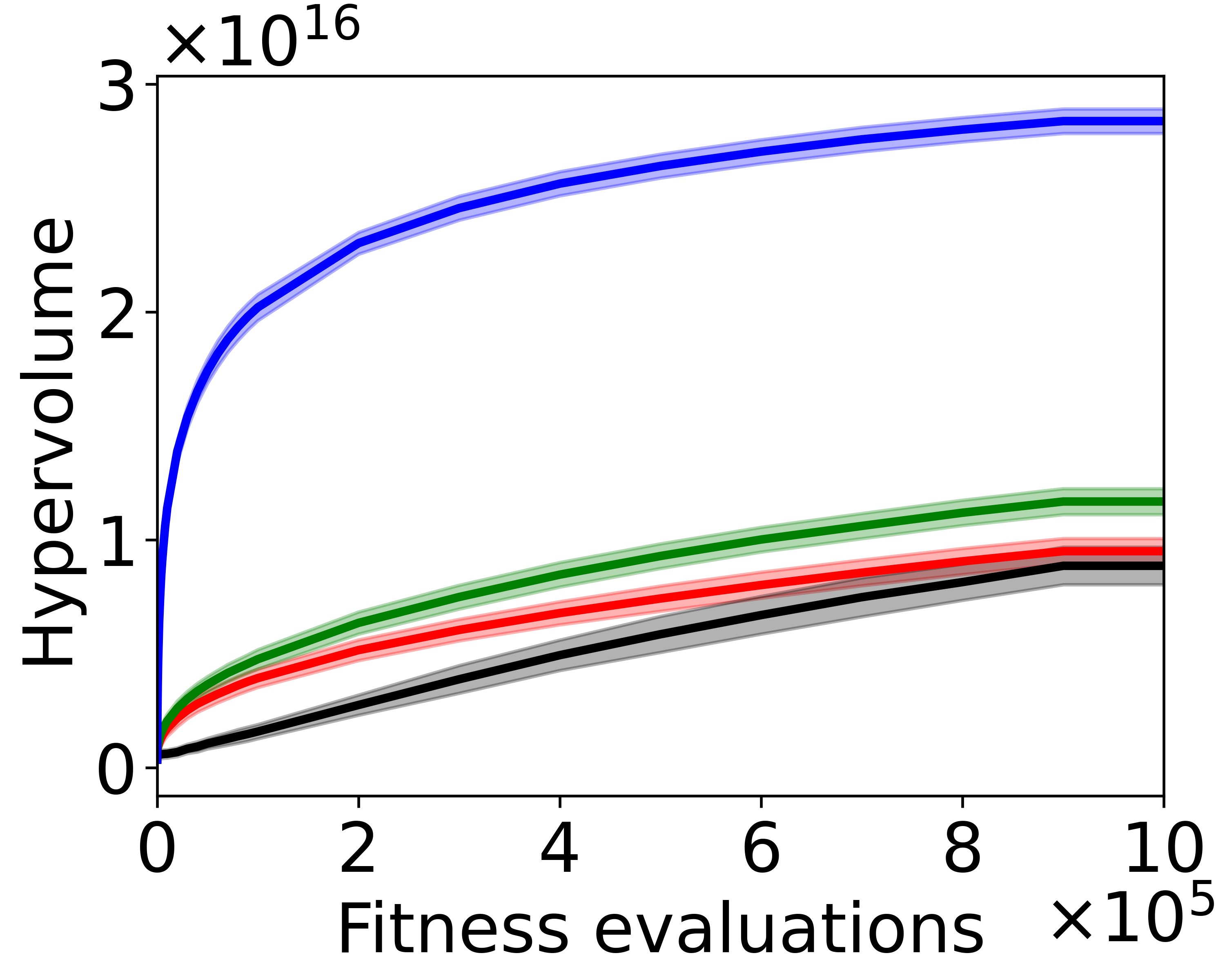} \\
	(a) Knapsack 200D &
    (b) NK-Landscape 200D &
    (c) TSP 200D  &
    (d) QAP 200D  \\	
    \includegraphics[scale=0.35]{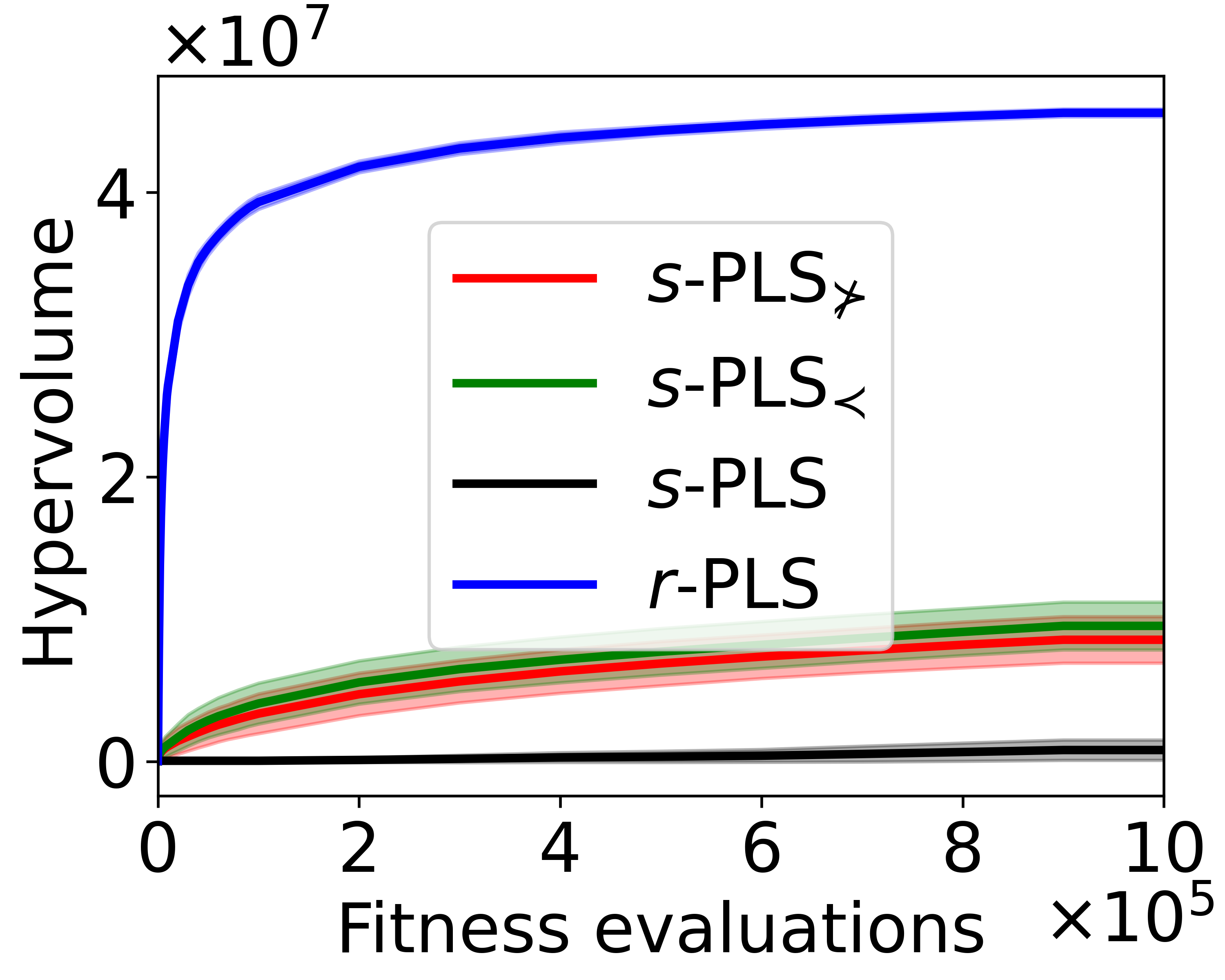} & 
    \includegraphics[scale=0.35] {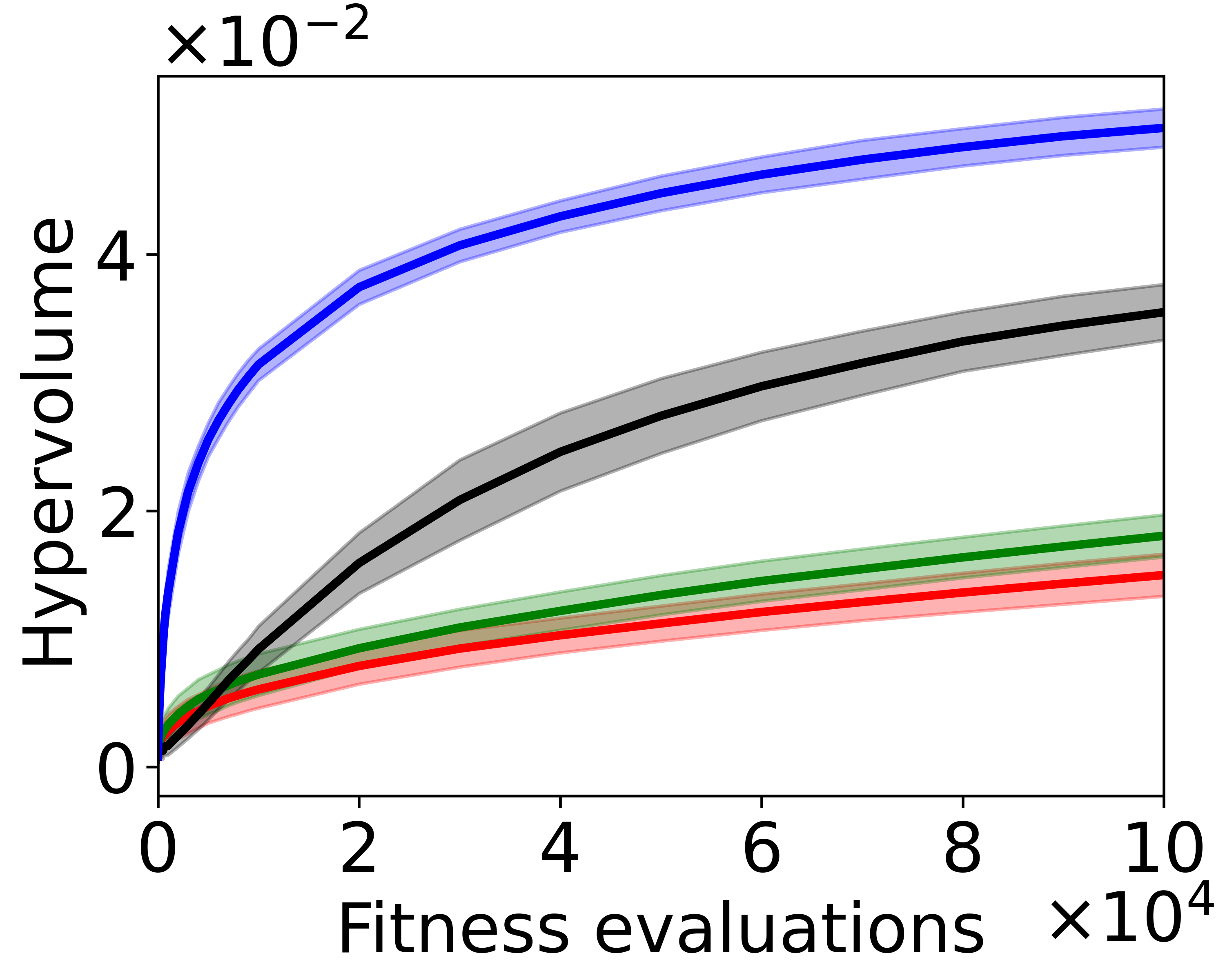}&
    \includegraphics[scale=0.35]{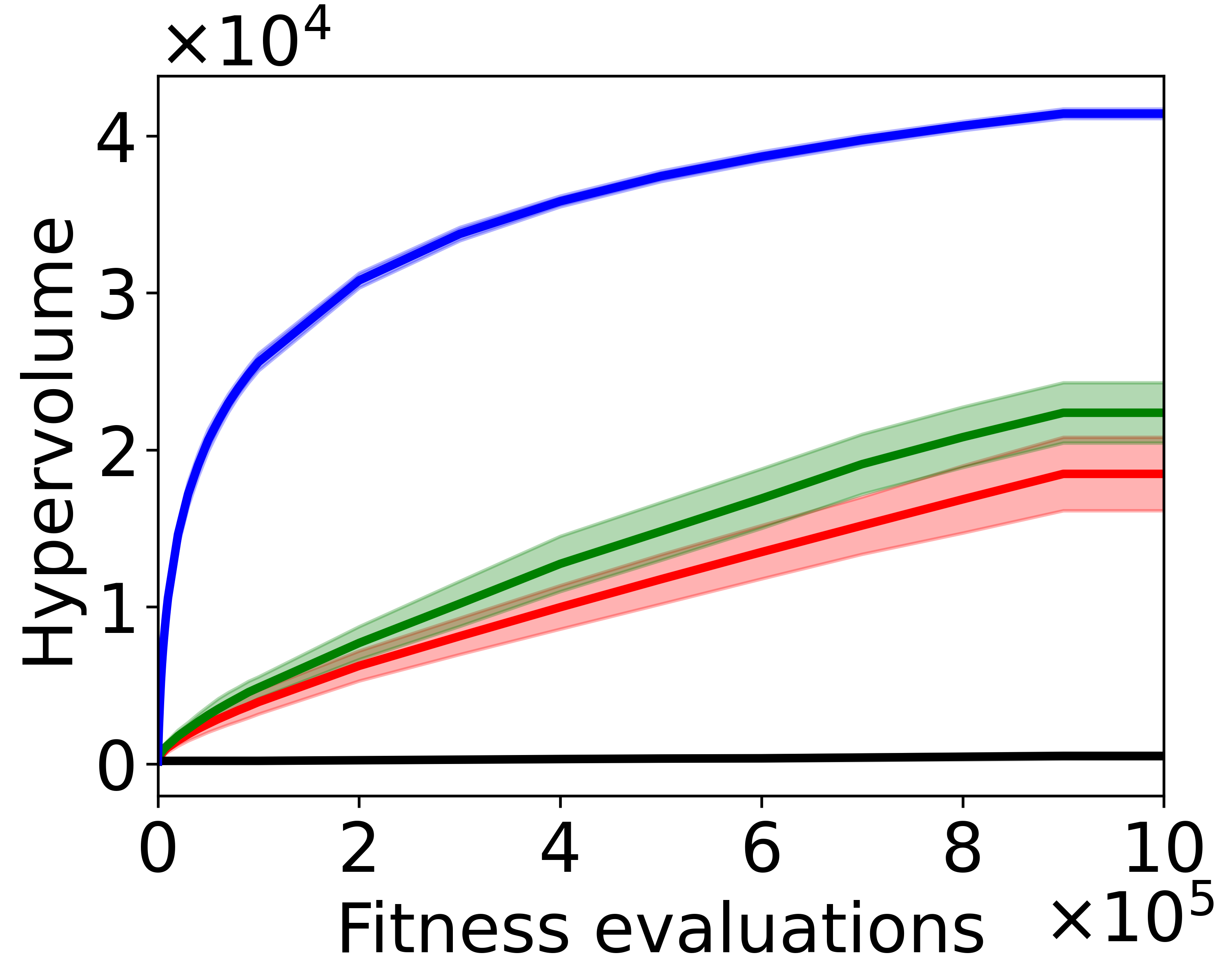}&
	\includegraphics[scale=0.35] {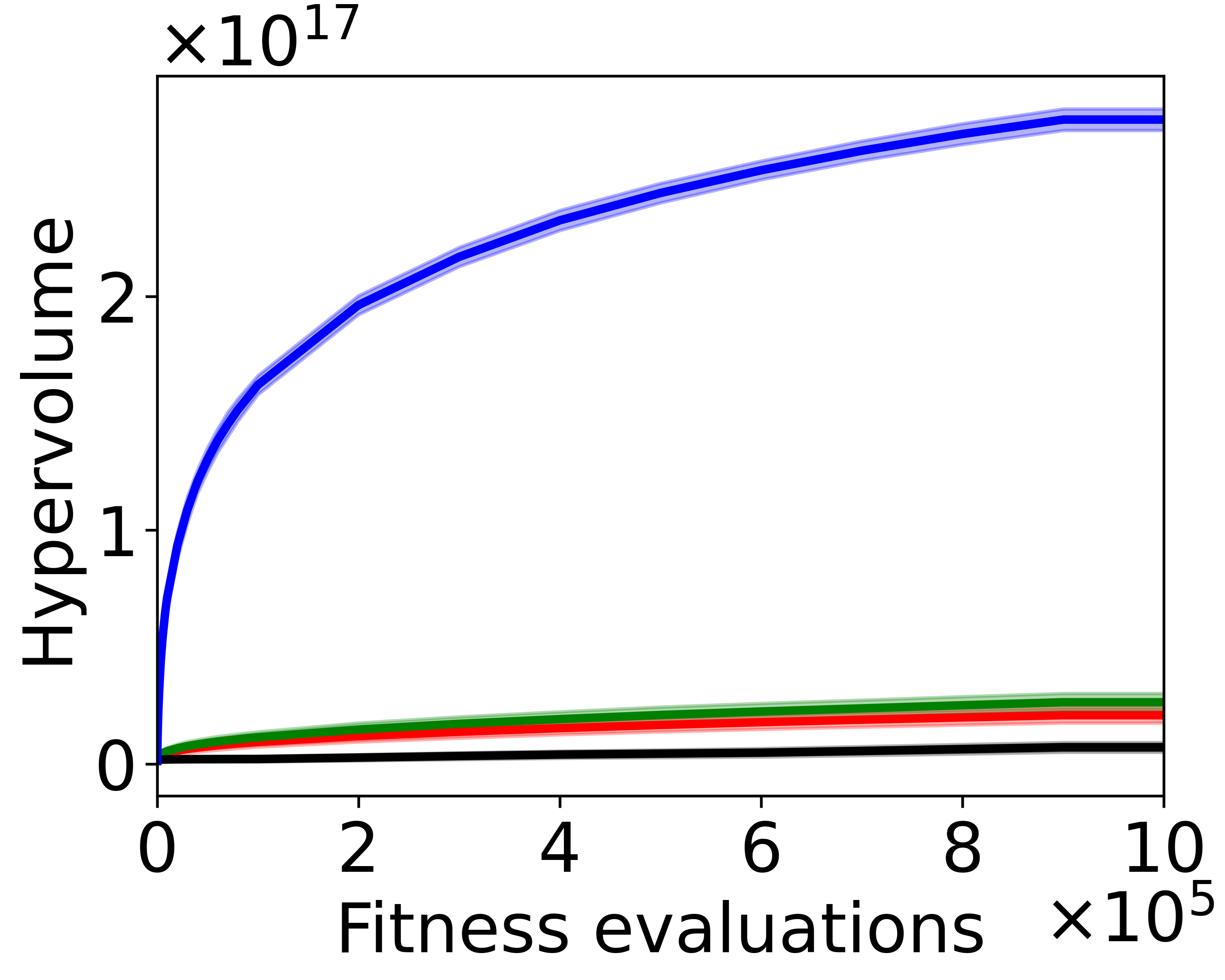}  \\
	(e) Knapsack 500D  &
    (f) NK-Landscape 500D &
    (g) TSP 500D  &
    (h) QAP 500D  \\	
	\end{tabular}
	\end{center}
        \vspace{-10pt}
	\caption{The hypervolume trajectory (higher is better) of the considered \emph{s}-PLS, \emph{s}-PLS$_\nsucc$, \emph{s}-PLS$_\prec$ and \emph{r}-PLS across 30 runs on the four MOCOPs with 200 variables (the top panel) and with 500 variables (the bottom panel). The bolded line and shaded area represent the mean and standard deviation of the hypervolume, respectively. 
    }
	\label{Fig:hv_all}
\end{figure*}

\begin{figure*}[!h]
	\vspace{-5pt}
	\renewcommand{\arraystretch}{0.1} 
	\fontsize{8.5pt}{10pt}\selectfont
	\begin{center}
        \begin{tabular}{@{}c@{}@{}c@{}@{}c@{}@{}c@{}}
			\includegraphics[scale=0.35]{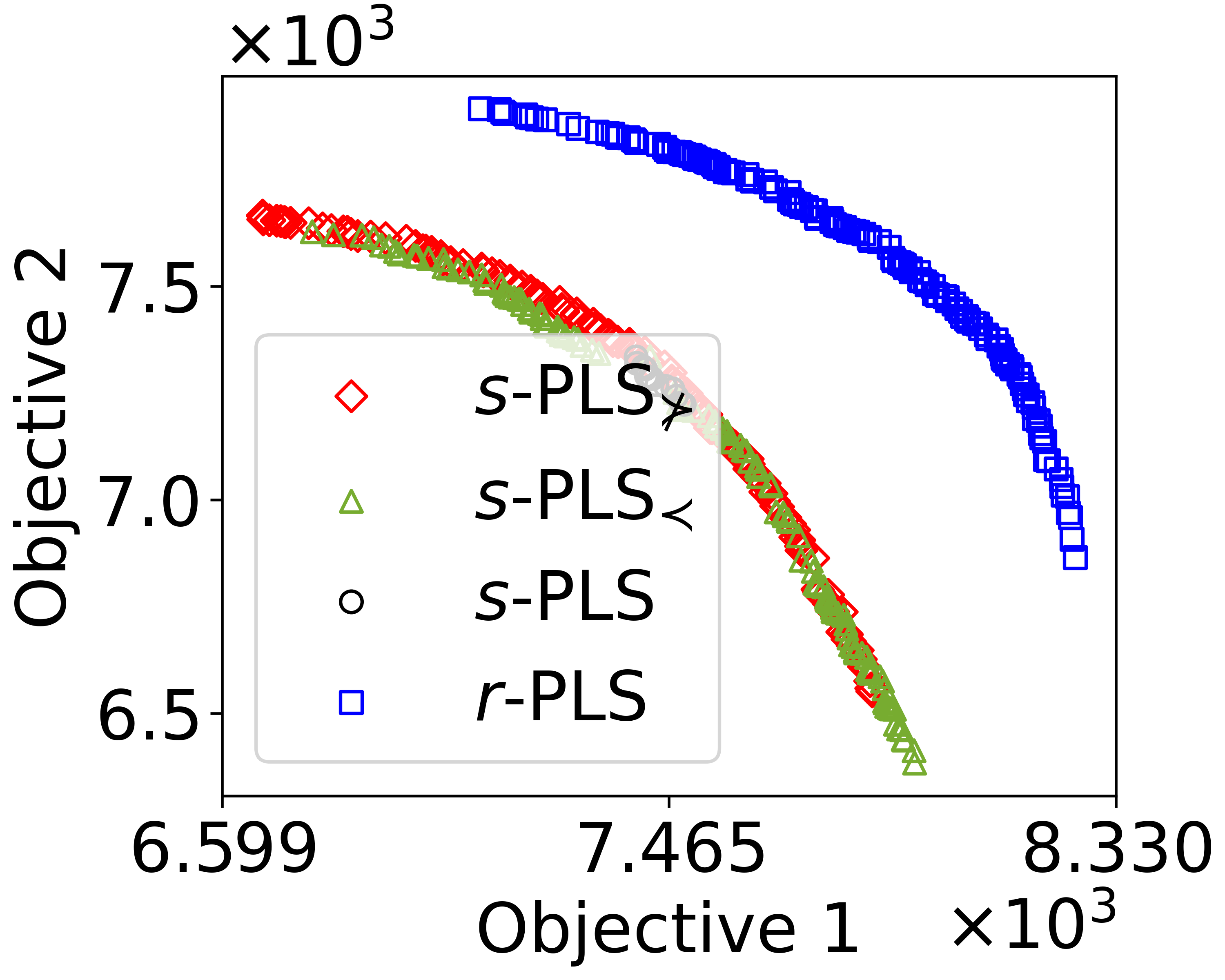} & 
			\includegraphics[scale=0.35] {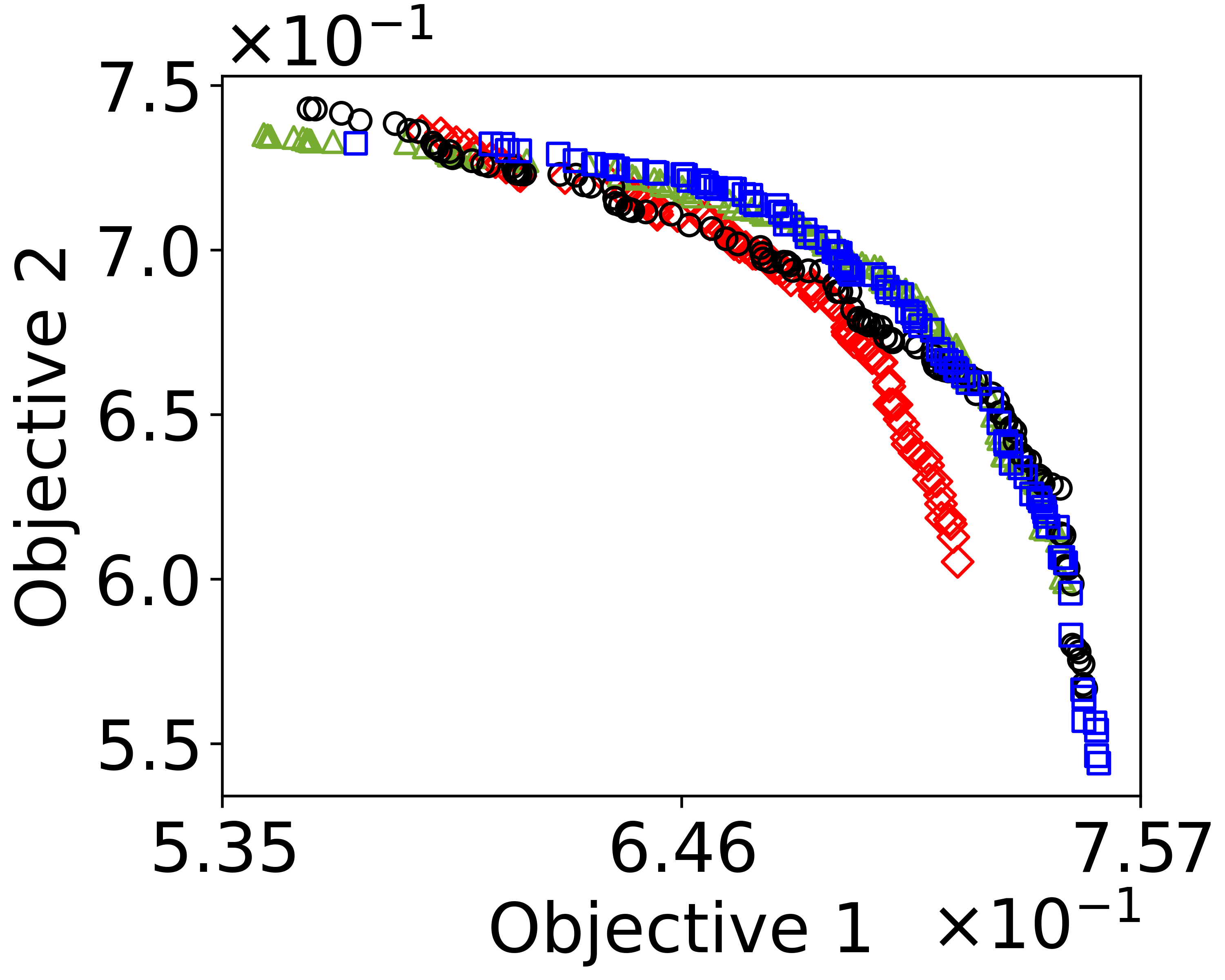}&
			\includegraphics[scale=0.35]{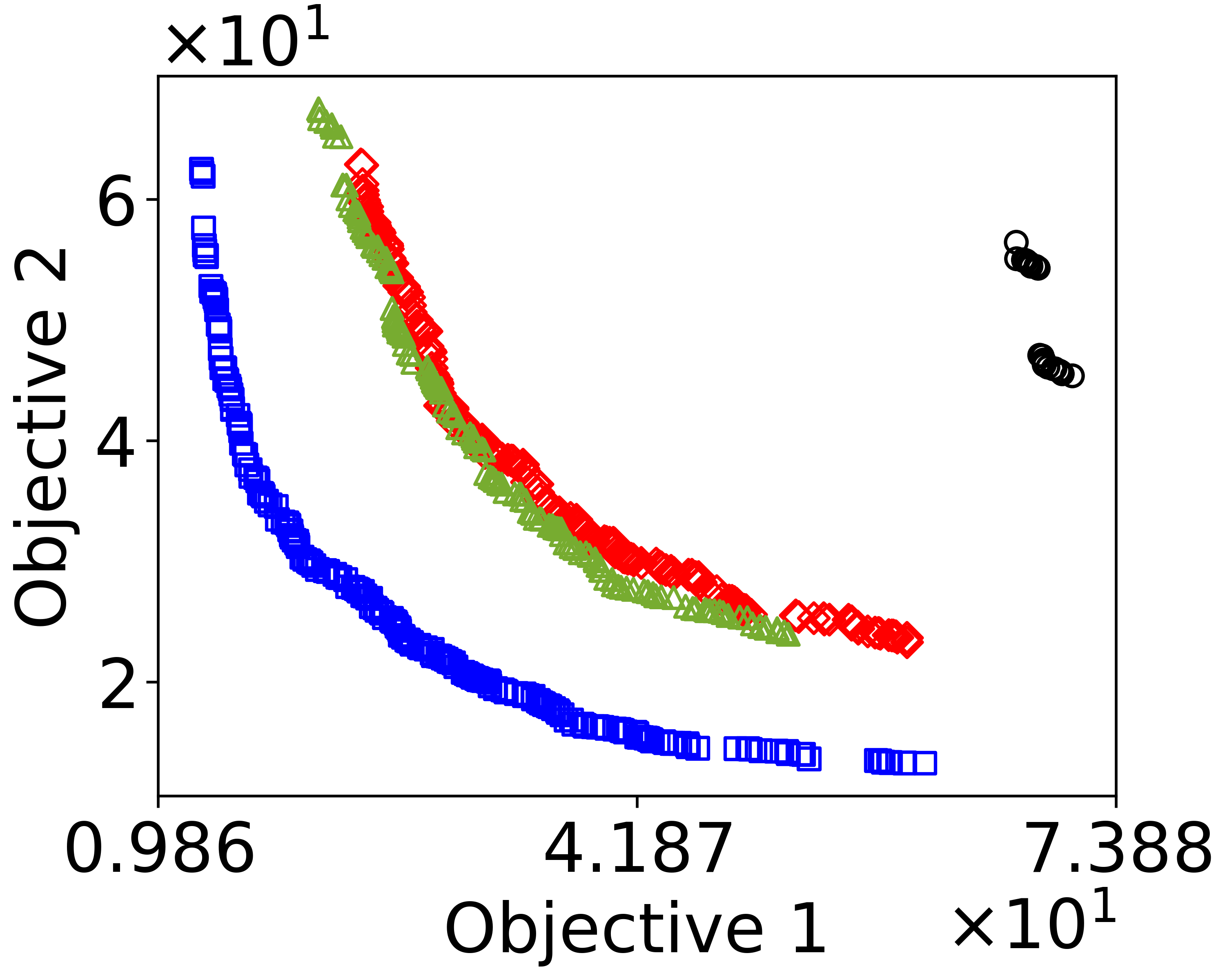}&
			\includegraphics[scale=0.35] {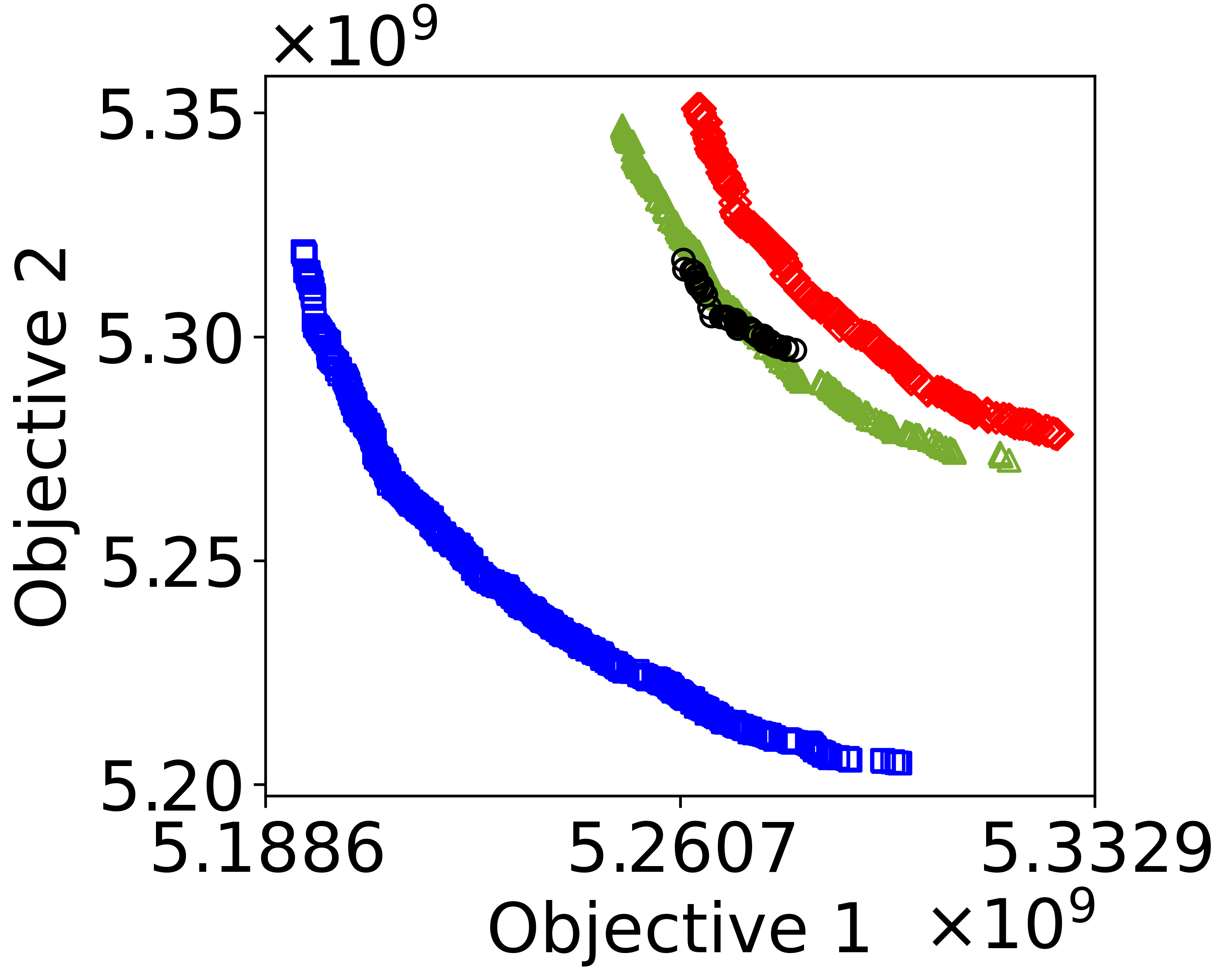}  \\
	(a) Knapsack 200D &
        (b) NK-Landscape 200D &
        (c) TSP 200D  &
        (d) QAP 200D  \\	
			\includegraphics[scale=0.35]{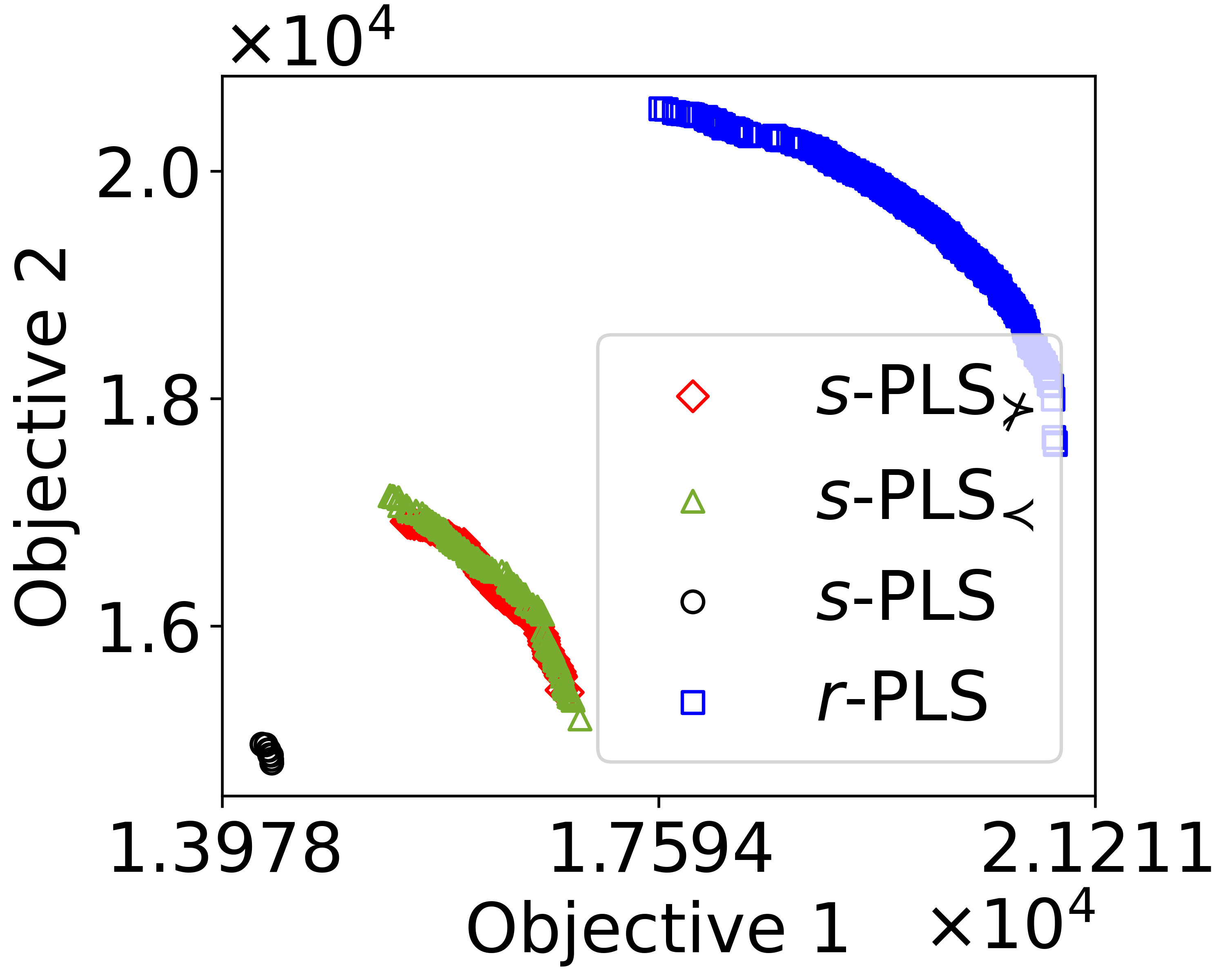} & 
			\includegraphics[scale=0.35] {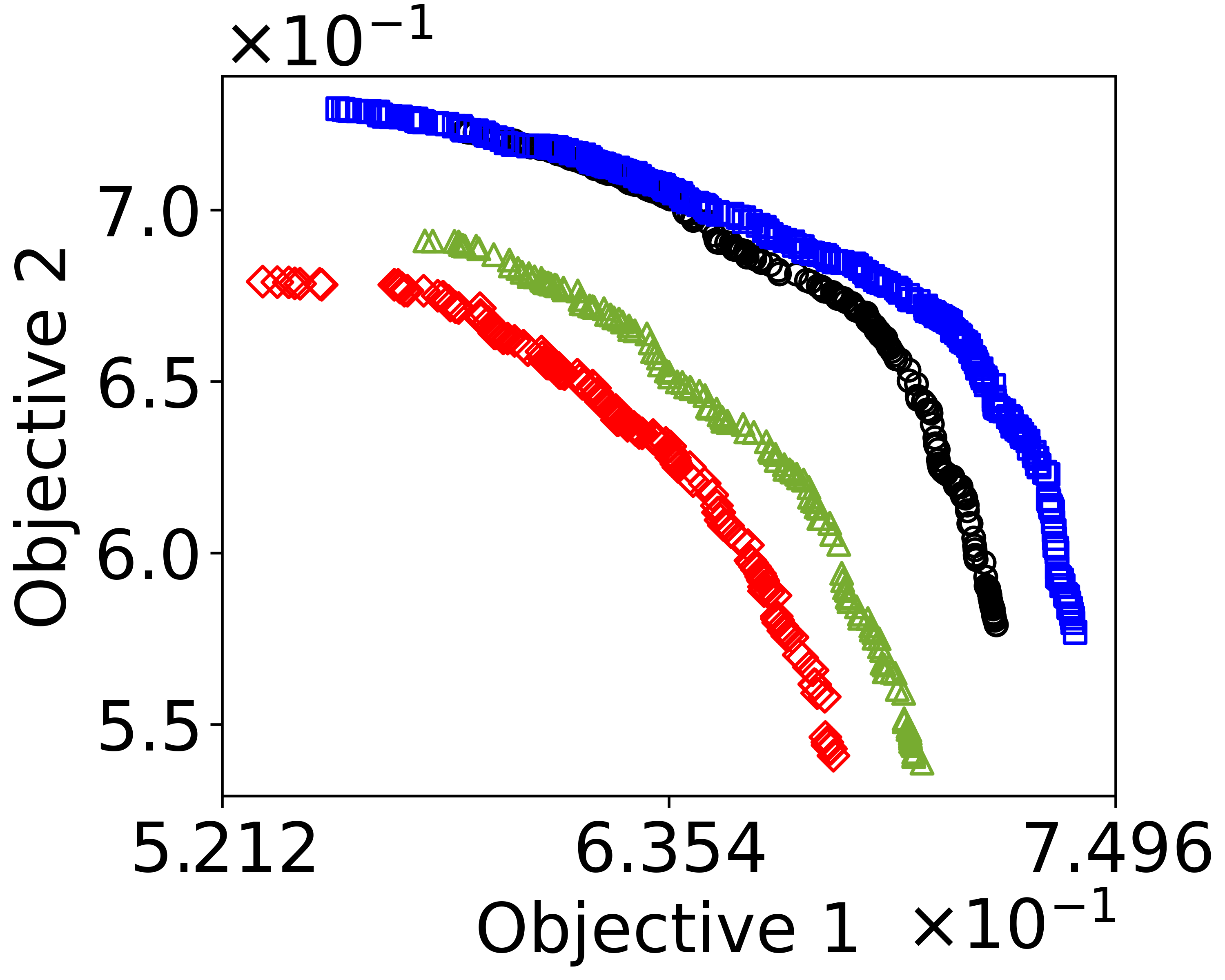} &
			\includegraphics[scale=0.35]{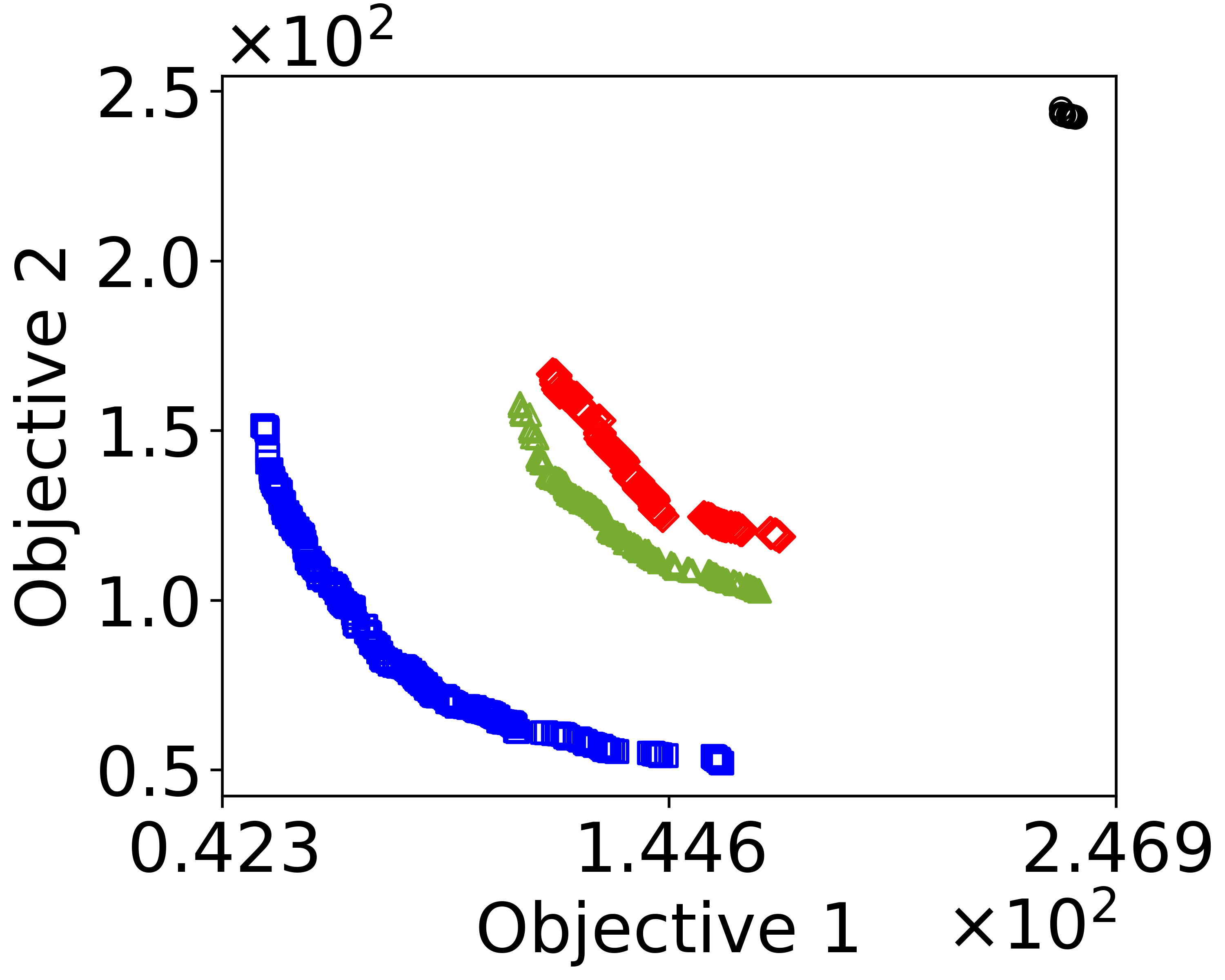}&
			\includegraphics[scale=0.35] {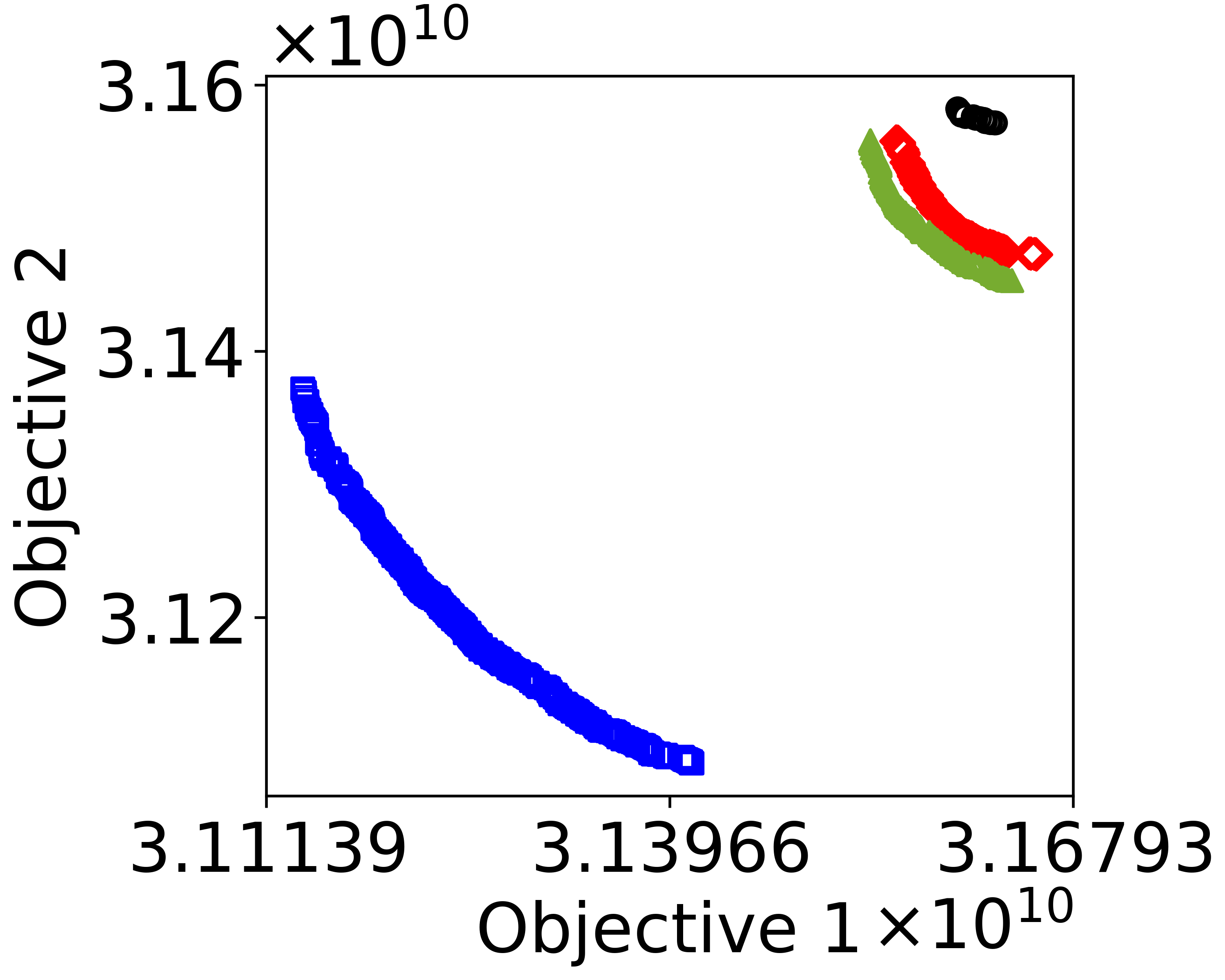} \\
	(e) Knapsack 500D  &
        (f) NK-Landscape 500D &
        (g) TSP 500D  &
        (h) QAP 500D  \\	
		\end{tabular}
	\end{center}
	\vspace{-10pt}
	\caption{All non-dominated solutions (i.e., solutions in the archive) obtained by the considered \emph{s}-PLS, \emph{s}-PLS$_\nsucc$, \emph{s}-PLS$_\prec$ and \emph{r}-PLS in a typical run on the four MOCOPs with 200 variables (the top panel) and with 500 variables (the bottom panel), where the Knapsack and NK-Landscape are maximisation problems, and the TSP and QAP are minimisation problems. 
    }
	\label{Fig:obj}
\end{figure*}

\subsection{Experimental Settings}\label{apx:setting}
In our experiments, for each MOCOP we consider three problem sizes (100, 200, and 500). 
It is important to note that \emph{r}-PLS and \emph{s}-PLS typically use different stopping conditions: \emph{r}-PLS runs until the evaluation budget is exhausted, whereas \emph{s}-PLS may terminate earlier when there are no more unexplored solutions in the archive.
To ensure fairness, all algorithms start from the same initial solution. For some situation, \emph{s}-PLS can terminate rather early, for example, within $10^5$ evaluations. In such cases, we stop \emph{r}-PLS when \emph{s}-PLS terminates. 
All algorithms are implemented in Java using the jMetal platform~\cite{durillo_jmetal_2011}. For each problem instance, each algorithm is executed over 30 independent runs, with each run performed on a separate core of an AMD Ryzen R9-7945HX CPU, with 16~GB of memory allocated per run.

Neighbourhood operators are matched to the encoding of the variables for each problem:  
NK-landscape uses 1-bit flip (binary encoding), TSP uses 2-opt (order-based permutation encoding), and QAP uses 2-swap (position-based permutation encoding)~\cite{eiben2015introduction}.
The knapsack problem requires a 2-bit flip neighbourhood, as 1-bit flip often fails to yield feasible or improved solutions near the constraint boundary, following the practice in this study~\cite{li_empirical_2024}. 

We evaluate performance using the hypervolume (HV) indicator~\cite{Zitzler1999}, with respect to a reference point estimated via random sampling. Specifically, we generate 100,000 random solutions from the decision space and only retain the non-dominated ones. According to \cite{Li2022}, we define the reference point as $r_i = max_i + (max_i-min_i)/10$ for minimisation problems, and $r_i = min_i - (max_i-min_i)/10$ for maximisation problems,
where $max_i$ and $min_i$ denote the maximum and minimum values of the non-dominated set on the $i$-th objective, respectively.
This is more suitable than considering the common practice that uses a reference point determined by the combined non-dominated set of all generated solutions, because in multi-objective combinatorial optimisation different algorithms may exhibit substantially different levels of convergence (i.e., closeness to the Pareto front)~\cite{li_empirical_2024}.

\subsection{Additional Results}\label{apx:results}

This section gives additional experimental results on the four MOCOPs with the problem sizes of 200 and 500.

Figure~\ref{Fig:hv_all} shows the average hypervolume (bolded line) and standard deviation (shaded area) of the basic \emph{s}-PLS and \emph{r}-PLS, as well as the two \emph{s}-PLS variants across 30 independent runs on the four problems.
The top row corresponds to problems with 200 variables, and the bottom row to those with 500 variables. 
Across all problem types, \emph{r}-PLS consistently achieves higher hypervolume values than all \emph{s}-PLS variants throughout the search. It also attains better final performance --
the final HV values of \emph{r}-PLS are significantly higher under the Wilcoxon rank-sum test at a 0.05 significance level.
On the NK-landscape problem, the performance advantage of \emph{r}-PLS is relatively smaller, while on the other three problems, \emph{r}-PLS demonstrates a clear edge since the very beginning of the search process, particularly as the problem size increases. 
It is also worth noting that \emph{s}-PLS$\prec$ and \emph{s}-PLS$\nsucc$ (first-improvement variants) are faster than the original \emph{s}-PLS (best-improvement) on most problems (except for the NK-landscape), which is in line with the previous findings~\cite{Liefooghe2012,Dubois2015}.

Figure~\ref{Fig:obj} presents the non-dominated solutions obtained by the four algorithms in a typical run on the four problems when the search ends in the Figure~\ref{Fig:hv_all}. The layout mirrors that of Figure~\ref{Fig:hv_all}, with smaller instances shown in the top row and larger ones in the bottom row.
For the Knapsack and NK-Landscape problems (maximisation), better solutions are located toward the top-right corner; for the TSP and QAP (minimisation), the desirable region lies in the bottom-left.
As illustrated, \emph{r}-PLS yields solutions with superior convergence and diversity compared to all the \emph{s}-PLS variants. The original \emph{s}-PLS, in particular, often exhibits significantly poorer spread across the objective space than the other algorithms. 

It is worth pointing out that \emph{s}-PLS exhibits poorer diversity on the TSP and QAP than on the NK-Landscape for all problem sizes (Figure~\ref{Fig:obj} here and Figure 2 in the main paper). The systematic neighbourhood exploration of \emph{s}-PLS evaluates all neighbours of a solution -- its effectiveness depends on how diverse those neighbours are. 
On the NK-Landscape, random epistatic interactions~\cite{Aguirre2004} make the objective changes of each bit-flip unpredictable, leading to distinct directions in the objective space even from a single solution. However, such diversity within a neighbourhood do not hold on the TSP and QAP.

The above results demonstrate that \emph{r}-PLS achieves better performance in both convergence and diversity, with the advantage becoming more pronounced as the problem size increases.

\begin{figure*}[!h]
	\vspace{-5pt}
	\renewcommand{\arraystretch}{0.1} 
	\fontsize{8.5pt}{10pt}\selectfont
    \begin{flushright}
    \includegraphics[scale=0.28]{figures/gof_legend.png}
    \end{flushright}
	\begin{center}
        \begin{tabular}{@{}c@{}@{}c@{}}
			\includegraphics[scale=0.3]{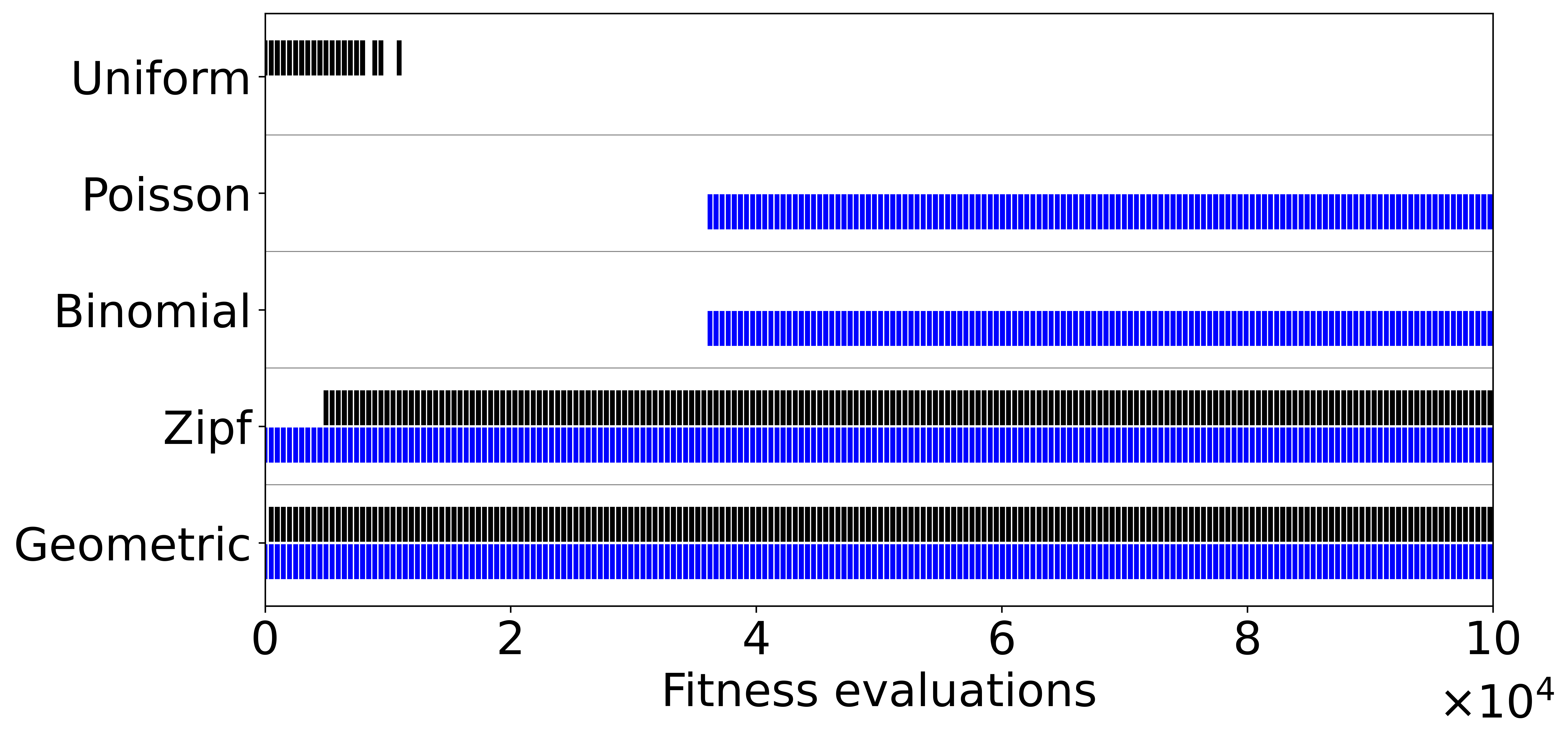} & 
			\includegraphics[scale=0.3] {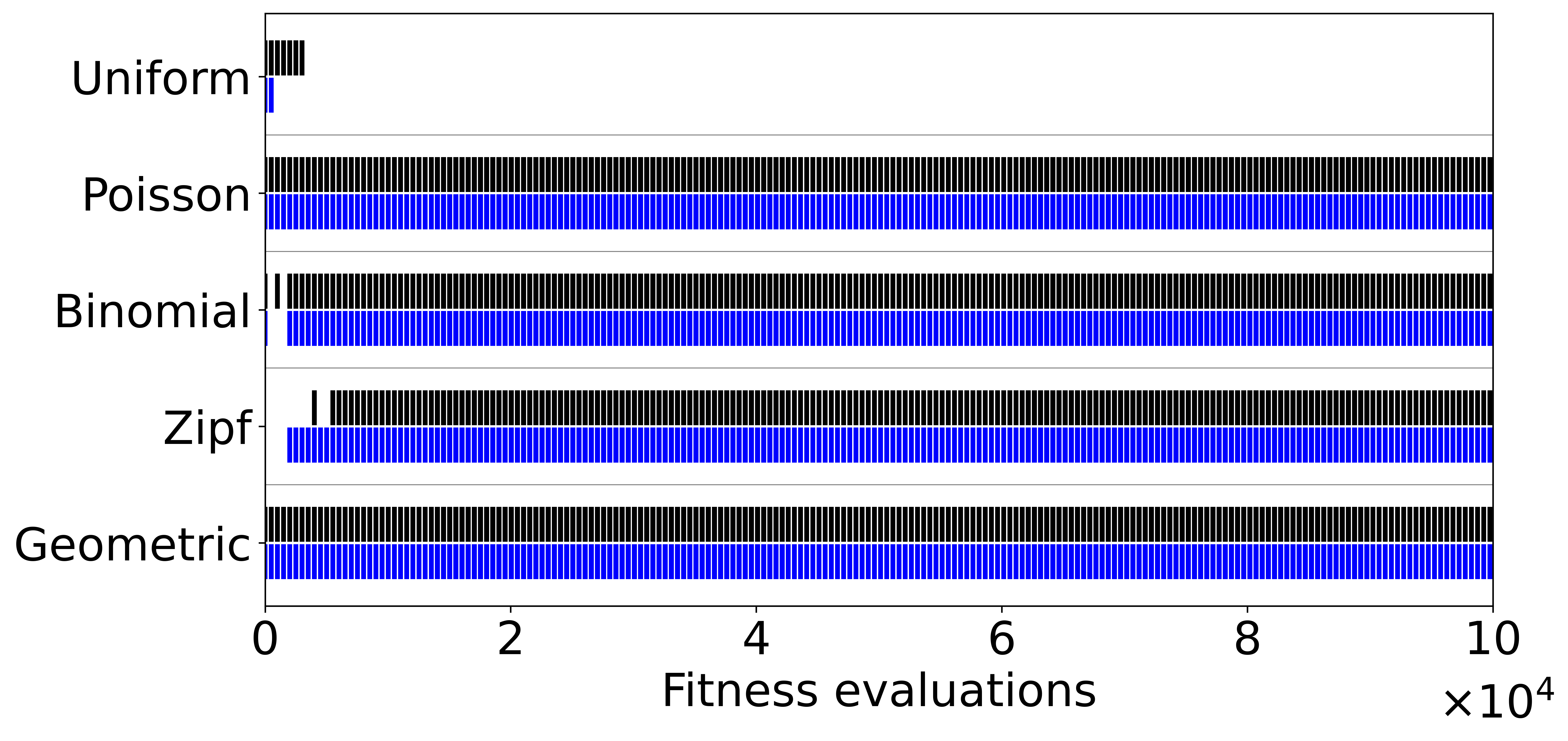} \\
            (a) Knapsack &
            (b) NK-Landscape \\
			\includegraphics[scale=0.3]{figures/TSP_100_gof_map.png} & \includegraphics[scale=0.3]{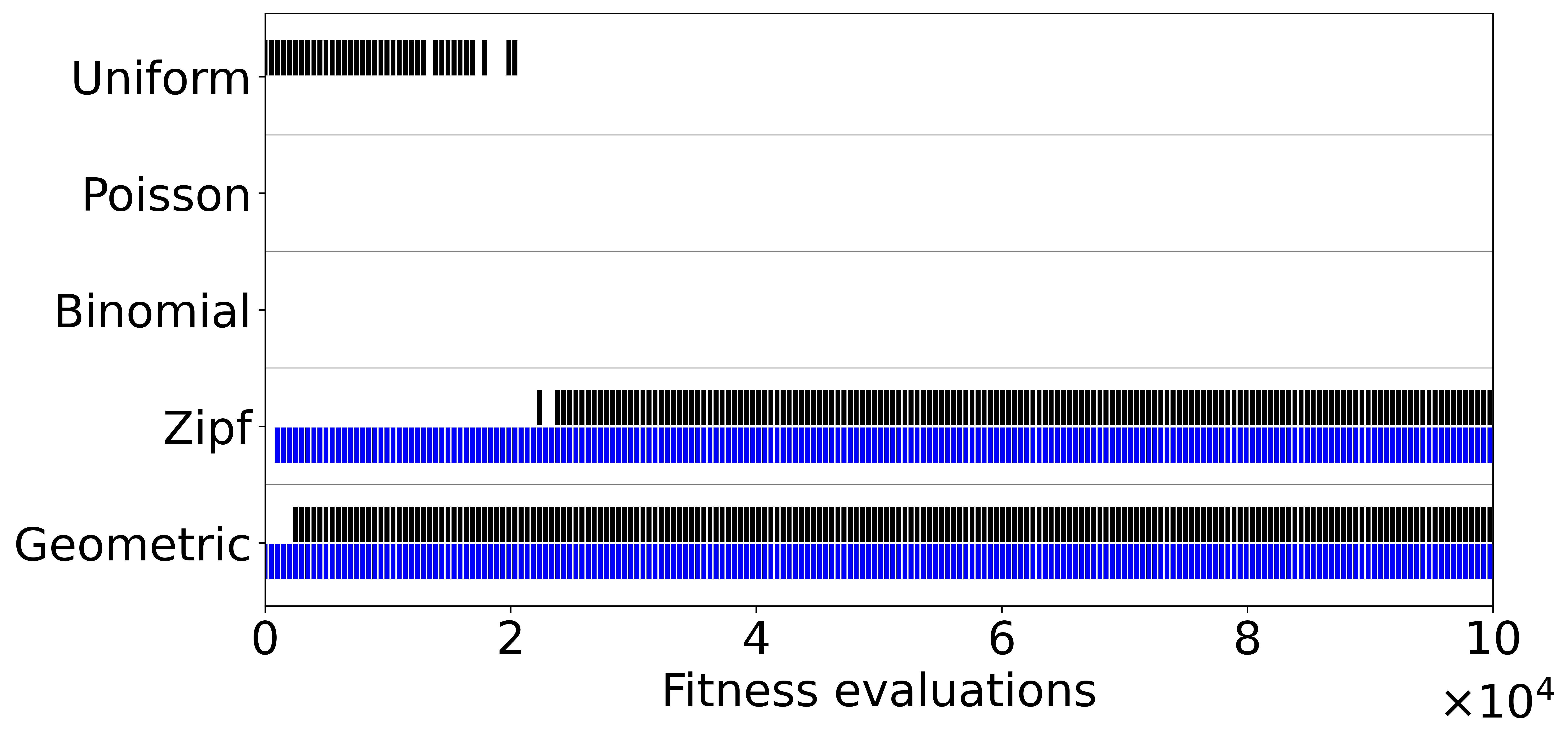} \\
	       (c) TSP & (d) QAP \\	
		\end{tabular}
	\end{center}
	\vspace{-10pt}
	\caption{Goodness-of-fit of the distributions with respect to the number of good neighbours of solutions in the archive during the search process of \emph{s}-PLS (black) and \emph{r}-PLS (blue) on the (a) Knapsack (100D), (b) NK-Landscape (N=100, K=10), (c) TSP (100D) and (d) QAP (100D). A coloured tick in a row indicates that the corresponding algorithm’s data at that point was not rejected under the model. }
	\caption{Goodness-of-fit of the distributions with respect to the number of good neighbours of solutions in the archive during the search process of \emph{s}-PLS (black) and \emph{r}-PLS (blue) on the (a) Knapsack (100D), (b) NK-Landscape (N=100, K=10), (c) TSP (100) and (d) QAP (100 factories). A coloured tick in a row indicates that the corresponding algorithm’s data at that point was not rejected under the model. }
	\label{Fig:gof}
\end{figure*}

\subsection{Distribution of Solutions' Neighbours in \emph{s}-PLS and \emph{r}-PLS}\label{apx:dist}
This section describes the details on investigating the distribution of the number of good neighbours for \emph{s}-PLS and \emph{r}-PLS.
Specifically, we count the number of good neighbours among solutions in the archive during the search process of \emph{s}-PLS and \emph{r}-PLS on the four MOCOPs, and test if this number follows a known discrete probabilistic model.
The distribution models~\cite{johnson2005} considered in this paper are the most commonly seen ones, characterised by their different probabilistic tail type: \emph{uniform} (a balance distribution), \emph{Zipf} (heavy-tailed with polynomial decay), \emph{geometric} (light-tailed with exponential decay), \emph{Poisson} (light-tailed with super-exponential decay), and \emph{binomial} (light-tailed with bounded support and a hard cut-off) distributions. 

Following the common practice, we evaluate the absolute fit quality using the $\chi ^2$ goodness-of-fit test at a significance level of $\alpha=5\%$. 
A fit is considered acceptable if the test does not reject the null hypothesis that the observed data come from the tested distribution. Parameters of the distributions were estimated via maximum likelihood.
For the Poisson distribution, the rate parameter $\lambda$ was set to the sample mean.
Similarly, for the geometric and binomial distributions, the success probability $p$ was estimated from the sample mean as well.
For the Zipf distribution, the exponent parameter $s$ was estimated by numerically minimising the negative log-likelihood over the interval $[1.01, 10]$.

\begin{figure*}[!h]
	\vspace{-5pt}
	\renewcommand{\arraystretch}{0.1} 
	\fontsize{8.5pt}{10pt}\selectfont
	\begin{center}
    \includegraphics[scale=0.4]{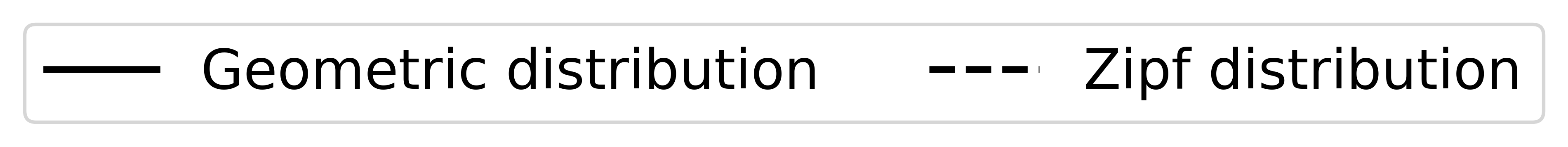}
    \vspace{-5pt}
        \begin{tabular}{@{}c@{}@{}c@{}@{}c@{}@{}c@{}}
			\includegraphics[scale=0.35]{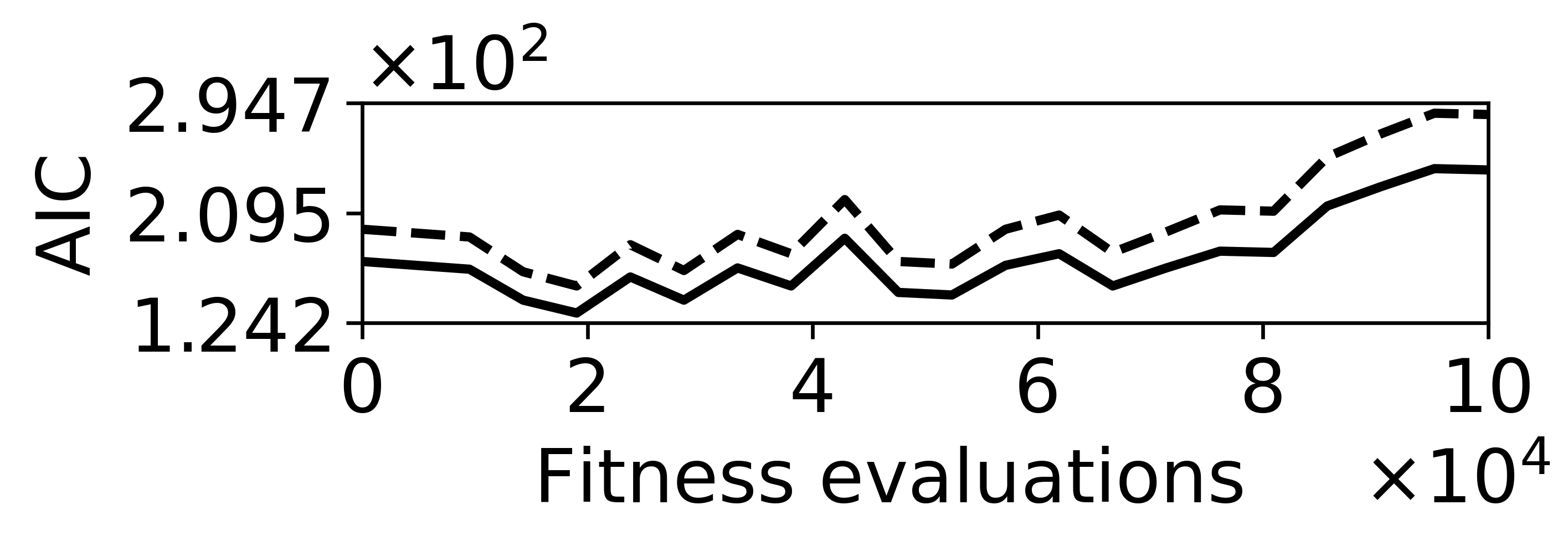} & 
			\includegraphics[scale=0.35] {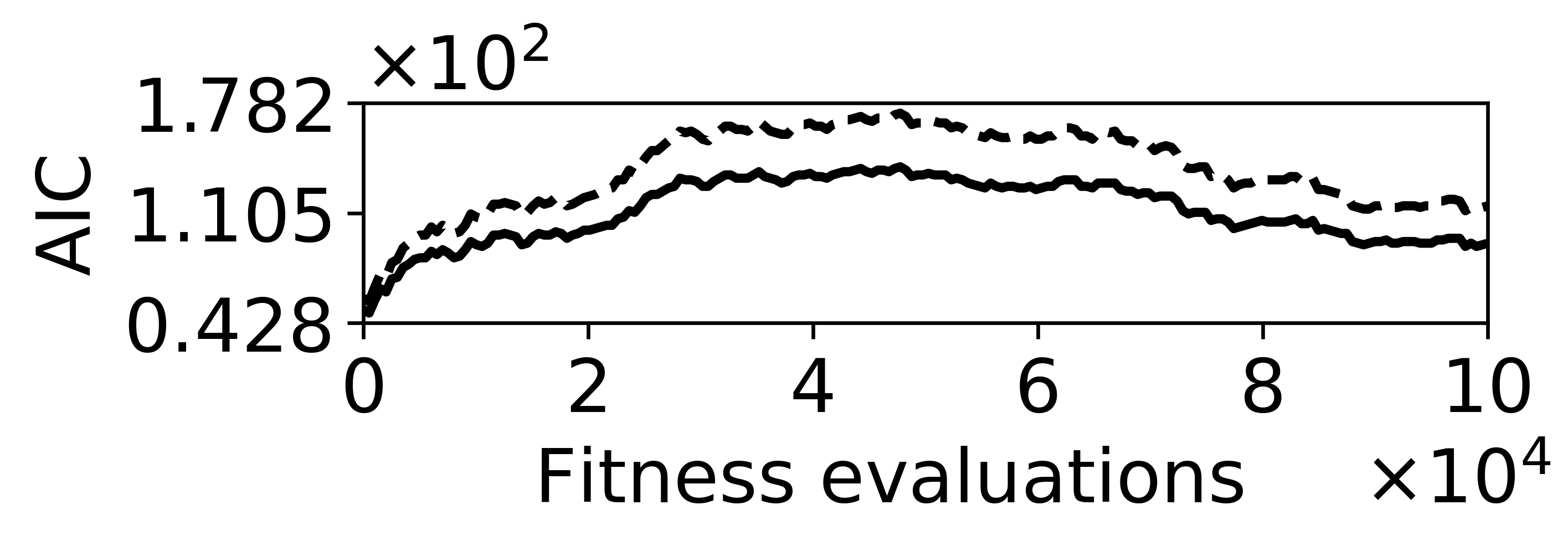}&
			\includegraphics[scale=0.35]{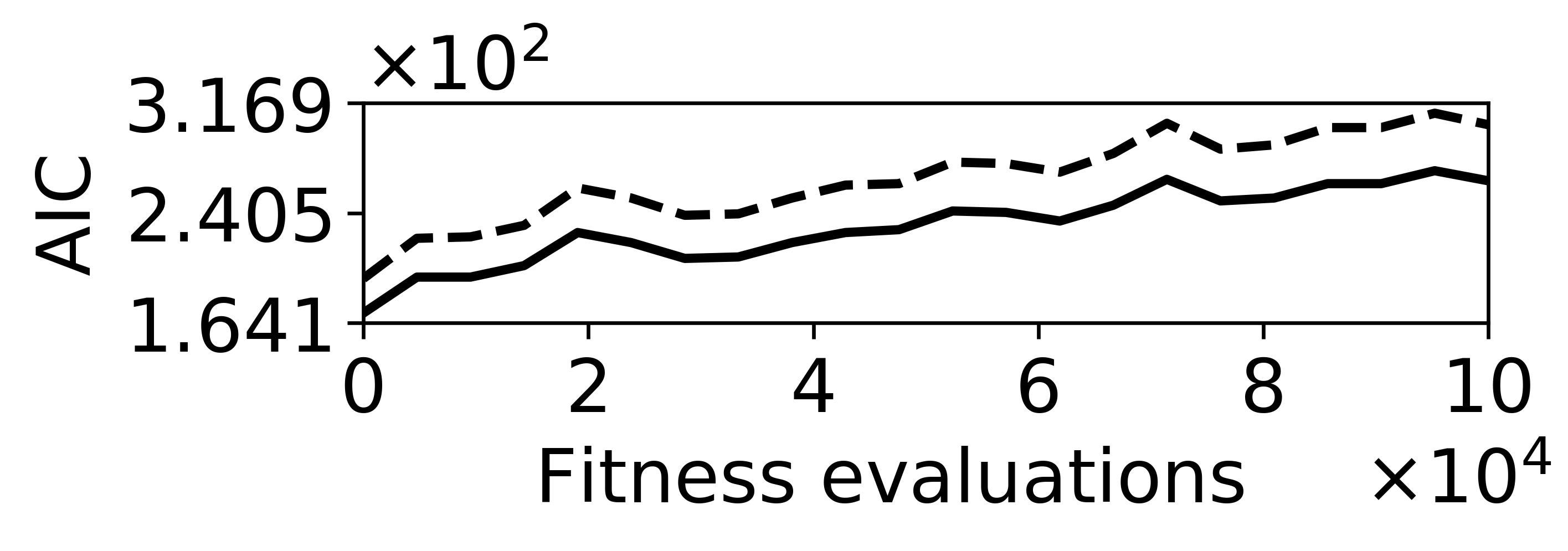}&
			\includegraphics[scale=0.35] {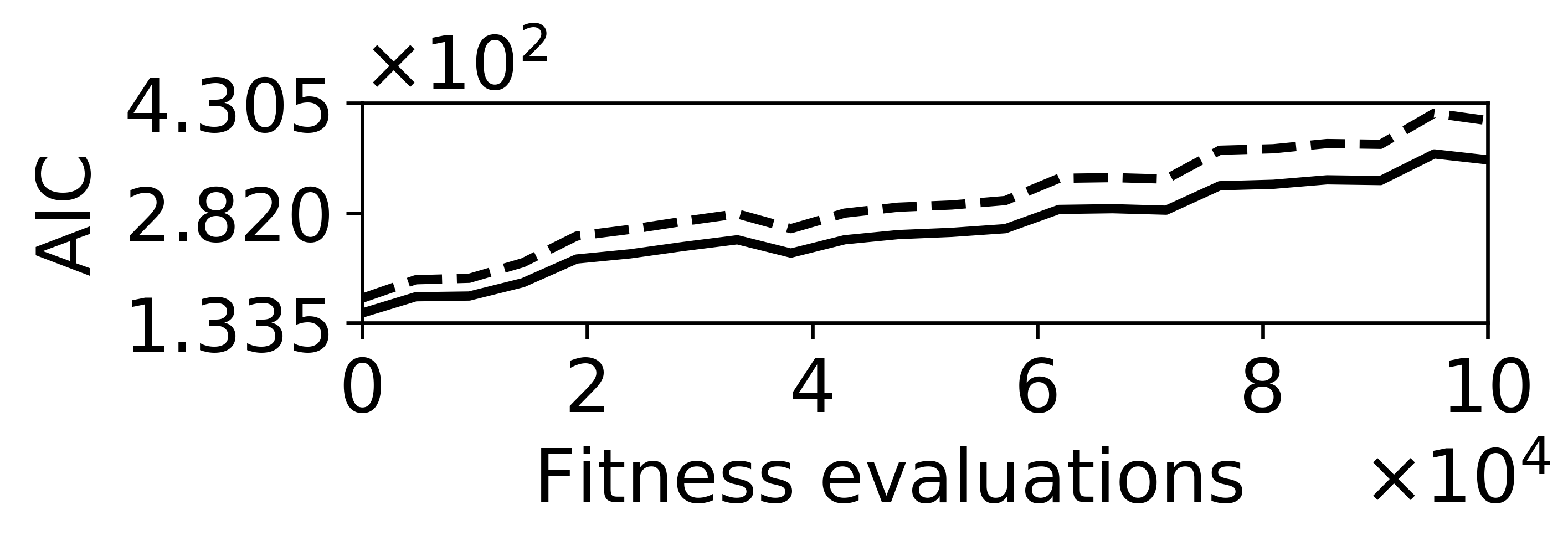}  \\
	(a) \emph{s}-PLS on Knapsack &
        (b) \emph{s}-PLS on NK-Landscape &
        (c) \emph{s}-PLS on TSP  &
        (d) \emph{s}-PLS on QAP   \\	
			\includegraphics[scale=0.37]{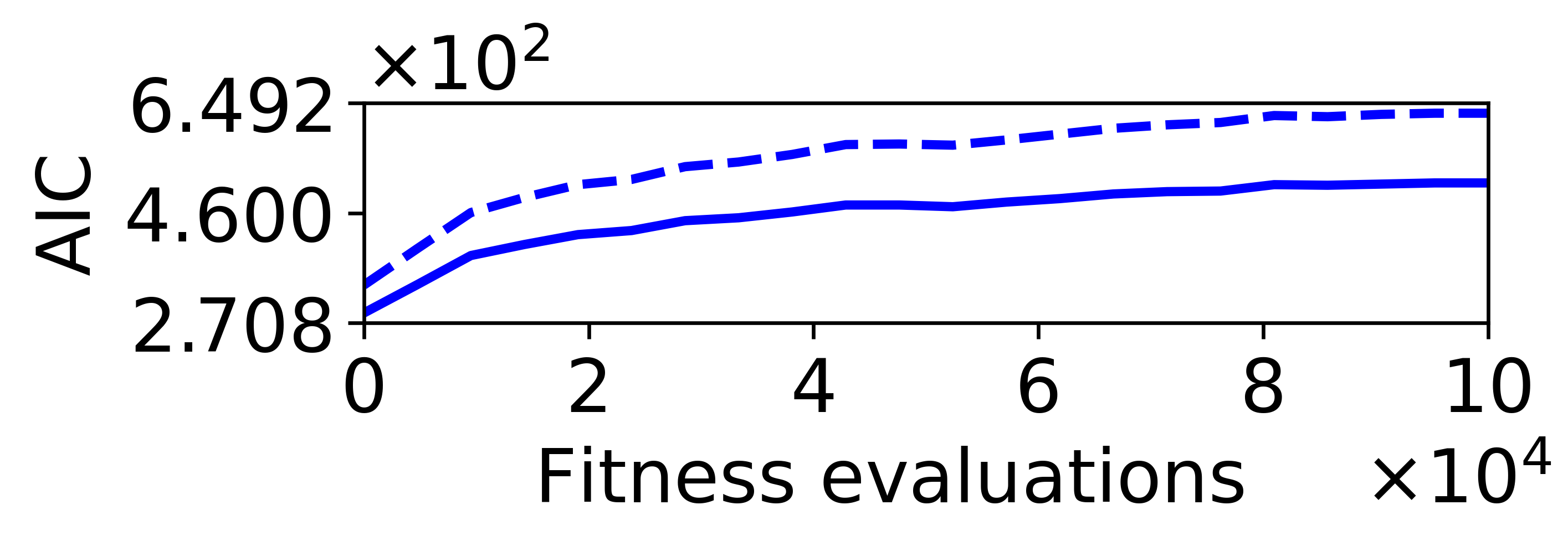} & 
			\includegraphics[scale=0.37] {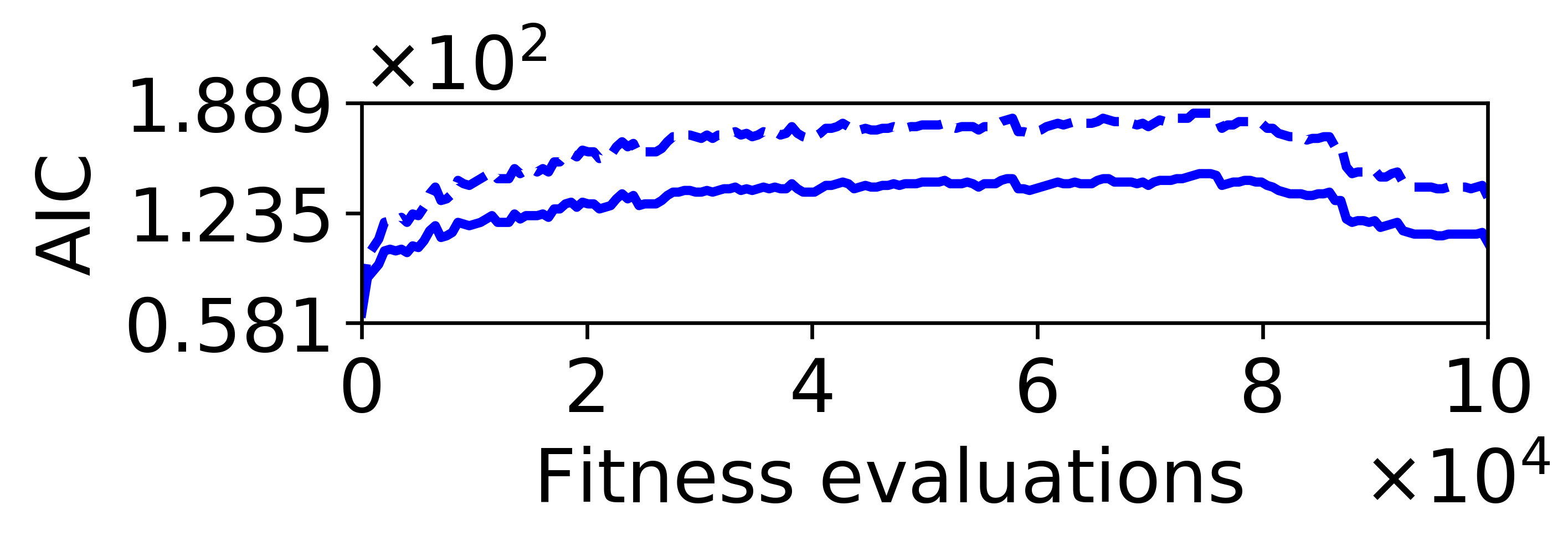} &
			\includegraphics[scale=0.37]{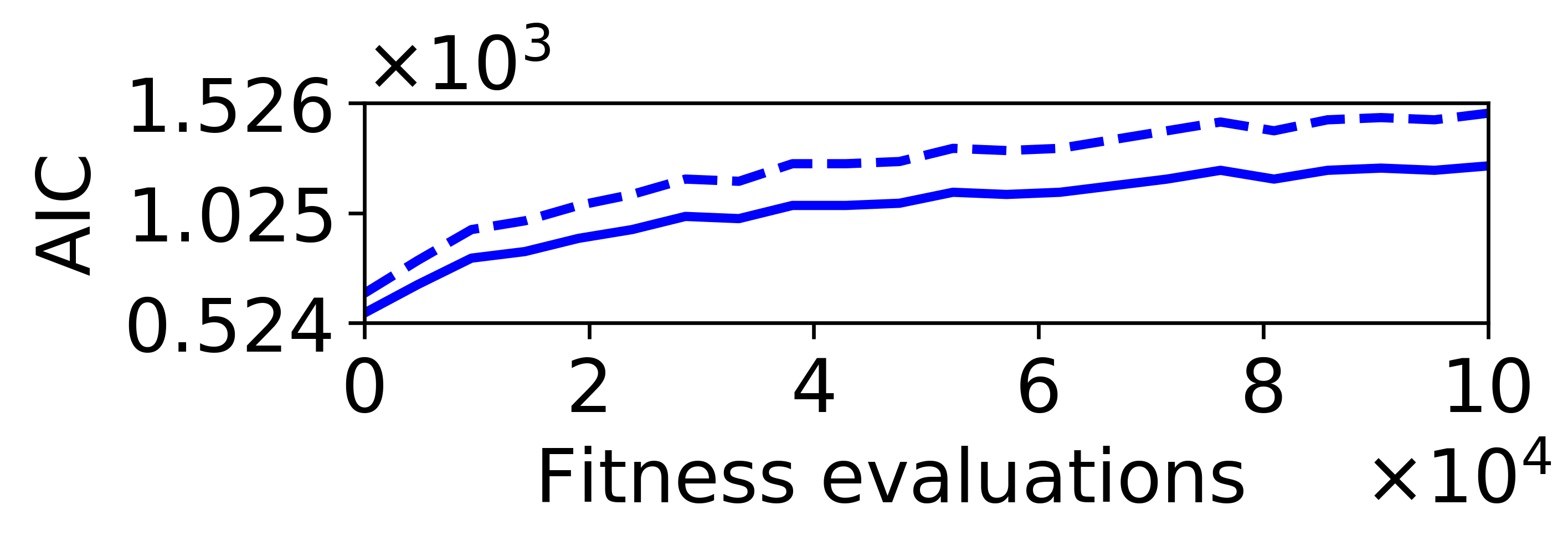}&
			\includegraphics[scale=0.37] {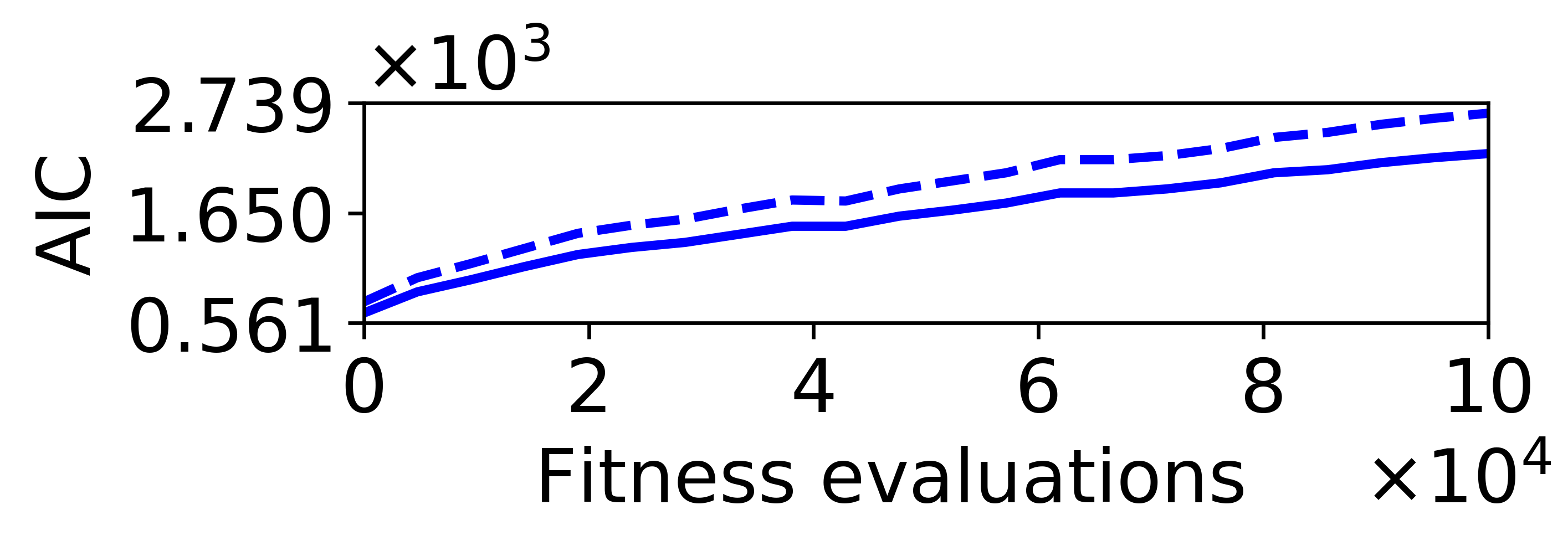} \\
	(e) \emph{r}-PLS on Knapsack &
        (f) \emph{r}-PLS on NK-Landscape &
        (g) \emph{r}-PLS on TSP  &
        (h) \emph{r}-PLS on QAP   \\	
		\end{tabular}
	\end{center}
	\vspace{-3pt}
	\caption{The Akaike’s Information Criterion (AIC)~\cite{akaike1974aic} trajectories (lower is better) of the geometric distribution (solid line) and the Zipf distribution (dashed line) throughout the search processes of \emph{s}-PLS (top panel) and \emph{r}-PLS (bottom panel) on the four MOCOPs with 100 variables. In all sub-figures, the solid lines lie below the dashed lines, indicating that the geometric distribution consistently provides a better fit than the Zipf distribution.
    }
	\label{Fig:aic}
\end{figure*}

Figure~\ref{Fig:gof} plots the goodness-of-fit of the five discrete distributions with respect to the numbers of good neighbours of solutions in the archive during the search process of \emph{s}-PLS and \emph{r}-PLS on the four MOCOPs. Each horizontal band corresponds to a candidate distribution, and at each sampled evaluation, two coloured ticks, black and blue, respectively indicate a good fit for \emph{s}-PLS and \emph{r}-PLS at $\alpha=0.05$, namely, the $\chi^2$ test does not reject the distribution.
As can be seen in the figure, across nearly all problem instances, the uniform distribution only fits for \emph{s}-PLS at the beginning of the search. 
In contrast, the geometric distribution passes the $\chi^2$ test almost across all the band of the problems for both \emph{s}-PLS and \emph{r}-PLS, indicating that the number of good neighbours fits well with a geometric distribution.

Note that the Zipf distribution has a comparable number of fits to the geometric distribution across all problems in Figure~\ref{Fig:gof}.
To compare the relative quality of fit between the two distributions, we use Akaike’s Information Criterion (AIC)~\cite{akaike1974aic}, which balances goodness of fit with model complexity by penalising models with more free parameters.
For each distribution, the degrees of freedom for each model are calculated as $n_{\text{bin}} - 1 - k$ where $n_{\text{bin}}$ is the number of solutions that have at least one good neighbour and $k$ is the number of free parameters of the distribution.
A lower AIC indicates that the model fits better to the data.
Figure~\ref{Fig:aic} plots the AIC trajectory of the geometric and Zipf distributions throughout the search processes of \emph{s}-PLS and \emph{r}-PLS on the four problems.
The solid and dashed lines correspond to the geometric and Zipf distributions, respectively.
As shown, the solid line lies below the dashed line in all cases, indicating that the geometric distribution consistently provides a better fit than the Zipf distribution throughout the search process in both \emph{s}-PLS and \emph{r}-PLS.

\subsection{Proofs of the Two Examples}\label{apx:example_proof}
In the main paper, we present two toy examples of solution's distributions in the archive to help illustrate why \emph{r}-PLS can be faster than \emph{s}-PLS. In Example 1, the archive contains two promising solutions (each having only good neighbours) and two unpromising solutions (each having no good neighbours). In Example 2, all the four solutions in the archive are half-promising (each having an equal mix of good and poor neighbours).
Here, we are interested in how quickly (in terms of the number of evaluations) an algorithm can find a good neighbour.
In this section, we provide the proofs showing that \emph{r}-PLS is faster in Example 1, whereas \emph{s}-PLS is faster in Example 2.

\begin{proposition}[\emph{r}-PLS is faster than \emph{s}-PLS in Example 1]\label{thn:halffull}
Let an archive of $n$ solutions contain exactly $n/2$ solutions with no good neighbours ($G=0$), and the other $n/2$ solutions with all neighbours being good ($G=|\mathcal N|$). Then, for any neighbourhood size $|\mathcal N|\geq 2$, \emph{r}-PLS requires less time (fewer number of evaluations) than \emph{s}-PLS in finding the next new good solution.
	
\end{proposition}
\begin{proof}
For \emph{r}-PLS, the algorithm selects a solution uniformly at random from the archive. With probability $\frac{1}{2}$, it selects a solution whose entire neighbourhood consists of good solutions. Since neighbours are also sampled uniformly, any selected neighbour will be good. Therefore, the success probability is:
\begin{align}
p_{\text{succ}}=\tfrac12,\quad \text{and} \quad \E[T_{r\text{-PLS}}]=\frac1{p_{\text{succ}}}=\frac1{1/2}=2.
\end{align}
	
In contrast, \emph{s}-PLS scans the archive sequentially as it marks each visited solution as ``explored''.
For each unpromising solution (one without good neighbours), it evaluates all $|\mathcal{N}|$ neighbours before moving on.
This continues until a promising solution is selected from the archive.
This process is equivalent to the one described in Lemma~1 (i.e., finding the first good neighbour in the neighbourhood in a sequential scan) of the main paper.
Accordingly, the expected number of evaluations required to find first promising solution is:
\[
\E[J]=\frac{n+1}{(n/2)+1}=\frac{2(n+1)}{n+2}.
\]
	
The total number of evaluations includes $|\mathcal{N}|$ evaluations for each of the $\mathbb{E}[J] - 1$ preceding unpromising solutions, and just 1 evaluation for the first promising one (since all of its neighbours are good). Therefore:
\begin{align}\label{Eq:example1Tspls}
\begin{split}
	\E[T_{s\text{-PLS}}]
		&= \bigl(\E[J]-1\bigr) |\N| + 1 \\
		&= \left(\frac{2(n+1)}{n+2}-1\right)|\N|+1 
		= \frac{n|\N|}{n+2}+1.
	\end{split}
	\end{align}
	Since $|\N|\ge2$, it follows that $\E[T_{s\text{-PLS}}]>\E[T_{r\text{-PLS}}]$.
\end{proof}

The above proposition shows that the \emph{r}-PLS is faster than \emph{s}-PLS in finding the first good solution in Example~1. The key reason is that \emph{s}-PLS may waste a substantial number of evaluations when it selects an unpromising solution, as it evaluates all of its neighbours. In contrast, \emph{r}-PLS avoids this by sampling only one neighbour at a time. In Example~2, however, the archive contains no unpromising solutions; instead, all solutions are half-promising -- each has an equal mix of good and poor neighbours. In what follows, we show that in Example~2, \emph{s}-PLS is faster than \emph{r}-PLS.

\begin{proposition}[\emph{s}-PLS is faster than \emph{r}-PLS in Example 2]\label{thn:halffull}
Let the archive consist of $n$ solutions, and suppose that for every solution, exactly half of its neighbours are good; that is, $G = |\mathcal{N}| / 2$. Then, for any neighbourhood size $|\mathcal N|\geq 2$, \emph{s}-PLS requires less time (fewer number of evaluations) than \emph{r}-PLS in finding the next new good solution.
	
\end{proposition}
\begin{proof}
In \emph{r}-PLS, a solution is selected uniformly at random from the archive, and a neighbour is drawn uniformly from its neighbourhood. Since each solution has exactly half good neighbours, the probability of success (i.e., selecting a good neighbour) is:
\begin{align}
p_{\text{succ}}=\tfrac12,\quad \text{and} \quad  \E[T_{r\text{-PLS}}]=\frac1{p_{\text{succ}}}=\frac1{1/2}=2.
\end{align}
	
As for \emph{s}-PLS, the algorithm explores the neighbourhood of one solution at a time, scanning its neighbours sequentially until a good one is found.
Since all solutions in the archive have identical neighbourhood structure, it suffices to consider the expected number of evaluations spent within a single neighbourhood.
In this case, \emph{s}-PLS only needs to scan one neighbourhood as all neighbourhoods have $|\N|/2$ good neighbours.
By Lemma~1 in the main paper, when half of the neighbours are good, the expected number of evaluations required to find a good neighbour within one neighbourhood is:
\[
\E[T_{s\text{-PLS}}]=\E[J]=\frac{|\N|+1}{(|\N|/2)+1}<2=\E[T_{r\text{-PLS}}].
\]
\end{proof}

\subsection{Proof for \emph{r}-PLS with Best-from-Multiple Selections}\label{apx:bms}
In the main paper, we state that the use of Best-from-Multiple-Selections (BMS)~\cite{Cai2015balance} in \emph{r}-PLS requires $\E[T_{r-\text{PLS}}]\cdot(1+O(k/|\N|))$ evaluations to find the next good neighbour. In this section, we provide the proof of that.

Consider a variant of \emph{r}-PLS equipped with Best-from-Multiple-Selections (BMS): in each iteration, a solution is drawn uniformly at random from the archive, and up to $k$ neighbours ($1\leq k\leq|\N|$) of this solution are sampled uniformly without replacement and evaluated one-by-one, stopping as soon as a good neighbour is found or after $k$ evaluations if no good neighbour is encountered. Let $T_{\text{BMS}}$ denote the number of evaluations required to find the next good neighbour.

\begin{proposition}[Expected time of \emph{r}-PLS with Best-from-Multiple-Selections]
\label{prop:bms-rpls}
Let $|\mathcal{N}|\in\mathbb{Z}^+$ be the neighbourhood size and let $G$ be the number of good neighbours among the solutions in the archive, where $G$ follows a geometric distribution with parameter $p\in(0,1]$. Let $T_{\mathrm{BMS}}$ be the number of evaluations until the next good neighbour is found by the BMS variant of \emph{r}-PLS. Then 
\[
\mathbb{E}[T_{\mathrm{BMS}}] > \mathbb{E}[T_{r\text{-PLS}}]
\]\noindent
where $\mathbb{E}[T_{r\text{-PLS}}]=p|\mathcal{N}|/(1-p)$.
\end{proposition}

\begin{proof}
Recalling \emph{r}-PLS from the main theorem, conditional on $G$, the success probability is given by
\[
  q := \mathbb{E}\!\left[\frac{G}{|\mathcal{N}|}\right]
     = \frac{\mathbb{E}[G]}{|\mathcal{N}|}
     = \frac{1-p}{p|\mathcal{N}|},
\]
\noindent
and $\mathbb{E}[T_{r\text{-PLS}}]=1/q=p|\mathcal{N}|/(1-p)$.

We next analyse the BMS variant, which samples $k>1$ neighbours of a solution at a time. Regarding the selected solution, there are typically two cases: $G=0$ (no good neighbour) and $G>0$ (at least one good neighbour). 

We start from considering the $G>0$ case. At the beginning of each BMS iteration, a solution with $G=g$ good neighbours is chosen. Within that iteration, suppose $h$ neighbours have already been examined and found poor ($0\le h\le k-1$). The next evaluation chooses uniformly one neighbour among the remaining $|\mathcal{N}|-h$ neighbours, which still contain the same $g$ good ones (because the iteration would have stopped earlier otherwise). Thus, conditional on $G=g$,
\[
  \Pr(\text{success at this evaluation}\mid G=g,\ h)
  = \frac{g}{|\mathcal{N}| - h}.
\]
Since $0\le h\le k-1\le|\mathcal{N}|-1$, we have
\[
  |\mathcal{N}| - (k-1) \,\le\, |\mathcal{N}| - h \,\le\, |\mathcal{N}|,
\]
and therefore
\[
  \frac{g}{|\mathcal{N}|}
  \,\le\,
  \frac{g}{|\mathcal{N}| - h}
  \,\le\,
  \frac{g}{|\mathcal{N}| - (k-1)}.
\]

With $0\le h\le k-1$, we have the bounds on the success probability $q_t$ of the $t$-th evaluation in the BMS process (given that we have not succeeded earlier):
\[
  \mathbb{E}[G]/|\mathcal{N}|
  \,\le\,
  q_t
  \,\le\,
  \frac{\mathbb{E}[G]}{|\mathcal{N}| - (k-1)}\quad\forall\,t\ge1.
\]
Plugging in the success probability of \emph{r}-PLS, we have
\[
  q \,\le\, q_t \,\le\, \frac{q}{1 - \tfrac{k-1}{|\mathcal{N}|}}\quad\forall\,t\ge1.
\label{eq:qt-bounds}
\]

We now bound the stopping time $T_{\mathrm{BMS}}$ from the bounds on $q_t$. Let $q_{\min}=q$ and $q_{\max}=q/(1-\tfrac{k-1}{|\mathcal{N}|})$. For the BMS process, write $q_t$ for the success probability of the $t$-th evaluation (conditioned on no success earlier), we have
\[
\begin{split}
  \frac{1}{q_{\max}}
  &\,\le\,
  \frac{1}{q_t}
  \,\le\,
  \frac{1}{q_{\min}}\quad \forall\,t\ge1.
\end{split}
\]

Substituting the equations with the expected number of evaluations until a good neighbour is found
\(
  \mathbb{E}[T_{\mathrm{bms}}]
\), 
\[
  \frac{1}{q_{\max}}
  \,\le\,
  \mathbb{E}[T_{\mathrm{bms}}]
  \,\le\,
  \frac{1}{q_{\min}}.
\]
Plugging $q_{\min}$ and $q_{\max}$ back,
we obtain
\[
  \Bigl(1-\tfrac{k-1}{|\mathcal{N}|}\Bigr)\mathbb{E}[T_{r\text{-PLS}}]
  \,\le\,
  \mathbb{E}[T_{\mathrm{bms}}]
  \,\le\,
  \mathbb{E}[T_{r\text{-PLS}}].
\]
Dividing by $\mathbb{E}[T_{r\text{-PLS}}]$ yields
\[
  1-\tfrac{k-1}{|\mathcal{N}|}
  \,\le\,
  \frac{\mathbb{E}[T_{\mathrm{bms}}]}{\mathbb{E}[T_{r\text{-PLS}}]}
  \,\le\,
  1,
\]

Then, we add the case of no good neighbours ($G=0$) that wastes $k$ evaluations at the rate $p$ out of the $\frac{1}{1-p}$ solutions explored, which yields \[
\mathbb{E}[T_{\mathrm{BMS}}]
  = \frac{kp}{1-p} + \mathbb{E}[T_{\mathrm{bms}}].
\] Plugging in the above inequalities, we have
\[
  \frac{kp}{1-p} + (1-\tfrac{k-1}{|\mathcal{N}|})\mathbb{E}[T_{r\text{-PLS}}] \leq \mathbb{E}[T_{\mathrm{BMS}}] \leq \frac{kp}{1-p} + \mathbb{E}[T_{r\text{-PLS}}] 
\]
\[
\begin{split}
  \mathbb{E}[T_{r\text{-PLS}}]+\frac{\mathbb{E}[T_{r\text{-PLS}}]}{|\N|} &\leq \mathbb{E}[T_{\mathrm{BMS}}]
  \leq \frac{k}{|\N|}\mathbb{E}[T_{r\text{-PLS}}]+\mathbb{E}[T_{r\text{-PLS}}] \\
  1+\frac{1}{|\N|}&\leq\frac{\mathbb{E}[T_{\mathrm{BMS}}]}{\mathbb{E}[T_{r\text{-PLS}}]}\leq 1+\frac{k}{|\N|}
\end{split}
\]

Since $1+\frac{1}{|\N|}\leq\frac{\mathbb{E}[T_{\mathrm{BMS}}]}{\mathbb{E}[T_{r\text{-PLS}}]}$, the BMS variant is slower than the \emph{r}-PLS.

\end{proof}

It is worth noting that, since $\frac{\mathbb{E}[T_{\mathrm{BMS}}]}{\mathbb{E}[T_{r\text{-PLS}}]}\leq 1+\frac{k}{|\N|}$, the BMS variant is only slightly slower than the \emph{r}-PLS when $k$ is small.

\section*{}

\end{document}